\documentclass[letterpaper, 11pt]{article}
\usepackage[utf8]{inputenc}
\usepackage[margin = 1in]{geometry}
\usepackage{smile}
\usepackage[colorlinks, linkcolor=orange, anchorcolor=blue, citecolor=blue]{hyperref}
\usepackage{xcolor}
\usepackage{natbib}
\usepackage{enumitem}
\newcommand{\ptz}{p_t^{\sf LD}}
\newcommand{\pld}{p^{\sf LD}}
\newcommand{\sld}{\tilde{\sbb}^{\sf LD}}
\newcommand{\myEE}{\bar{\mathbb{E}}}
\newcommand{\Ybp}{\bXb_{t,\parallel}}
\newcommand{\Ybo}{\bXb_{t,\perp}}
\newcommand{\pscore}{\sbb_{\parallel}}
\newcommand{\oscore}{\sbb_{\perp}}

\newcommand*\diff{\mathop{}\!\mathrm{d}}

\newcommand{\bXb}{\Xb^{\leftarrow}}
\newcommand{\tildebXb}{\tilde{\Xb}^{\leftarrow}}
\newcommand{\hatbXb}{\tilde{\Xb}^{\leftarrow}}
\newcommand{\hatbXbDis}{\tilde{\Xb}^{\Leftarrow}}
\newcommand{\bXbo}{\Xb^{\leftarrow}_{t,\perp}}

\newcommand{\bZb}{\Zb^{\leftarrow}}
\newcommand{\tildebZb}{\tilde{\Zb}^{\leftarrow}}
\newcommand{\tildebZbRot}{\tilde{\Zb}^{\leftarrow,r}}
\newcommand{\tildebZbRotDis}{\tilde{\Zb}^{\Leftarrow,r}}

\newcommand{\TV}{{\sf TV}}

\newcommand{\CE}{C_{E}}
\newcommand{\Czero}{c_0}
\newcommand{\Ctwo}{C_\zb}

\title{\Large \bf Score Approximation, Estimation and Distribution Recovery of Diffusion Models on Low-Dimensional Data}

\author{Minshuo Chen$^{{1, \dagger}}$ \quad  Kaixuan Huang$^{{1, \dagger}}$ \quad Tuo Zhao$^2$ \quad Mengdi Wang$^1$ \thanks{$\dagger$ Equal contribution. Emails: \texttt{\{mc0750, kaixuanh, mengdiw\}@princeton.edu, tourzhao@gatech.edu}.} \vspace{0.05in} \\
$^1$Princeton University $\quad$ $^2$Georgia Tech}

\begin{document}

\maketitle

\begin{abstract}
Diffusion models achieve state-of-the-art performance in various generation tasks. However, their theoretical foundations fall far behind. This paper studies score approximation, estimation, and distribution recovery of diffusion models, when data are supported on an unknown low-dimensional linear subspace. Our result provides sample complexity bounds for distribution estimation using diffusion models. We show that with a properly chosen neural network architecture, the score function can be both accurately approximated and efficiently estimated. Furthermore, the generated distribution based on the estimated score function captures the data geometric structures and converges to a close vicinity of the data distribution. The convergence rate depends on the subspace dimension, indicating that diffusion models can circumvent the curse of data ambient dimensionality.
\end{abstract}


\section{Introduction}\label{sec:intro}

Diffusion models achieve state-of-the-art performance in image and audio generating tasks \citep{song2019generative, dathathri2019plug, song2020score, ho2020denoising} and are one of the fundamental building blocks of the more advanced image synthesis system, e.g., DALL-E-2 \citep{ramesh2022hierarchical} and stable diffusion \citep{rombach2022high}.

A standard diffusion model \citep{sohl2015deep, ho2020denoising} consists of a forward process and a backward process: In the forward process, a data point is sequentially corrupted by Gaussian random noises and in the limit the data distribution is transformed into white noise; In the backward process, a denoising neural network is trained to sequentially remove the added noise in the data and restore the clean data point. Using the trained denoising network for the backward process, one can generate diverse and high fidelity samples by first sampling from the standard Gaussian distribution and then progressively removing noises.

The distinctive denoising objective separates diffusion models from other deep generative models such as GANs \citep{NIPS2014_5ca3e9b1}, and Normalizing Flows \citep{rezende2015variational}. As shown by \citet{vincent2011connection}, the training of denoising network essentially learns the so-called ``score function'', i.e., the gradient of log probability density function. Therefore, diffusion models fall into the category of Score-based Generative Models (SGMs).

Despite the empirical success of diffusion models, the theory is still in its embryo. Here we are interested in answering two fundamental questions:

\noindent {\it Q1}. {\it Can neural networks well approximate and learn score functions, especially when data have intrinsic geometric structures? If so, how should one choose the neural network architectures, and what is the sample complexity of learning?}

\noindent {\it Q2}. {\it Can diffusion models estimate the data distribution using the learned score functions? If so, how are the data intrinsic geometric structures being captured and how do they affect the sample complexity?}

Both {\it Q1} and {\it Q2} raise a practical concern about the real world data, such as high resolution images. These data, though having high ambient dimensions, often exhibit low-dimensional structures \citep{pope2021intrinsic}, due to symmetries, repetitive patterns, and local regularities \citep{tenenbaum2000global, roweis2000nonlinear}. Deep neural networks have been known for capturing certain low-dimensional data geometric structures \citep{schmidt2017nonparametric, suzuki2018adaptivity, nakada2020adaptive, shen2022optimal}. However, whether such abilities hold for diffusion models remains unclear.

Some recent works skipped {\it Q1} and attempted to study {\it Q2}, by directly assuming that the score function is accurately learned up to a small error under certain metric, e.g., $L^2$/$L^\infty$ norm \citep{de2022convergence, lee2022convergencea, chen2022sampling, lee2022convergenceb}. \citet{de2022convergence} in particular studied low-dimensional manifold data. These progresses unveil important theoretical insights about the sampling properties of the backward process of diffusion models, however, leaving {\it Q1} largely untouched. As a result, a full theoretical picture of diffusion models is lacking.

To bridge the gap between theory and practice, we make a first step towards an integrated analysis to answer both {\it Q1} and {\it Q2} for diffusion models. The combined result provides sample complexity bounds of diffusion models for learning data distributions supported on low-dimensional linear subspaces. Specifically, we consider data point $\xb = A\zb$, where $\zb$ is referred to as the latent variable, columns of $A \in \RR^{D \times d}$ form an orthonormal basis of $\RR^d$ for $d < D$. We refer to $d$ as the intrinsic dimension and $D$ as the ambient dimension. 

Based on such a low-dimensional linear subspace assumption, we can decompose the score function of the linear subspace data into on-support and orthogonal components (Lemma \ref{lemma:subspace_score}). We then characterize their distinct behaviors of the two components, where on-support component carries latent distribution information and orthogonal component forces the subspace recovery.

Our main contributions are summarized as follows:

\noindent $\bullet$ We specify an encoder-decoder neural architecture with skip-layer connections and establish its approximation guarantees with respect to the score functions under the $L^2$ norm (Theorem \ref{thm:score_approximation}). Specifically, given an approximation error $\epsilon$, we show that the network size needs to be exponential in $1/\epsilon$ with the exponent depending on the data intrinsic dimension $d$.

\noindent $\bullet$ We establish statistical guarantees of score estimation using our properly chosen encoder-decoder neural network. We show that such a neural score estimator converges to the ground truth score under the $L^2$ norm at a rate of $\tilde{\cO}(\frac{1}{\sqrt{t_0}} n^{-\frac{1}{d+5}})$, where $n$ is the sample size and $t_0$ is an early stopping time (Theorem \ref{thm:score_estimation}). This result indicates that the neural score estimator does not suffer from the curse of the data ambient dimensionality in score estimation, when the data exhibit intrinsic geometric structures.

\noindent $\bullet$ We establish distribution estimation guarantees using the learned neural score estimator. By simulating a discretized backward process, the generated data distribution of diffusion models converges to a close vicinity of the data distribution (Theorem \ref{thm:distro_estimation}). Specifically, for the on-support direction, generated distribution enjoys a $\tilde{\cO}(n^{-\frac{1}{2(d+5)}})$ rate of convergence in Total Variation distance. For the orthogonal direction, the generated distribution vanishes in magnitude, and the support of the data is approximated recovered. Our analysis demonstrates that diffusion models are free of the curse of data ambient dimensionality.

\subsection{Related work}
Several recent works study diffusion models from the sampling perspective. \citet{de2021diffusion} study the convergence of diffusion Schr\"{o}dinger bridges by assuming the score estimator is accurate under the $L^\infty$ norm. \citet{lee2022convergencea} provide polynomial convergence guarantees of SGMs, under the assumption that the score estimator is accurate under the $L^2$ norm. In addition, \citet{lee2022convergencea} require the data distribution satisfying a log-Sobolev inequality. Concurrent works \citet{chen2022sampling} and \citet{lee2022convergenceb} improve previous results by extending to distributions with bounded moments. Their analyses still assume access to an accurate score estimator under the $L^2$ norm. It is worth mentioning that \citet{lee2022convergenceb} allow the error of the score estimator under the $L^2$ norm to scale with time.

Moreover, \citet{de2022convergence} made an interesting attempt to analyze diffusion models for learning low-dimensional manifold data. Assuming the score estimator is accurate under the $L^\infty$ norm (extension to the $L^2$ norm is also provided), \citet{de2022convergence} provide distribution estimation guarantees of diffusion models in terms of the Wasserstein distance. The obtained convergence rate has an exponential dependence on the diameter of manifold. 

As stated, aforementioned works hardly touch {\it Q1} and provide partial understandings of diffusion models. To the best of our knowledge, \citet{block2020generative} is the only work in existing literature, which provides score estimation guarantees under the $L^2$ norm. Yet the error bound depends on some unknown Rademacher complexity of certain concept class. In comparison, our work is explicit on the choice of a neural network concept class and score estimation error bound. Note that \citet{block2020generative} also provide sampling convergence guarantees under the assumption of access to an accurate score estimator under the $L^2$ norm. We are also aware of \citet{song2020sliced} and \citet{liu2022let} studying score estimation and distribution estimation from an asymptotic statistics point of view. 

\noindent {\bf Notations}: We use bold lower case letters to denote vectors, e.g., $\xb \in \RR^D$. For a vector $\xb$, $\norm{\xb}_2$ and $\norm{\xb}_\infty$ denote its Euclidean norm and maximum magnitude of entries, respectively. Normal upper case letters denote matrices, e.g., $A \in \RR^{D \times d}$. For a matrix $A$, $\norm{A}_{\rm op}$ and $\norm{A}_{\rm F}$ denote its operator norm and Frobenius norm, respectively. Given a mapping $\fb$ and a distribution $P$, we denote $\norm{\fb}_{L^2(P)} = \EE_P^{1/2}[\norm{\fb}_2^2]$ as the $L^2(P)$ norm. We also denote $\fb_{\sharp} P$ as a pushforward measure, i.e., for any measurable $\Omega$, $(\fb_{\sharp}P)(\Omega) = P(\fb^{-1}(\Omega))$, We reserve $\phi$ for (conditional) Gaussian density functions.

\section{Preliminaries}\label{sec:background}

We briefly review diffusion models and score matching using neural networks.

\paragraph{Forward and backward SDEs} The forward process in diffusion models progressively adds noise to original data. Here we consider the Ornstein-Ulhenbeck process, which is described by the following SDE,
\begin{align}\label{eq:forward_sde}
\diff \Xb_t = -\frac{1}{2} g(t) \Xb_t \diff t + \sqrt{g(t)} \diff\Wb_t ~~~ \text{for} ~~ g(t) > 0,
\end{align}
where initial $\Xb_0 \sim P_{\rm data}$ follows the data distribution, $(\Wb_t)_{t\geq 0}$ is a standard Wiener process, and $g(t)$ is a nondecreasing weighting function. We denote the marginal distribution of $\Xb_t$ at time $t$ as $P_t$. Roughly speaking, after an infinitesimal time, \eqref{eq:forward_sde} shrinks the magnitude of data and corrupts data by Gaussian white noise.
More precisely, given $\Xb_0$, the conditional distribution of $\Xb_t | \Xb_0$ is Gaussian ${\sf N}(\alpha(t) \Xb_0, h(t)I_D)$, where $\alpha(t) = \exp(-\int_0^t \frac{1}{2}g(s) ds)$ and $h(t) = 1 - \alpha^2(t)$. Consequently, under mild conditions, \eqref{eq:forward_sde} transforms initial distribution $P_{\rm data}$ to $P_{\infty} = {\sf N}(\boldsymbol{0}, I_D)$. Therefore, \eqref{eq:forward_sde} is also known as the variance preserving forward SDE \citep{song2020score}.

In practice, the forward process \eqref{eq:forward_sde} will terminate at a sufficiently large time horizon $T > 0$, where the corrupted marginal distribution $P_T$ is expected to be close to the standard Gaussian distribution.

Diffusion models generate fake data by reversing the time of \eqref{eq:forward_sde}, which leads to the following backward SDE,
\begin{align}\label{eq:backward_sde}
\diff \bXb_t & = \left[\frac{1}{2}g(T-t)\bXb_t + g(T-t) \nabla \log p_{T-t}(\bXb_t)\right] \diff t + \sqrt{g(T-t)} \diff \overline{\Wb}_t,
\end{align}
where $\nabla \log p_t(\cdot)$ is the score function, i.e., the gradient of log probability density function of $P_t$, and $\overline{\Wb}_t$ is a reversed Wiener process. Under mild conditions, when initialized at $\bXb_0 \sim P_T$, the backward process $(\bXb_t)_{0\leq t\leq T}$ has the same distribution as the time-reversed version of the forward process $(\Xb_{T-t})_{0 \leq t \leq T}$ \citep{anderson1982reverse, haussmann1986time}.

Working with \eqref{eq:backward_sde}, however, leads to difficulties, as both the score function $\nabla \log p_t$ and initial distribution $P_T$ are unknown. In practice, several surrogates are deployed. Firstly, we replace $P_T$ by the standard Gaussian distribution. Secondly, we use a score estimator $\hat{\sbb}$ instead of ground truth score $\nabla \log p_t$. The estimated score $\hat{\sbb}$ is often parameterized by a neural network. With these substitutions, we obtain the following practical backward SDE,
\begin{align}\label{eq:backward_practice}
\diff  {\hatbXb}_t & = \left[\frac{1}{2}g(T-t) {\hatbXb}_t + g(T-t) \hat{\sbb}(\hatbXb_{t}, T-t)\right] \diff t + \sqrt{g(T-t)} \diff \overline{\Wb}_t, \quad \hatbXb_{0} \sim {\sf N}(\boldsymbol{0}, I_D).
\end{align}
Diffusion models then generate data by simulating a discretization of \eqref{eq:backward_practice} with $\eta>0$ being the discretization step size:
\begin{align}\label{eq:backward_piecewise}
    \diff  {\hatbXbDis}_t & = \left[\frac{1}{2}g(T-t) {\hatbXbDis}_{k\eta} + g(T-t) \hat{\sbb}(\hatbXbDis_{k\eta}, T-k\eta)\right] \diff t + \sqrt{g(T-t)} \diff \overline{\Wb}_t, \text{ for} ~~ t \in [k\eta, (k+1)\eta],
\end{align}
Throughout the paper, we take $g(t) = 1$ for simplicity.

\paragraph{Score matching}  
To estimate the score function, a conceptual way is to minimize a weighted quadratic loss:
\begin{align*}
\min_{\sbb \in \cS} \int_{0}^T w(t) \EE_{\Xb_t \sim P_t} \left[\norm{\nabla \log p_t(\Xb_t) - \sbb(\Xb_t, t)}_2^2 \right] \diff t,
\end{align*}
where $w(t)$ is a weighting function and $\cS$ is a concept class (often neural networks). However, such an objective function is intractable, as $\nabla \log p_t$ is unknown. As shown by \citet{vincent2011connection}, rather than minimizing the integral above, we can minimize an equivalent objective,
\begin{align}\label{eq:denoising_score_matching}
\min_{\sbb \in \cS} & \int_{0}^T w(t) \EE_{\Xb_0 \sim P_{\rm data}} \Big[\EE_{\Xb_t|\Xb_0}\Big[\big\|\nabla_{\Xb_t} \log \phi_t(\Xb_t | \Xb_0) - \sbb(\Xb_t , t)\big\|_2^2 \Big] \Big]\diff t.
\end{align}
Here $\phi_t(\Xb_t | \Xb_0)$ denotes the transition kernel of the forward process, which is Gaussian. Hence, we have an analytical form
\[
    \nabla_{\Xb_t} \log \phi_t(\Xb_t | \Xb_0) = -\frac{\Xb_t - \alpha(t)\Xb_0 }{h(t)}.
\]
Note that $\nabla_{\Xb_t} \log \phi_t(\Xb_t | \Xb_0)$ is the noise added to $\Xb_0$ at time $t$. Therefore, \eqref{eq:denoising_score_matching} is known as denoising score matching.

In practice, we approximate \eqref{eq:denoising_score_matching} by its empirical version. Specifically, given $n$ i.i.d. data points $\xb_i \sim P_{\rm data}$ for $i = 1, \dots, n$, we sample $\Xb_t$ given $\Xb_0 = \xb_i$ from ${\sf N}(\alpha(t)\xb_i, h(t)I_D)$. We also sample time $t$ uniformly from interval $[t_0, T]$ for some small $t_0 > 0$. (In Section \ref{sec:distro_guarantee}, we will choose $t_0$ based on sample size $n$.) The reason behind avoiding $[0, t_0]$ is to prevent score from blowing up and stabilize training \citep{vahdat2021score, song2020improved}. To this end, the empirical score matching objective is
\begin{align}\label{eq:score_matching_empirical}
\min_{\sbb \in \cS}~ \hat{\cL}(\sbb) & = \frac{1}{n} \sum_{i=1}^n \ell(\xb_i; \sbb),
\end{align}
where the loss function $\ell(\xb_i; \sbb)$ is defined as $$\ell(\xb_i; \sbb) = \frac{1}{T-t_0} \int_{t_0}^T \EE_{\Xb_t |\Xb_0 =  \xb_i} [\norm{\nabla_{\Xb_t} \log \phi_t(\Xb_t | \Xb_0) - \sbb(\Xb_t, t)}_2^2] \diff t.$$ Note that we have already taken $w(t) = 1/(T-t_0)$ for simplicity and assumed sufficient sampling of $\Xb_t|\xb_i$ and $t$, as they are cheap to generate. For notational convenience, we denote population loss $\cL(\cdot) = \EE_{P_{\rm data}} [\hat{\cL}(\cdot)]$.

\section{Score decomposition}\label{sec:score_decomp}
In this section, we show that for a low-dimensional data distribution, the score function can be decomposed -- each component of the score function has distinct properties. Exploiting these properties enables an efficient approximation and estimation of the score function; see Section \ref{sec:score_approx_est}.

We consider data $\xb \in \RR^D$ supported on a $d$-dimensional unknown linear subspace with $d \ll D$.
\begin{assumption}\label{assumption:subspace_data}
Data point $\xb$ can be written as $\xb = A \zb$, where $A \in \RR^{D \times d}$ is an unknown matrix with orthonormal columns. The latent variable $\zb \in \RR^d$ follows some distribution $P_z$ with a density function $p_z$.
\end{assumption} 

Given such a low dimensional structure of the data, we can show that the ground-truth score function has the following decomposition.
\begin{lemma}\label{lemma:subspace_score}
Let data $\xb = A\zb$ follows Assumption~\ref{assumption:subspace_data}. The score function $\nabla \log p_t(\xb)$ decomposes as
\begin{align*}
\nabla \log p_t(\xb) = \underbrace{A\nabla \log p_t^{\sf LD}(A^\top \xb)}_{\pscore(A^\top \xb, t) \text{:~on-support~score}}   \underbrace{- \frac{1}{h(t)} \left(I_D - AA^\top \right) \xb}_{\oscore(\xb, t) \text{:~ortho.~score}},
\end{align*}
where
\[
p_t^{\sf LD}(\zb') =  \int  \phi_t(\zb'|\zb)p_z(\zb) \diff \zb
\]
with $\phi_t( \cdot | \zb)$ being the Gaussian density function of ${\sf N}(\alpha(t)\zb, h(t)I_d)$ for $\alpha(t) = e^{-t/2}$ and $h(t) = 1 - e^{-t}$.
\end{lemma}
The proof follows from algebraic manipulation, which is deferred to Appendix \ref{pf:subspace_score}. Here $p_t^{\sf LD}$ denotes a density function on the latent space (superscript stands for ``latent distribution''). The on-support score $\sbb_{\parallel}$ belongs to the column span of $A$, depends on the projected data $A^\top \xb$, and is orthogonal to $\sbb_{\perp}$. When $t \to 0$, we can check that $\sbb_{\perp}$ will blow up since $h(t) \to 0$. This observation is consistent with the score blowup phenomenon for manifold data \citep{song2020improved, kim2021soft, pidstrigach2022score, de2022convergence}, as our linear subspace is a special type of manifolds.

The decomposition of $\nabla \log p_t$ also suggests a decomposition of the backward process. Specifically, we denote $\Ybp = AA^\top \bXb_t$ and $\Ybo = (I_D - AA^\top) \bXb_t$. Then the dynamic in \eqref{eq:backward_sde} leads to
\begin{align*}
\diff \Ybp & = \left[\frac{1}{2} \Ybp + \pscore(\Ybp, T-t) \right] \diff t + AA^\top \diff \overline{\Wb}_{t}, \\
\diff \Ybo & = \left[\frac{1}{2} - \frac{1}{h(T-t)} \right] \Ybo \diff t + (I_D - AA^\top) \diff \overline{\Wb}_{t}.
\end{align*}
A graphical illustration is provided in Figure \ref{fig:process_decomp}. The dynamics of $\Ybp$ incorporates information from the latent distribution $P_z$, while the dynamics of $\Ybo$ is linear and much simpler. The interesting part is that the coefficient in the drift term of $\Ybo$ is always negative, indicating that $\Ybo$ will vanish eventually and the data support will be perfectly recovered.
\begin{figure}[htb!]
\centering
\includegraphics[width = 0.7\textwidth]{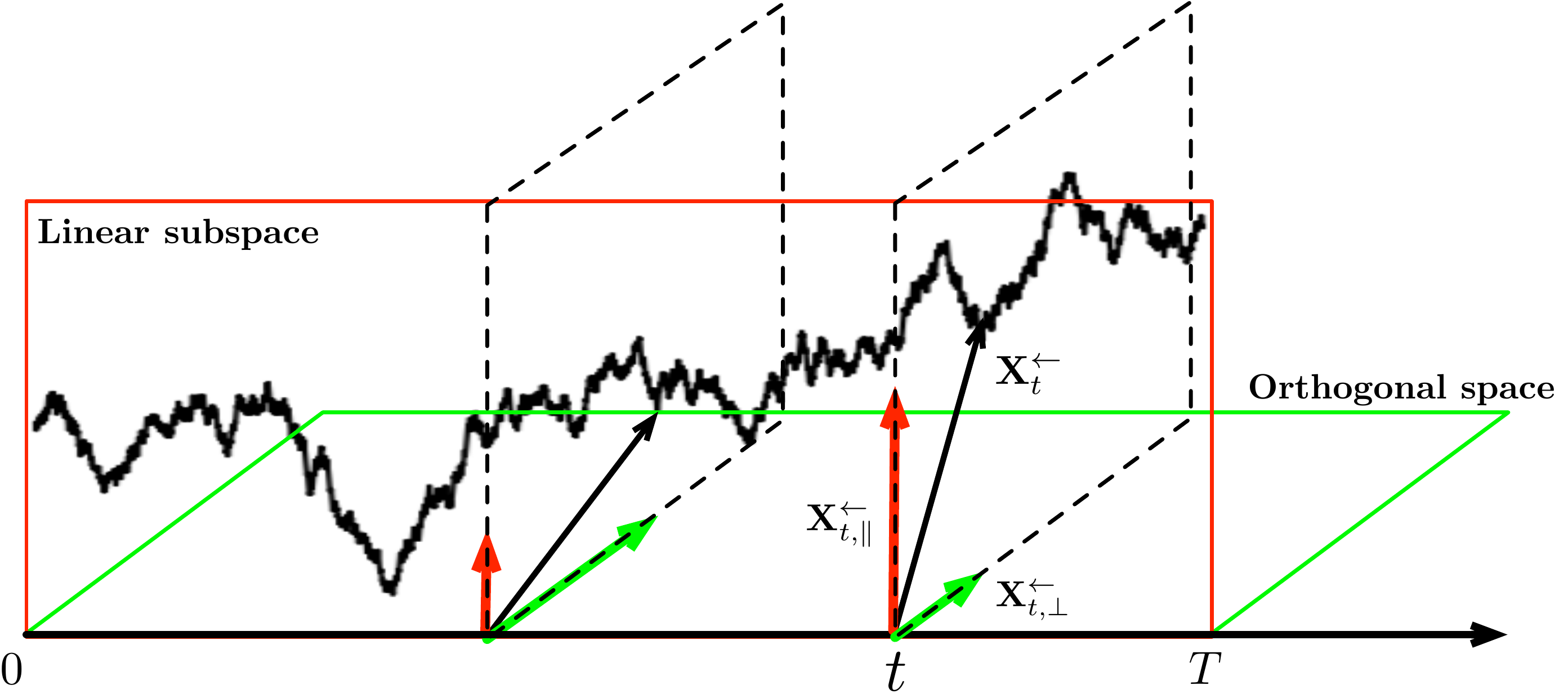}
\caption{Demonstration of score decomposition induces two backward processes.}
\label{fig:process_decomp}
\end{figure}

For better interpretation, we analyze a Gaussian example. Detailed computation is provided in Appendix \ref{pf:example_gaussian}.
\begin{example}\label{example:gaussian}
We take latent distribution $P_z = {\sf N}(\boldsymbol{0}, \Sigma)$ with $\Sigma = \diag(\lambda^2_1, \dots, \lambda_d^2) \succ 0$, a $d$-dimensional Gaussian distribution. The score function can be computed as
\begin{align*}
\nabla \log p_t(\xb) = \underbrace{- A \Sigma_t^{-1} A^\top \xb}_{\sbb_\parallel}\underbrace{  - \frac{1}{h(t)} (I_D - AA^\top)\xb}_{\sbb_{\perp}},
\end{align*}
where $\Sigma_t = \diag(\dots,  \alpha^2(t) \lambda_k^2 + h(t) , \dots)$.

One can verify that $\sbb_{\parallel}$ now is linear in $\xb$, whereas $\sbb_{\perp}$ blows up when $t$ approaches $0$. Moreover, only the on-support score $\sbb_{\parallel}$ carries the covariance information of the latent distribution and will guide the distribution recovery.

A closer evaluation further reveals $\pscore$ is Lipschitz continuous, i.e.,
\begin{align*}
\norm{\pscore(\zb_1, t) - \pscore(\zb_2, t)}_{2} \leq \max\{\lambda_d^{-2}, 1 \}\norm{\zb_1 - \zb_2}_2
\end{align*}
for any $t \in [0, T]$ and $\zb_1, \zb_2$, and
\begin{align*}
\norm{\pscore(\zb, t_1) - \pscore(\zb, t_2)}_{2} \leq \max\{\lambda_d^{-2}, 1\} \norm{\zb}_2 |t_1 - t_2|.
\end{align*}
for any $\zb$ and $t_1, t_2 \in [0, T]$. Such properties are essential to develop score approximation and estimation results.
\end{example}

\section{Score approximation and estimation}\label{sec:score_approx_est}
In practice, score functions are approximated by  neural networks. To ensure an effective learning, the network class should be expressive enough to approximate the score function. This section first establishes a score approximation theory. Built upon the approximation theory, we next provide statistical guarantees of the score matching.

\subsection{Score approximation}
We rearrange terms of $\nabla \log p_t$ in Lemma \ref{lemma:subspace_score} as
\[
     \nabla \log p_t(\xb) =  \frac{1}{h(t)} A \big(h(t)\nabla \log p_t^{\sf LD}(A^\top\xb) + A^\top \xb \big) - \frac{1}{h(t)} \xb.
\]

Accordingly, we consider score networks in the form of
\begin{align*}
\cS_{\rm NN} = \bigg\{\sbb_{V, \btheta}(\xb, t) = \frac{1}{h(t)} V {\fb}_{\btheta}(V^\top \xb, t) - \frac{1}{h(t)}\xb & :~ V \in \RR^{D \times d}~ \text{with~orthonormal~columns}, \\
& \quad {\fb}_{\btheta}: \RR^d \times [t_0, T] \to \RR^d~\text{a~ReLU~network} \bigg\}.
\end{align*}
\begin{figure}[htb!]
\centering
\includegraphics[width = 0.65\textwidth]{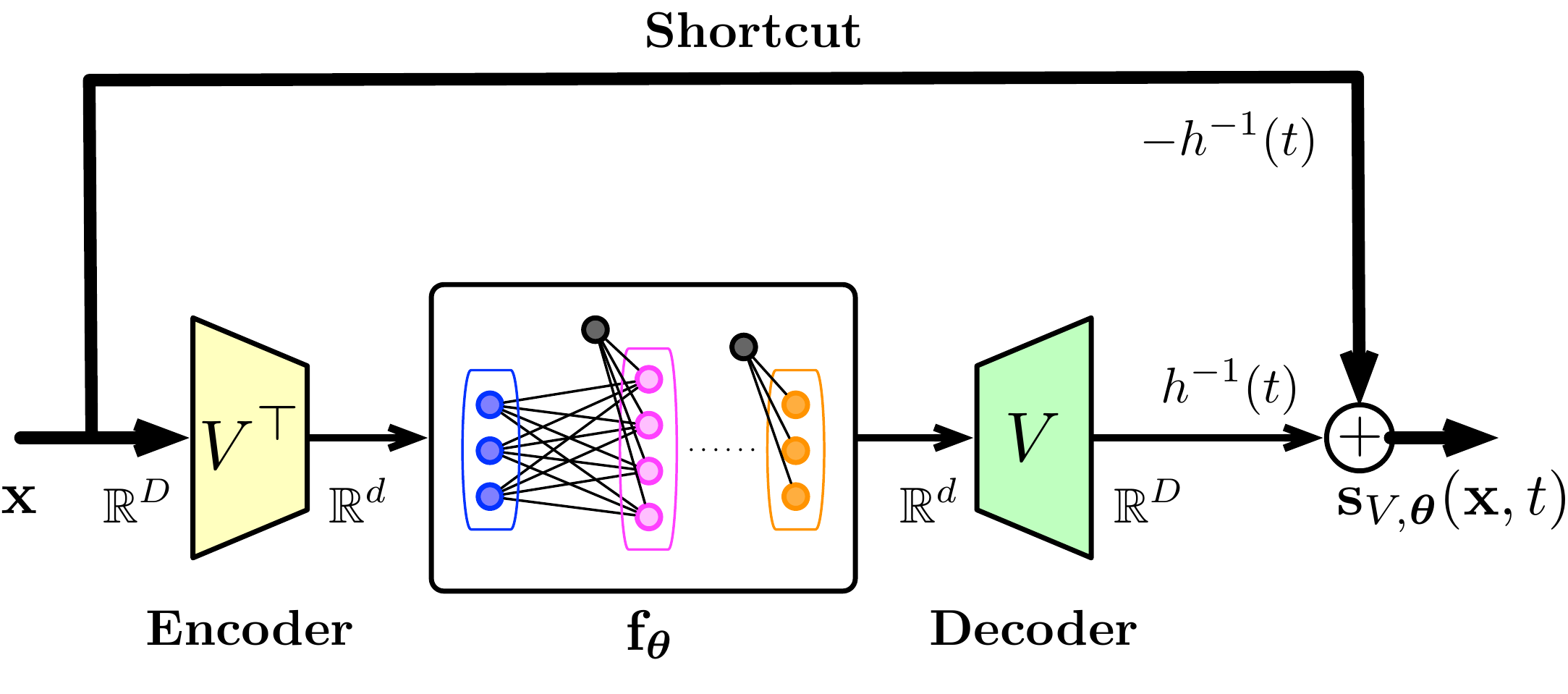}
\caption{Network architecture of $\cS_{\rm NN}$.}
\label{fig:SNN}
\end{figure}
\begin{remark} 
The network family $\cS_{\rm NN}$ resembles commonly used architectures of score networks, e.g., U-Net \citep{ronneberger2015u}: (1) $-\frac{1}{h(t)} \xb$ contributes as a shortcut connection; (2) $V\fb_{\btheta}(V^\top \xb, t)$ retains an encoder-decoder structure, where $V$, $V^\top$ are the linear decoder and encoder, respectively. See Figure \ref{fig:SNN} for an illustration of the network architecture. We will show later that through score matching, $V$ indeed recovers the unknown data subspace.
\end{remark}

We configure the ReLU network $\fb_{\btheta}$ in $\cS_{\rm NN}$ by hyperparameters. Specifically, $\fb_{\btheta} \in {\rm NN}(L, M, J, K, \kappa, \gamma, \gamma_t)$ with
\begin{align*}
{\rm NN}(L, M, J, K, \kappa, \gamma, \gamma_t) = & \Big \{\fb(\zb, t) = W_L \sigma( \dots \sigma(W_1 [\zb^\top, t]^\top + \bbb_1) \dots ) + \bbb_{L} : \\
& \quad  \text{network~width~bounded~by~}M, ~\sup_{\zb, t} \norm{\fb(\zb, t)}_2 \leq K, \\
& ~~~ \max\{\norm{\bbb_i}_\infty, \norm{W_i}_\infty\} \leq \kappa ~\text{for}~i = 1, \dots, L, \\
& ~~~ \sum_{i=1}^L \big( \norm{W_i}_0 + \norm{\bbb_i}_0 \big) \leq J, \\
& ~~~ \norm{\fb(\zb_1, t) - \fb(\zb_2, t)}_2 \leq \gamma \norm{\zb_1 - \zb_2}_2~\text{for~any~}t \in [0, T], \\
& ~~~ \norm{\fb(\zb, t_1) - \fb(\zb, t_2)}_2 \leq \gamma_t |t_1 - t_2 |~\text{for~any~}\zb \Big\},
\end{align*}
where the network width refers to the maximum dimensions of the weight matrices, $\sigma$ is the ReLU activation, and $\norm{\cdot}_\infty$ and $\norm{\cdot}_0$ denote the maximum magnitude of entries and the number of nonzero entries, respectively. In the sequel, we write $\cS_{\rm NN}(L, M, J, K, \kappa, \gamma, \gamma_t)$ to reflect the configuration of $\fb_{\btheta}$. To establish our score approximation theory, we impose an assumption on the latent distribution $P_z$.
\begin{assumption}\label{assumption:pz}
The density function $p_z > 0$ is twice continuously differentiable. Moreover, there exist positive constants $B, C_1, C_2$ such that when $\norm{\zb}_2 \geq B$, the density function $p_z(\zb) \leq (2\pi)^{-d/2} C_1 \exp(-C_2 \norm{\zb}_2^2 / 2)$.
\end{assumption}
Assumption \ref{assumption:pz} describes the tail behavior of $P_z$ being sub-Gaussian, which is commonly adopted in high-dimensional statistics literature \citep{vershynin2018high, wainwright2019high}. We also need the following regularity assumption on the score function.
\begin{assumption}\label{assumption:score_lip}
The on-support score function $\sbb_{\parallel}(\zb, t)$ is $\beta$-Lipschitz in $\zb \in \RR^d$ for any $t \in [0, T]$.
\end{assumption}
Lipschitz score functions are a standard assumption in existing literature \citep{block2020generative, lee2022convergencea, chen2022sampling}. Yet Assumption \ref{assumption:score_lip} only requires the Lipschitz continuity of the on-support score. As an example, the Gaussian data in Example \ref{example:gaussian} verifies Assumption \ref{assumption:score_lip}. We remark that $\nabla \log p_t$ itself is $(\beta+\frac{1}{h(t)})$-Lipschitz, which matches the weaker assumption of \citet[Assumption 3]{lee2022convergenceb}. When $t$ goes to zero, the Lipschitz constant of $\nabla \log p_t$ goes to infinity. 

The following theorem presents an approximation theory using $\cS_{\rm NN}$ for score functions.
\begin{theorem}\label{thm:score_approximation}
Given an approximation error $\epsilon > 0$, we choose $\cS_{\rm NN}$ with
\begin{align*}
& L = \cO\left(\log \frac{1}{\epsilon} + d \right), ~ K = \cO\left(2d^2 \log \left(\frac{d}{t_0 \epsilon}\right)\right), \\
& M = \cO\left((1+\beta)^{d}T \tau d^{d/2+1} \epsilon^{-(d+1)} \log^{d/2} \left(\frac{d}{t_0\epsilon}\right)\right), \\
& J = \cO\left((1+\beta)^d T \tau d^{d/2+1} \epsilon^{-(d+1)} \log^{d/2} \left(\frac{d}{t_0\epsilon}\right) \left(\log \frac{1}{\epsilon} + d \right)\right), \\
& \kappa = \cO\left(\max\left\{2(1+\beta) \sqrt{d\log \left(\frac{d}{t_0\epsilon}\right)}, T\tau\right\}\right), \\
& \gamma = 10d (1+\beta), ~ \gamma_t = 10\tau,
\end{align*}
where $\tau = \sup_{t \in [t_0, T]} \sup_{\norm{\zb}_\infty \leq \sqrt{d\log \frac{d}{t_0\epsilon}}} \norm{\frac{\partial}{\partial t}[h(t)\sbb_{\parallel}(\zb, t)]}_2$. Then for any data distribution $P_{\rm data}$ satisfying Assumptions 
\ref{assumption:subspace_data} -- \ref{assumption:score_lip}, there exists an $\bar{\sbb}_{V, \btheta} \in \cS_{\rm NN}$ such that for any $t \in [t_0, T]$, we have
\begin{align*}
\norm{\bar{\sbb}_{V, \btheta}(\cdot, t) - \nabla \log p_t(\cdot)}_{L^2(P_t)} \leq \frac{\sqrt{d} + 1}{h(t)} \epsilon.
\end{align*}
\end{theorem}
The proof is provided in Appendix \ref{pf:score_approximation}. Theorem \ref{thm:score_approximation} confirms the universal approximation ability of $\cS_{\rm NN}$ for score functions. A few remarks are in order.

\paragraph{Universal approximation under the $L^2$ norm} Many existing universal approximation theory of neural networks focus on approximating target functions on a compact domain under the $L^\infty$ norm \citep{yarotsky2017error, schmidt2017nonparametric, chen2019efficient, guhring2020error}. Instead, we provide an $L^2$-approximation error bound over the unbounded input domain $\RR^D$, where we tackle the unboundedness through a truncation argument. In addition, thanks to the encoder-decoder architecture, the network size only depends on the intrinsic dimension $d$ of data.

\paragraph{Lipschitz score network} Conventional universal approximation theory of neural networks hardly provide network Lipschitz continuity guarantees \citep{cybenko1989approximation, barron1993universal, yarotsky2017error}. By our construction, the Lipschitz constraints $\gamma$ and $\gamma_t$ do not undermine the approximation power of score networks. In practice, such a Lipschitz regularity is often enforced during training, e.g., adding regularization  \citep{virmaux2018lipschitz, pauli2021training, gouk2021regularisation}. Further, from a theoretical perspective, the Lipschitz property of the estimated score is essential to bounding the distribution recovery error, as we demonstrate in Section \ref{sec:distro_guarantee}.

\paragraph{Time as an additional input dimension} We take time $t$ as an additional input dimension to the score network. The network size depends on the Lipschitz constant $\tau$. We show a very coarse upper bound of $\tau$ in Appendix \ref{pf:score_approximation}. However, $\tau$ depends on the latent distribution $P_z$ and is highly instance specific. In Example \ref{example:gaussian}, we have $\tau = \cO(\sqrt{d\log \left(d/(t_0 \epsilon)\right)})$, much smaller than its coarse upper bound. More interestingly, in practice, time $t$ is embedded using sinusoidal positional encoding scheme \citep{vaswani2017attention} and the processed embedding is added to the input data. Such a dimensional lift of time opens research directions, however, the analysis is beyond the scope of this paper.

\subsection{Score estimation theory}

In this subsection, we provide sample complexity for score estimation using $\cS_{\rm NN}$. As we have parameterized the score function using deep neural networks, we can rewrite the score matching objective in \eqref{eq:score_matching_empirical} as 
\begin{align*}
\hat{\sbb}_{V, \btheta} \in \argmin_{\sbb_{V, \btheta} \in \cS_{\rm NN}} \hat{\cL}(\sbb_{V, \btheta}),
\end{align*}
where $\hat{\cL}$ is defined in \eqref{eq:score_matching_empirical}. The following theorem establishes the $L^2$ convergence of $\hat{\sbb}_{V, \btheta}$ to $\nabla \log p_t$ when the sample size $n \to \infty$.

\begin{theorem}\label{thm:score_estimation}
Suppose Assumptions \ref{assumption:subspace_data} -- \ref{assumption:score_lip} hold. We choose $\cS_{\rm NN}$ as in Theorem \ref{thm:score_approximation} with $\epsilon = n^{-\frac{1 - \delta(n)}{d+5}}$ for $\delta(n) = \frac{d\log\log n}{\log n}$. Then with probability $1 - \frac{1}{n}$, it holds
\begin{align*}
& \frac{1}{T-t_0}\int_{t_0}^T \norm{\hat{\sbb}_{V, \btheta}(\cdot, t) - \nabla \log p_t(\cdot)}_{L^2(P_t)}^2 \diff t = \tilde{\cO}\left(\frac{1}{t_0} \left(n^{-\frac{2 - 2\delta(n)}{d+5}} + Dn^{-\frac{d+3}{d+5}}\right)\log^{3} n\right),
\end{align*}
where $\tilde{\cO}$ hides factors depending on $\beta$, $\log D$, $d$, $\log t_0$ and $\tau$ defined in Theorem \ref{thm:score_approximation}.
\end{theorem}
The proof is provided in Appendix \ref{pf:score_estimation}. To the best of our knowledge, Theorem \ref{thm:score_estimation} is the first explicit sample complexity bound for score matching. The rate of convergence only depends on intrinsic dimension $d$. When $n$ is sufficiently large, $\delta(n)$ is negligible and the squared $L^2$ estimation error converges at a rate of $\tilde{\cO}(\frac{1}{t_0} n^{-\frac{2}{d+5}})$. (We hide other factors depending on $d$ in the bound to highlight the fast convergence in terms of sample size $n$.
As $d$ is often much smaller than $D$ and $n$ is large for diffusion models, those factors on $d$ do not undermine the convergence guarantee.)

Theorem \ref{thm:score_estimation} becomes vacuous if $t_0 \to 0$ when $n$ is fixed. This is a consequence of the blowup of score function $\nabla \log p_t$ as $t_0 \to 0$. Although larger $t_0$ leads to a better estimation error bound, following the backward process until a large time $t_0$ gives poor distribution recovery. In the following section, we will show a tradeoff on $t_0$.

\section{Distribution estimation}\label{sec:distro_guarantee}

This section establishes distribution estimation guarantees using the estimated score functions. Recall that in reality, diffusion models generate data using the discretized backward process \eqref{eq:backward_piecewise} with step size $\eta$. Given an estimated score function $\hat{\sbb}_{V, \btheta}$ as in Theorem \ref{thm:score_estimation}, we denote the generated distribution by $\hat{P}^{\sf dis}_{t_0}$.

We focus on three major criteria to assess the quality of $\hat{P}_{t_0}^{\sf dis}$: 1). How accurate is the subspace $A$ recovered; 2). What is the estimation error of $\hat{P}_{t_0}^{\sf dis}$ to the on-support latent distribution $P_z$; 3). What is the behavior of $\hat{P}_{t_0}^{\sf dis}$ in the orthogonal space.

Recall from Lemma \ref{lemma:subspace_score}, we denote on-support latent distribution as $P^{\sf LD}_t$ with density function $p^{\sf LD}_t$. Since we early-stop at time $t_0$, we compare the estimated distribution with $P_{t_0}^{\sf LD}$. Now we summarize our results in the following theorem.
\begin{theorem}\label{thm:distro_estimation}
Given the estimated score $\hat{\sbb}_{V, \btheta} \in \cS_{\rm NN}$ in Theorem \ref{thm:score_estimation}, we choose $T = \Theta(\log n), t_0 = {\cO}(\min\{\Czero,  1/\beta\})$, where $c_0 = \sigma_{\min}(\EE_{P_z}[\zb \zb^\top])$ is the minimum eigenvalue. Then the following items hold with probability $1 - \frac{1}{n}$.

\noindent {\bf 1)}. The unknown data subspace is recovered as
\begin{align*}
\norm{VV^\top - AA^\top}_{\rm F}^2 = \tilde{\cO}\left(\frac{1}{c_0} n^{-\frac{2-2\delta(n)}{d+5}}\log^{7/2} n\right),
\end{align*}

\noindent {\bf 2)}. Under the condition ${\sf KL}(P_z || {\sf N}(\boldsymbol{0}, I_d)) < \infty$, we choose the step size $\eta \leq \frac{t_0^2}{d} n^{-\frac{2-2\delta(n)}{d+5}}$. Recall $(VU)^\top_{\sharp} \hat{P}_{t_0}^{\sf dis}$ denotes the pushforward distribution. Then there exists an orthogonal matrix $U \in \RR^{d \times d}$ such that the total variation distance
\begin{align*}
{\sf TV} (P_{t_0}^{\sf LD}, (VU)^\top_{\sharp} \hat{P}_{t_0}^{\sf dis}) = \tilde{\cO}\left(\sqrt{\frac{1}{c_0t_0}} n^{-\frac{1-\delta(n)}{d+5}}\log^{2} n\right).
\end{align*}
Moreover, the Wasserstein-2 distance between $P_{t_0}^{\sf LD}$ and $P_z$ satisfies
\begin{align*}
{\sf W}_2(P_{t_0}^{\sf LD}, P_z) = \cO\left(\sqrt{d t_0}\right).
\end{align*}

\noindent {\bf 3)}. The orthogonal pushforward $(I-VV^\top)_{\sharp} \hat{P}_{t_0}^{\sf dis}$ of the continuous-time generated data distribution is ${\sf N}(\boldsymbol{0}, \Sigma)$, with $\Sigma \preceq ct_0 I$ for a constant $c > 0$.
\end{theorem}

The proof is provided in Appendix \ref{append:proof:de}. Theorem \ref{thm:distro_estimation} has the following interpretations.

\paragraph{Subspace recovery error} Item 1 of Theorem \ref{thm:distro_estimation} confirms that the subspace is accurately learned, in that the column span of matrix $V$ closely matches that of $A$. The error is proportional to the score estimation error and depends on the minimum eigenvalue of the covariance of $P_z$. The intuition behind is that we need $P_z$ to span every direction of column span of $A$ for estimation.

Meanwhile, item 1 does not translate to $\|A - V\|_{\rm F}$ being small, since the column span is invariant under orthogonal transformation, i.e., column spans of $A$ and $AU$ for an orthogonal $U$ are identical. Therefore, we need such an orthogonal transformation in item 2.

\paragraph{Tradeoff on $t_0$} From item 2, we observe that the latent distribution error ${\sf TV} (P_{t_0}^{\sf LD}, (VU)^\top_{\sharp} \hat{P}^{\sf dis}_{t_0})$ increases as $t_0$ decreases, because the error of score estimation amplifies. On the other hand, the bias ${\sf W}_2(P_{t_0}^{\sf LD}, P_z) = \cO\left(\sqrt{t_0d}\right)$ shrinks as $t_0$ decreases. This reveals a tradeoff concerning recovery of data distribution $P_z$. Although we cannot directly translate total variation distance to Wasserstein-2 distance and vice versa, we can make them in the same order, which implies setting $t_0 = n^{-\frac{1-\delta(n)}{d+5}}$. We thus obtain
\begin{align*}
{\sf TV} (P_{t_0}^{\sf LD}, (VU)^\top_{\sharp} \hat{P}^{\sf dis}_{t_0}) & = \tilde{\cO}\left(n^{-\frac{1-\delta(n)}{2(d + 5)}} \log^{2} n \right) \quad \text{and} \quad {\sf W}_2(P_{t_0}^{\sf LD}, P_z) = \tilde{\cO}\left(n^{-\frac{1-\delta(n)}{2(d + 5)}} \right).
\end{align*}

\paragraph{Vanishing in the orthogonal space} The behavior of $\hat{P}_{t_0}^{\sf dis}$ matches our discussion in the score decomposition. In particular, $(I-VV^\top)_{\sharp} \hat{P}_{t_0}^{\sf dis}$ degenerates to a point mass at origin when $t_0 \to 0$. Due to item 1, $(I-AA^\top)_{\sharp} \hat{P}_{t_0}^{\sf dis}$ is also approximately vanishing.

\section{Proof sketch of main results}\label{sec:proof}

This section is devoted to proving Theorems \ref{thm:score_approximation} -- \ref{thm:distro_estimation}. Due to space limit, we only describe key steps.

\subsection{Proof sketch of Theorem \ref{thm:score_approximation}}
Theorem \ref{thm:score_approximation} is established by construction. A significant difference from the existing universal approximation theories is that the input domain of $\cS_{\rm NN}$ is unbounded. We manipulate the tail behavior of $P_z$ for developing a truncation argument.

In the construction, we choose $V = A$ and the approximation of the score boils down to that of $\fb_{\btheta}(\zb, t)$ to $h(t) \nabla \log \ptz(\zb) + \zb$ for $\zb \in \RR^d$. We denote $\gb(\zb, t) = h(t)  \nabla \log \ptz(\zb) + \zb$. By Assumption \ref{assumption:score_lip}, $\gb(\zb, t)$ is $(\beta + 1)$-Lipschitz in $\zb$.
\begin{figure}[htb!]
\centering
\includegraphics[width = 0.6\textwidth]{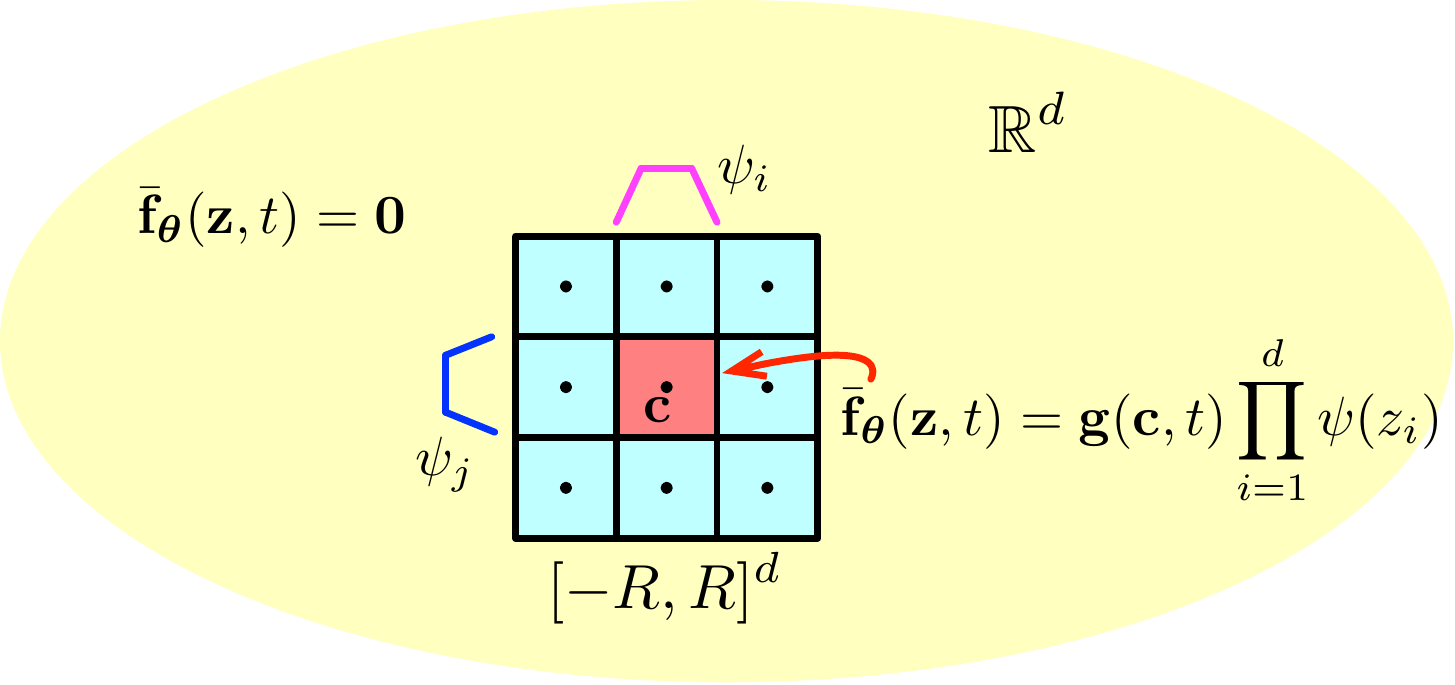}
\caption{Construction of $\bar{\fb}_{\btheta}(\zb, t)$ for approximating $\gb(\zb, t)$. For a fixed $t$, inside $[-R, R]^d$, we uniformly partition the hypercube into smaller hypercubes. On each of the small hypercube, we locally approximate $\gb(\zb, t)$ by its value on the center. To detect the local region, we construct a trapezoid function $\psi$ on each coordinate.}
\label{fig:SNN_approx}
\end{figure}

Let $R > B$ be a truncation radius. On the hypercube $[-R, R]^d \times [t_0, T]$, we construct $\bar{\fb}_{\btheta}$ as a piecewise linear function for approximating $\sbb(\zb, t)$. Outside of the hypercube, we simply set $\bar{\fb}_{\btheta} = \boldsymbol{0}$. See Figure \ref{fig:SNN_approx} for an illustration.

The $L^2$ approximation error is evaluated as
\begin{align*}
\norm{\bar{\fb}_{\btheta}(\cdot, t) - \gb(\cdot, t)}_{L^2(P_t^{\sf LD})} & \leq \underbrace{\norm{\left(\bar{\fb}_{\btheta}(\cdot, t) - \gb(\cdot, t)\right)\mathds{1}\{\norm{\cdot}_2 \leq R\}}_{L^2(P_t^{\sf LD})}}_{(A)} \\
& + \underbrace{\norm{\left(\bar{\fb}_{\btheta}(\cdot, t) - \gb(\cdot, t)\right)\mathds{1}\{\norm{\cdot}_2 > R\}}_{L^2(P_t^{\sf LD})}}_{(B)}.
\end{align*}
Term $(A)$ is directly bounded by the approximation error of $\bar{\fb}_{\btheta}$ on the hypercube. Term $(B)$ utilizes the tail behavior of $P_t$. In particular, since $\gb(\zb, t)$ is Lipschitz in $\zb$, for sufficiently large $R$, $\norm{\gb(\zb, t)}_2$ is bounded by $\cO\left(\norm{\zb}_2\right)$ whenever $\norm{\zb}_2 > R$. Consequently, term $(B)$ is bounded by
\begin{align*}
(B) = \cO\left(\int_{\norm{\zb}_2 > R} \norm{\zb}_2^2 p_t(\zb) d\zb \right).
\end{align*}
Note that Assumption \ref{assumption:pz} implies that $P_t$ has a sub-Gaussian tail. Therefore, $(B)$ can be bounded (by Lemma \ref{lemma:truncation_error}), which leads to a choice of $R = \cO\left(\sqrt{d \log \frac{d}{t_0} + \log \frac{1}{\epsilon}}\right)$. The Lipschitzness of the constructed network is analyzed by adapting \citet[Lemma 10]{chen2020statistical}.

\subsection{Proof of Theorem \ref{thm:score_estimation}}
We first focus on the equivalent objective $\cL(\hat{\sbb}_{V, \btheta})$ and then translate to the desired score matching error.

We begin with an oracle inequality for bounding $\cL(\hat{\sbb}_{V, \btheta})$:
\begin{align*}
\cL(\hat{\sbb}_{V, \btheta}) & \leq \underbrace{\cL^{\rm trunc}(\hat{\sbb}_{V, \btheta}) - (1+a)\hat{\cL}^{\rm trunc}(\hat{\sbb}_{V, \btheta})}_{(A)} + \underbrace{\cL(\hat{\sbb}_{V, \btheta}) - \cL^{\rm trunc}(\hat{\sbb}_{V, \btheta})}_{(B)} \\
& \quad + (1+a) \underbrace{\inf_{\sbb_{V, \btheta} \in \cS_{\rm NN}} \hat{\cL}(\sbb_{V, \btheta})}_{(C)},
\end{align*}
where $a > 0$ is arbitrary, and $\cL^{\rm trunc}(\hat{\sbb}_{V, \btheta})$ is a truncated loss defined as
\begin{align*}
\cL^{\rm trunc}(\hat{\sbb}_{V, \btheta}) = \EE_{\xb \sim P_{\rm data}}  \left[\ell(\xb; \hat{\sbb}_{V, \btheta})\mathds{1}\{\norm{\xb}_2 \leq R\} \diff t \right]
\end{align*}
for some radius $R > 0$ to be determined, and $\hat{\cL}^{\rm trunc}$ is the empirical counterpart of $\cL^{\rm trunc}$. We truncate on $\norm{\xb}_2$ to achieve an uniform upper bound on the loss $\cL$ for concentration. Here term $(A)$ is the statistical error due to finite samples, term $(B)$ is the truncation error, term $(C)$ reflects the approximation error of $\cS_{\rm NN}$. We bound these error terms separately.

\noindent $\bullet$ {\bf Bounding term $(A)$}. Suppose we choose $\cS_{\rm NN}$ as in Theorem \ref{thm:score_approximation} with $\epsilon$ to be determined. For term $(A)$, we apply a Bernstein-type concentration inequality (Lemma \ref{lemma:concentration_bernstein}) to obtain with probability $1 - \delta$,
\begin{align*}
(A) = \cO\left(\frac{R^2 + K^2}{t_0 a n} \log \frac{\cN(1/(t_0n), \cS_{\rm NN}, \norm{\cdot})}{\delta} + \frac{1}{n} \right).
\end{align*}
Here $\cN(\tau, \cS_{\rm NN}, \norm{\cdot})$ is the covering number of $\cS_{\rm NN}$ under a properly defined norm. (The choice of norm is involved; see details in Appendix \ref{pf:score_estimation}.)

\noindent $\bullet$ {\bf Bounding term $(B)$}. Term $(B)$ is relatively simple and share the same high-level idea of bounding the truncation error in Theorem \ref{thm:score_approximation}. As a result, we have
\begin{align*}
(B) = \cO\left(\frac{1}{t_0} K^2 R^d \exp(-C_2 R^2 / 2)\right).
\end{align*}

\noindent $\bullet$ {\bf Bounding term $(C)$}. Denote $\bar{\sbb}_{V, \btheta}$ as the approximator constructed in Theorem \ref{thm:score_approximation}, and we further decompose $(C)$ into two terms,
\begin{align*}
(C) \leq \underbrace{\hat{\cL}(\bar{\sbb}_{V, \btheta}) - (1+a)\cL^{\rm trunc}(\bar{\sbb}_{V, \btheta})}_{(C_1)} + (1+a) \underbrace{\cL^{\rm trunc}(\bar{\sbb}_{V, \btheta})}_{(C_2)}.
\end{align*}
Term $(C_1)$ is very similar to term $(A)$. Recall  $P_z$ has a light tail. Thus, with high probability, we have $(C_1) = \hat{\cL}^{\rm trunc}(\bar{\sbb}_{V, \btheta}) - (1+a)\cL^{\rm trunc}(\bar{\sbb}_{V, \btheta})$, which allows us to apply Bernstein-type concentration. Since $\bar{\sbb}_{V, \btheta}$ is independent of data, $(C_1)$ converges rather fast. Term $(C_2)$ is the approximation error and we bound it by
\begin{align*}
(C_2) = \cO\left(\frac{d}{t_0 (T-t_0)}\epsilon^2\right) + E,
\end{align*}
where $E$ is the gap between equivalent losses $\cL(\sbb_{V, \btheta})$ and $\frac{1}{T-t_0} \int_{t_0}^T \norm{\sbb_{V, \btheta}(\cdot, t) - \nabla \log p_t(\cdot)}_{L^2(P_t)}^2 \diff t$. Such a gap is independent of $\sbb_{V, \btheta}$ (see derivation in \citet[Appendix A]{chen2022sampling}).

\noindent $\bullet$ {\bf Putting $(A)$, $(B)$, $(C)$ together}. We choose $a = \epsilon^2$ and $R = \cO\left(\sqrt{d\log d + \log K + \log \frac{n}{\delta}}\right)$. Summing up $(A)$, $(B)$ and $(C)$ gives rise to
\begin{align*}
\frac{1}{T-t_0} \int_{t_0}^T \norm{\hat{\sbb}_{V, \btheta}(\cdot, t) - \nabla \log p_t(\cdot)}_{L^2(P_t)}^2 \diff t = \tilde{\cO}\left(\frac{\epsilon^{-d-3-2\delta(n)}}{t_0 n} + \frac{\epsilon^2}{t_0} + \frac{1}{n} \right).
\end{align*}
We have plugged in an upper bound on the covering number $\cN(1/(t_0 n), \cS_{\rm NN}, \norm{\cdot})$ and omitted a $\mathrm{polylog}(n)$ term. Optimally choosing $\epsilon = n^{-\frac{1-\delta(n)}{d+5}}$ yields the desired result.

\subsection{Proof of Theorem \ref{thm:distro_estimation}}

We will be succinct on how to prove items 1 and 3, and focus on the proof of item 2. The intuition behind item 1 is that the mismatch between the column span of $A$ and $V$ will be significantly amplified due to the blowup of the orthogonal score. Therefore, an accurate neural score estimator forces $A$ and $V$ to match. Item 3 can be obtained by analytically solving the orthogonal backward process.

$\bullet$ {\bf Proof of item 2}. We consider the continuous-time generated distribution $\hat{P}_{t_0}$ for an exposure of the main idea. The discrete result is obtained by adding discretization error (Lemma \ref{lem:de}). 

For the ground-truth backward process, we consider the corresponding latent backward process $\bZb_t = A^\top \bXb_t$, which satisfies the following SDE
\begin{align*}
    \diff  \bZb_t   = \left[\frac{1}{2} \bZb_t  +  \nabla  \log \pld _{T-t} (\bZb_t )  \right] \diff t +   \diff \overline{\Wb}^{\sf LD}_t,
\end{align*}
where $\overline{\Wb}_t^{\sf LD}$ is a standard Wiener process in the latent space. 

For the learned process, similarly we consider ${\tildebZbRot}_{t } = U^\top V^\top {\tildebXb}_t $. We first show that $({\tildebZbRot}_{t })_{t\geq 0}$ satisfies the following SDE 
\begin{align*}
    \diff {\tildebZbRot}_{t }  & = \left[\frac{1}{2} {\tildebZbRot}_{t }  + \sld_{\theta, U}({\tildebZbRot}_{t },T-t)  \right] \diff t +   \diff  \overline{\Wb}^{\sf LD}_t,
\end{align*}
where $\tilde{\sbb}_{U, \btheta}^{\sf LD}(\zb, t) = \frac{1}{h(t)}[U^\top \fb_{\btheta} (U \zb, t) - \zb]$ is the latent score estimator.

Observe that $P_{t_0}^{\sf LD}$ is the marginal distribution of $\bZb_{T-t_0}$, and $(VU)_{\sharp}^\top \hat{P}_{t_0}$ is the marginal distribution of  ${\tildebZbRot}_{T-t_0}$. 
To this end, it suffices to bound the divergence between the two stochastic processes above. In the proof, we first convert the score matching error bound to the latent score matching error between $ \nabla  \log \pld _{t} (\zb ) $ and $\tilde{\sbb}_{U, \btheta}^{\sf LD}(\zb, t)$. Then, similar to \citet{chen2022sampling}, we adopt Girsanov's Theorem and bound the difference of the KL divergence of the two process via the error bound of their drift terms.

\section{Conclusion and discussion}\label{sec:discuss}

This paper studies distribution estimation of diffusion models for low-dimensional linear subspace data. We show that with a properly chosen neural network, the score function can be accurately approximated and estimated. The estimation error converges at a rate depending on the data intrinsic dimension. We further show data distribution can be efficiently learned using the estimated score function. The convergence rate is also free of the curse of ambient dimensionality.

\paragraph{Linear subspace assumption} Diffusion models are very new in the field of machine learning theory. The theoretical analysis has been very challenging, especially when taking the intrinsic geometric structures of the data into consideration. Although we make a linear subspace assumption, characterizing the behavior of diffusion models in the on-support and orthogonal subspaces has already required highly non-trivial analysis. We expect to stimulate more sophisticated followup works, which analyze diffusion models under more general assumptions such as manifold data.

\paragraph{End-to-end distribution learning} Given our linear subspace assumption, one may advocate PCA-like methods, which first reduce the data dimension by estimating the subspace structure, and then estimate the
data distribution on a projected subspace. However, such a two-step method is rarely used in practice, and does not necessarily help us understand the empirical success of diffusion models. On the contrary, our results consider a more realistic end-to-end learning scheme, and show that the learned diffusion model can capture the unknown linear structure and the data distribution, and enjoy fast distribution estimation guarantees with a proper score network.

\bibliography{ref.bib}
\bibliographystyle{ims}

\newpage
\appendix

\section{Omitted proofs in Section \ref{sec:score_decomp}}
\subsection{Proof of Lemma \ref{lemma:subspace_score}}\label{pf:subspace_score}
\begin{proof}
Using the latent variable $\zb$ and according to the forward process \eqref{eq:forward_sde}, we have
\begin{align*}
p_t(\xb) = \int \phi_t(\xb | A\zb) p_z(\zb) \diff \zb,
\end{align*}
where $\phi_t(\xb | A\zb) = (2\pi)^{-D/2} h^{-D/2}(t) \exp\left(-\frac{1}{2h(t)}\norm{\alpha(t)A \zb - \xb}_2^2 \right).$
Then the score function can be written as
\begin{align}\label{eq:nabla_pt}
\nabla \log p_t(\xb) = \frac{\nabla \int \phi_t(\xb | A\zb) p_z(\zb) \diff \zb}{\int \phi_t(\xb | A\zb) p_z(\zb) \diff \zb} = \frac{\int \nabla \phi_t(\xb | A\zb) p_z(\zb) \diff \zb}{\int \phi_t(\xb | A\zb) p_z(\zb) \diff \zb},
\end{align}
where the last equality holds since $\phi_t(\xb | A\zb)$ is continuously differentiable in $\xb$. Substituting $\phi_t(\xb | A\zb)$ into \eqref{eq:nabla_pt} gives rise to
\begin{align*}
\nabla \log p_t(\xb)
& = \frac{(2\pi)^{-D/2} h^{-D/2}(t)}{\int \phi_t(\xb | A\zb) p_z(\zb) \diff \zb} \int \frac{1}{h(t)} \left(\alpha(t)A\zb - \xb\right) \exp\left(-\frac{1}{2h(t)}\norm{\alpha(t)A \zb - \xb}_2^2 \right) p_z(\zb) \diff \zb \\
& = \frac{(2\pi)^{-D/2} h^{-D/2}(t)}{\int \phi_t(\xb | A\zb) p_z(\zb) \diff \zb} \int \frac{1}{h(t)} \left(\alpha(t)A\zb - AA^\top \xb\right) \exp\left(-\frac{1}{2h(t)}\norm{\alpha(t)A \zb - \xb}_2^2 \right) p_z(\zb) \diff \zb \\
& \quad - \frac{(2\pi)^{-D/2} h^{-D/2}(t)}{\int \phi_t(\xb | A\zb) p_z(\zb) \diff \zb} \int \frac{1}{h(t)} \left(I_D - AA^\top\right)\xb \cdot  \exp\left(-\frac{1}{2h(t)}\norm{\alpha(t)A \zb - \xb}_2^2 \right) p_z(\zb) \diff \zb \\
& = \underbrace{\frac{1}{\int \phi_t(\xb | A\zb) p_z(\zb) \diff \zb} \int \frac{1}{h(t)} \left(\alpha(t)A\zb - AA^\top \xb\right) \phi_t(\xb | A\zb) p_z(\zb) \diff \zb}_{\pscore} \underbrace{-  \frac{1}{h(t)}\left(I_D - AA^\top\right) \xb}_{\oscore}.
\end{align*}
We can further simplify $\sbb_{\parallel}$. We decompose $\phi_t(\xb | A\zb)$ as
\begin{align*}
\phi_t(\xb | A\zb) & = (2\pi)^{-D/2} h^{-D/2}(t) \exp\left(-\frac{1}{2h(t)}\norm{\alpha(t)A \zb - AA^\top \xb + \left(I_D - AA^\top \right)\xb}_2^2 \right) \\
& = (2\pi)^{-D/2} h^{-D/2}(t) \exp\left(-\frac{1}{2h(t)}\left(\norm{\alpha(t)A \zb - AA^\top \xb}_2^2 + \norm{\left(I_D - AA^\top \right)\xb}_2^2\right) \right) \\
& = (2\pi)^{-d/2} h^{-d/2}(t) \exp\left(-\frac{1}{2h(t)}\norm{\alpha(t) \zb - A^\top \xb}_2^2\right) \\
& \quad \times (2\pi)^{-(D-d)/2} h^{-(D-d)/2}(t) \exp\left(-\frac{1}{2h(t)}\norm{\left(I_D - AA^\top \right)\xb}_2^2\right).
\end{align*}
We denote
\begin{align*}
\phi_t\left(A^\top \xb | \zb\right) & = (2\pi)^{-d/2} h^{-d/2}(t) \exp\left(-\frac{1}{2h(t)}\norm{\alpha(t) \zb - A^\top \xb}_2^2\right) \quad \text{and} \\ 
\phi_t\left((I_D - AA^\top)\xb\right) & = (2\pi)^{-(D-d)/2} h^{-(D-d)/2}(t) \exp\left(-\frac{1}{2h(t)}\norm{\left(I_D - AA^\top \right)\xb}_2^2\right)
\end{align*}
being both Gaussian densities. Substituting $\phi_t(\xb | A\zb) = \phi_t\left(A^\top \xb | \zb\right) \phi_t\left((I_D - AA^\top)\xb\right)$ into $\pscore$, we obtain
\begin{align*}
\pscore(\xb, t) = \frac{1}{\int \phi_t(A^\top \xb | \zb) p_z(\zb) \diff \zb} \int \frac{1}{h(t)} \left(\alpha(t)A\zb - AA^\top \xb\right) \phi_t(A^\top \xb | \zb) p_z(\zb) \diff \zb.
\end{align*}
As can be seen, $\pscore$ only depends on the projected data $A^\top \xb$. Therefore, it is legitimate to overload $\pscore(\xb, t)$ by $\pscore(A^\top \xb, t)$. The benefit is that the first input of $\pscore(A^\top \xb, t)$ now has the intrinsic dimension $d$. Denoting $\zb' = A^\top \xb$, we observe $\frac{1}{h(t)} (\alpha(t) \zb - A^\top \xb) \phi_t(A^\top \xb | \zb) = \nabla_{\zb'} \phi_t(\zb' | \zb)$. Therefore, we can rewrite $\sbb_{\parallel}(A^\top \xb, t) = \frac{\nabla_{\zb'}\phi_t(\zb' | \zb) p_z(\zb)}{\int \phi_t(\zb' | \zb) p_z(\zb) \diff \zb} \diff \zb = A \nabla \log p_t^{\rm ld}(A^\top \xb)$. The proof is complete.
\end{proof}

\subsection{Computation in Example \ref{example:gaussian}}\label{pf:example_gaussian}
We find the marginal distribution $P_t$ of the forward process is still Gaussian. Density function $p_t(\xb) = \int \phi_t(A^\top \xb | \zb) p_z(\zb) \diff \zb$. We check
\begin{align*}
\phi_t(A^\top \xb | \zb) p_z(\zb) & \propto \exp\left(-\frac{1}{2h(t)} \norm{A^\top \xb - \alpha(t) \zb}_2^2 - \zb^\top \Sigma^{-1} \zb \right) \\
& \propto \exp\left(-\frac{1}{2h(t)} \norm{\zb - \alpha(t)\left(\alpha^2(t) I_d + h(t)\Sigma^{-1}\right)^{-1} A^\top \xb}_{\left(\alpha^2(t)I_d + h(t)\Sigma^{-1} \right)^{-1}}^2 \right),
\end{align*}
where $\norm{\xb}_{A} = \xb^\top A \xb$. Therefore, $\phi_t(A^\top\xb | \zb) p_z(\zb)$ corresponds to a Gaussian distribution with mean vector $\alpha(t) \left(\alpha^2(t) I_d + h(t)\Sigma^{-1} \right)^{-1}A^\top \xb$. To this end, Lemma \ref{lemma:subspace_score} leads to
\begin{align*}
\pscore(A^\top \xb, t) & = \frac{1}{h(t)} \left(\alpha^2(t) A \left(\alpha^2(t) I_d + h(t)\Sigma^{-1} \right)^{-1} A^\top \xb - AA^\top \xb\right) \\
& = \frac{1}{h(t)} A \left( \diag\left(\frac{\alpha^2(t)}{\alpha^2(t) + h(t) \lambda_1^{-2}}, \dots, \frac{\alpha^2(t)}{\alpha^2(t) + h(t) \lambda_d^{-2}}\right) - I_d \right) A^\top \xb \\
& = A ~\diag\left(\frac{\lambda_1^{-2}}{\alpha^2(t) + h(t) \lambda_1^{-2}}, \dots, \frac{\lambda^{-2}_d}{\alpha^2(t) + h(t) \lambda_1^{-2}}\right) A^\top \xb \\
& = A ~\diag\left(\frac{1}{\alpha^2(t) \lambda_1^2 + h(t)}, \dots, \frac{1}{\alpha^2(t) \lambda_d^2 + h(t)}\right) A^\top \xb.
\end{align*}
Lastly, we check $\sbb_{\parallel}$ is Lipschitz continuous. We need to upper bound
\begin{align*}
\norm{\diag\left(\frac{1}{\alpha^2(t) \lambda_1^2 + h(t)}, \dots, \frac{1}{\alpha^2(t) \lambda_d^2 + h(t)}\right)}_{\rm op} \leq \frac{1}{\alpha^{2}(t) \lambda_d^2 + h(t)} = \frac{1}{\lambda_d^2 + (1-\lambda_d^2) h(t)}.
\end{align*}
We discuss two cases. If $\lambda_d > 1$, we have $\frac{1}{\lambda_d^2 + (1-\lambda_d^2) h(t)} \leq 1$; if $\lambda_d \leq 1$, we have $\frac{1}{\lambda_d^2 + (1-\lambda_d^2) h(t)} \leq \lambda_d^{-2}$. Combining the two cases gives rise to
\begin{align*}
\norm{\diag\left(\frac{1}{\alpha^2(t) \lambda_1^2 + h(t)}, \dots, \frac{1}{\alpha^2(t) \lambda_d^2 + h(t)}\right)}_{\rm op} \leq \max\{\lambda_d^{-2}, 1\}.
\end{align*}
For the Lipschitzness with respect to $t$, we take time derivative of $\diag\left(\frac{1}{\alpha^2(t) \lambda_1^2 + h(t)}, \dots, \frac{1}{\alpha^2(t) \lambda_d^2 + h(t)}\right)$:
\begin{align*}
\frac{\partial}{\partial t}\diag\left(\frac{1}{\alpha^2(t) \lambda_1^2 + h(t)}, \dots, \frac{1}{\alpha^2(t) \lambda_d^2 + h(t)}\right) & = \diag\left(\frac{\alpha^2(t)(\lambda_1^2 - 1)}{(\alpha^2(t) \lambda_1^2 + h(t))^2}, \dots, \frac{\alpha^2(t)(\lambda_d^2 - 1)}{(\alpha^2(t) \lambda_d^2 + h(t))^2}\right) \\
& \preceq \diag\left(\frac{1}{\alpha^2(t) \lambda_1^2 + h(t)}, \dots, \frac{1}{\alpha^2(t) \lambda_d^2 + h(t)}\right).
\end{align*}
Therefore, for any $t_1, t_2 \in [0, T]$ and $\zb$, we have
\begin{align*}
\norm{\sbb_{\parallel}(\zb, t_1) - \sbb_{\parallel}(\zb, t_2)}_2 & \leq \norm{\diag\left(\frac{1}{\alpha^2(t) \lambda_1^2 + h(t)}, \dots, \frac{1}{\alpha^2(t) \lambda_d^2 + h(t)}\right)\zb}_2 |t_1 - t_2| \\
& \leq \max\{\lambda_d^{-2}, 1\} \norm{\zb}_2 |t_1 - t_2|.
\end{align*}

\section{Omitted proofs in Section \ref{sec:score_approx_est}}

\subsection{Proof of Theorem \ref{thm:score_approximation}}\label{pf:score_approximation}
\begin{proof}
Due to Lemma \ref{lemma:subspace_score}, we cast score function $\nabla \log p_t(\xb)$ into
\begin{align}\label{eq:score_sbb}
\nabla \log p_t(\xb) = \frac{1}{h(t)} \underbrace{A \int \frac{\zb \phi_t(A^\top \xb | \zb) p_z(\zb)}{\int \phi_t(A^\top \xb | \zb) p_z(\zb) \diff \zb} \diff \zb}_{A\gb(A^\top \xb, t)} - \frac{1}{h(t)} \xb.
\end{align}
Note that $\gb(A^\top \xb, t) = h(t) A^\top (\sbb_{\parallel}(A^\top \xb, t) + \xb)$. It suffices to construct $V\fb_{\btheta}(V^\top \xb, t)$ for approximating $A\gb(A^\top \xb, t)$. By taking $V = A$, it further reduces to construct $\fb_{\btheta}(\zb', t)$ well approximating $\gb(\zb', t)$, where $\zb' \in \RR^d$.

A major difficulty in approximating $\gb(\zb', t)$ is that the input space $\RR^d \times [t_0, T]$ is unbounded. Here we partition $\RR^d$ into a compact subset $\cS$ and its complement. On set $\cS \times [t_0, T]$, we construct $\fb_{\btheta}$ to achieve an $L^\infty$ approximation. On the complement of $\cS$, we simply let $\fb_{\btheta}(\zb', t) = 0$. Thanks to the tail behavior of $P_z$, the $L^2$ approximation error of $\fb_{\btheta}(\zb', t)$ to $\sbb(\zb', t)$ can still be controlled.

\noindent $\bullet$ {\bf Approximation on $\cS \times [t_0, T]$}.
We choose $\cS = \{\zb' | \norm{\zb'}_\infty \leq R\}$ to be a $d$-dimensional hypercube of edge length $2R > 0$, where $R$ will be determined later. On $\cS \times [t_0, T]$, we approximate coordinate maps $g_k(\zb', t)$ of $\gb(\zb', t)$ separately, where $\gb(\zb', t) = [g_1(\zb', t), \dots, g_d(\zb', t)]^\top$. The main idea replicates {\it Lemma 10} in \citet{chen2020statistical}. To match the function domain, we first rescale the input by $\yb' = \frac{1}{2R} (\zb' + R \mathbf{1})$ and $t' = t/T$, so that the transformed input space is $[0, 1]^{d} \times [t_0/T, 1]$. Such a transformation can be exactly implemented by a single ReLU layer.

By Assumption \ref{assumption:score_lip}, on-support score $\sbb_{\parallel}(\zb', t)$ is $\beta$-Lipschitz in $\zb'$. This implies $\gb(\zb', t)$ is $1+\beta$-Lipschitz in $\zb'$. When taking the transformed inputs, $\gb(\yb', t') = \sbb(2R\yb'-R \mathbf{1}, T t')$ becomes $2R(1+\beta)$-Lipschitz in $\yb'$; so is each coordinate map. For notational simplicity, we denote $L_z = 1 + \beta$.

We also denote the Lipschitz constant of $\gb(\yb', t')$ with respect to $t$ as $T\tau(R)$, when $\yb' \in [0, 1]^d$. That is, we denote
\begin{align*}
\tau(R) = \sup_{t \in [t_0, T]} \sup_{\zb' \in [0, R]^d} \norm{\frac{\partial}{\partial t} \gb(\zb', t)}_2.
\end{align*}
A very coarse upper bound on $\tau(R)$ is computed by
\begin{align*}
\frac{\partial}{\partial t} \gb(\zb', t) & = A \int \frac{\zb \frac{\partial}{\partial t}\phi_t(\zb' | \zb) p_z(\zb)}{\int \phi_t(\zb' | \zb) p_z(\zb) \diff \zb} \diff \zb - A \int \frac{\zb \phi_t(\zb' | \zb) p_z(\zb) \int \frac{\partial}{\partial t} \phi_t(\zb' | \zb) p_z(\zb) \diff \zb}{\left(\int \phi_t(\zb' | \zb) p_z(\zb) \diff \zb\right)^2} \diff \zb \\
& \overset{(i)}{=} \frac{\alpha(t)}{h^2(t)} A \left[\EE_{P_z}\left[\zb \norm{\zb}_2^2\right] - (1+\alpha^2(t)) \Cov[\zb | \zb'] \zb'\right],
\end{align*}
where we plug in $\frac{\partial}{\partial t} \phi_t(\zb' | \zb) = \frac{\alpha(t)}{h^2(t)} \left(\norm{\zb}_2^2 - (1 + \alpha^2(t)) \zb^\top \zb' + \alpha(t) \norm{\zb'}_2^2\right) \phi_t(\zb' | \zb)$ and collect terms in $(i)$. Since $P_z$ has Gaussian tail, its third moment is bounded. By the computation in Appendix \ref{pf:conditional_cov}, we have $\norm{\Cov[\zb | \zb']}_{\rm op} \leq \frac{h^2(t)}{\alpha^2(t)} (\beta + \frac{1}{h(t)})$. Therefore, we deduce
\begin{align*}
\tau(R) = \cO\left(\frac{1+\alpha^2(t)}{\alpha(t)} \left(\beta + \frac{1}{h(t)}\right) \sqrt{d} R\right) = \cO\left(e^{T/2}\beta R \sqrt{d}\right),
\end{align*}
as $P_z$ having sub-Gaussian tail implies $\EE_{P_z}[\zb \norm{\zb}_2^2]$ is bounded.

Now we form a partition of $[0, 1]^{d} \times [t_0/T, 1]$. For the first $d$ dimension, we uniformly partition $[0, 1]^d$ into nonoverlapping hypercubes with edge length $e_1$. We also evenly partition the interval $[t_0/T, 1]$ into nonoverlapping subintervals of length $e_2$. $e_1$ and $e_2$ will be chosen depending on the desired approximation error. We also denote $N_1 = \lceil \frac{1}{e_1} \rceil$ and $N_2 = \lceil \frac{1}{e_2} \rceil$.

Let $\mb = [m_1, \dots, m_d]^\top \in \{0, \dots, N_1-1\}^d$ be a multi-index. We define $\bar{f}$ as
\begin{align*}
\bar{f}_i(\yb', t') = \sum_{\mb, j = 0, \dots, N_2-1} g_i\left(2R\frac{\mb}{N_1} -R\mathbf{1}, T \frac{j}{N_2}\right) \Psi_{\mb, j}(\yb', t'),
\end{align*}
where $\Psi_{\mb, j}(\yb', t')$ is a partition of unity function. We choose $\Psi$ as a product of coordinatewise trapezoid functions:
\begin{align*}
\Psi_{\mb, j} (\yb', t') = \psi\left(3N_2 \left(t' - \frac{j}{N_2}\right)\right) \prod_{i=1}^d \psi\left(3N_1\left(y'_i - \frac{m_i}{N_1}\right)\right),
\end{align*}
where $\psi$ is a trapezoid function (see also a graphical illustration in Figure \ref{fig:trapezoid}),
\begin{align*}
\psi(a) = \begin{cases}
1, & \vert a\vert < 1 \\
2 - \vert a\vert, & \vert a\vert \in [1, 2] \\
0, & \vert a\vert > 2 \\
\end{cases}.
\end{align*}

\begin{figure}[!htb]
\centering
\includegraphics[width = 0.5\textwidth]{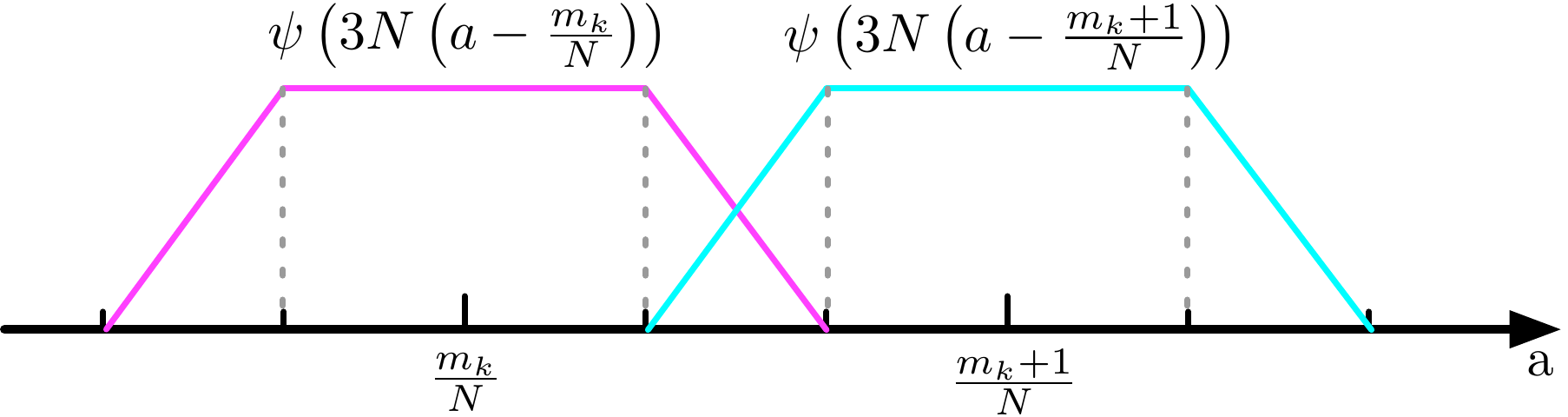}
\caption{Trapezoid function in one dimension.}
\label{fig:trapezoid}
\end{figure}

We claim that
\begin{enumerate}
\item $\bar{f}_i$ is an approximation to $g_i$;
\item $\bar{f}_i$ can be implemented by a ReLU neural network $\hat{f}_i$ with small error.
\end{enumerate}
Both claims are verified in \citet[Lemma 10]{chen2020statistical}, where we only need to substitute the Lipschitz coefficients $2cR(1+\beta)$ and $T\tau(R)$ into the error analysis. (We use the coordinate wise analysis in the proof of \citet[Lemma 10]{chen2020statistical} for deriving the Lipschitz continuity w.r.t. $\yb'$ and $t'$.) By concatenating $\bar{f}_i$'s together, we construct $\bar{\fb}_{\btheta} = [\bar{f}_1, \dots, \bar{f}_d]^\top$. Given $\epsilon$, if we achieve
\begin{align*}
\sup_{\yb', t' \in [0, 1]^d \times [t_0/T, 1]} \norm{\bar{\fb}_{\btheta}(\yb', t') - \gb(\yb', t')}_\infty \leq \epsilon,
\end{align*}
the neural network configuration is
\begin{align*}
& L = \cO\left(\log \frac{1}{\epsilon} + d \right), \quad M = \cO\left(T \tau(R) (R L_z)^{d} \epsilon^{-(d+1)}\right), \quad J = \cO\left(T \tau(R) (R L_z)^{d} \epsilon^{-(d+1)} \left(\log \frac{1}{\epsilon} + d \right)\right), \\
& \hspace{1.7in} K = \cO\left(\sqrt{d}RL_z\right), \quad \kappa = \max\{1, RL_z, T\tau(R)\}.
\end{align*}
Here we already take $e_1 = \cO\left(\frac{\epsilon}{RL_z}\right)$ and $e_2 = \cO\left(\frac{\epsilon}{T\tau(R)}\right)$. The output range $K$ is computed by $K = \sqrt{d} \max_i \norm{s_k}_\infty$. Combining with the input transformation layer (i.e., $\zb' \to \yb'$ and $t \to t'$ rescaling), we have the constructed network is Lipschitz continuous in $\zb'$, i.e., for any $\zb'_1, \zb'_2 \in \cS$ and $t \in [t_0, T]$, it holds
\begin{align*}
\norm{\bar{\fb}_{\btheta}(\zb'_1, t) - \bar{\fb}_{\btheta}(\zb'_2, t)}_{\infty} \leq 10d L_z \norm{\zb'_1 - \zb'_2}_2.
\end{align*}
Moreover, the network is also Lipschitz in $t$, i.e., for any $t_1, t_2 \in [t_0, T]$ and $\norm{\zb'}_2 \leq R$, it holds
\begin{align*}
\norm{\bar{\fb}_{\btheta}(\zb', t_1) - \bar{\fb}_{\btheta}(\zb', t_2)}_{\infty} \leq 10 \tau(R) \norm{t_1 - t_2}_2.
\end{align*}
Due to the partition of unity function $\Psi$ vanishes outside $\cS$, we have $\bar{\fb}_{\btheta}(\zb', t) = \boldsymbol{0}$ for $\norm{\zb'}_2 > R$. Therefore, the above Lipschitz continuity in $\zb'$ extends to whole $\RR^d$.

\noindent $\bullet$ {\bf Bounding $L^2$ approximation error}.
The $L^2$ approximation error of $\bar{\fb}_{\btheta}$ can be decomposed into two terms,
\begin{align*}
\norm{\gb(\zb', t) - \bar{\fb}_{\btheta}(\zb', t)}_{L^2(P_t^{\sf LD})} = \norm{(\gb(\zb', t) - \bar{\fb}_{\btheta}(\zb', t) \mathds{1}\{\norm{\zb'}_2 < R\}}_{L^2(P_t^{\sf LD})} + \norm{\gb(\zb', t) \mathds{1}\{\norm{\zb'}_2 > R\}}_{L^2(P_t^{\sf LD})}.
\end{align*}
The first term on the right-hand side of the last display is bounded by
\begin{align*}
\norm{(\gb(\zb', t) - \bar{\fb}_{\btheta}(\zb', t) \mathds{1}\{\norm{\zb'}_2 < R\}}_{L^2(P_t^{\sf LD})} \leq \sqrt{d} \sup_{\zb', t \in \cS \times [t_0, T]} \norm{\gb(\zb', t) - \bar{\fb}_{\btheta}(\zb', t)}_\infty \leq \sqrt{d} \epsilon.
\end{align*}
The second term assumes an upper bound in Lemma \ref{lemma:truncation_error}. Specifically, when choosing $R = \cO\left(\sqrt{d\log \frac{d}{t_0} + \log \frac{1}{\epsilon}}\right)$, we have
\begin{align*}
\norm{\gb(\zb', t) \mathds{1}\{\norm{\zb'}_2 > R\}}_{L^2(P_t^{\sf LD})} \leq \epsilon.
\end{align*}
As a result, with the choice of $R$, we obtain
\begin{align*}
\norm{\gb(\zb', t) - \bar{\fb}_{\btheta}(\zb', t)}_{L^2(P_t^{\sf LD})} \leq (\sqrt{d} + 1)\epsilon.
\end{align*}
Substituting $R$ into the network configuration and $\tau(R)$ denoted as $\tau$, we obtain
\begin{align*}
& \hspace{1in} L = \cO\left(\log \frac{1}{\epsilon} + d \right), \quad M = \cO\left((1+\beta)^d T \tau d^{d/2+1} \epsilon^{-(d+1)} \log^{d/2} \frac{d}{t_0\epsilon}\right), \\
& \hspace{1.2in} J = \cO\left((1+\beta)^d T \tau d^{d/2+1} \epsilon^{-(d+1)} \log^{d/2} \frac{d}{t_0\epsilon} \left(\log \frac{1}{\epsilon} + d \right)\right), \\
& \hspace{0.3in} K = \cO\left((1+\beta)d \log^{1/2} \frac{d}{t_0 \epsilon}\right), \quad \kappa = \max\left\{(1+\beta)\sqrt{d\log \frac{d}{t_0\epsilon}}, T\tau\right\}, \quad \gamma = 10d(1+\beta), \quad \gamma_t = 10\tau.
\end{align*}
The constructed approximator to $\nabla \log p_t$ is $\bar{\sbb}_{V, \btheta} = \frac{1}{h(t)} A \bar{\fb}_{\btheta}(A^\top \xb, t) - \frac{1}{h(t)} \xb$, whose $L^2$ approximation error is
\begin{align*}
\norm{\nabla \log p_t(\cdot, t) - \bar{\sbb}_{V, \btheta}(\cdot, t)}_{L^2(P_t)} \leq \frac{\sqrt{d} + 1}{h(t)} \epsilon
\end{align*}
for $t \in [t_0, T]$.
\end{proof}

\subsection{Proof of Theorem \ref{thm:score_estimation}}\label{pf:score_estimation}
\begin{proof}
The proof is based on the following oracle inequality to decompose $\cL(\hat{\sbb}_{V, \btheta})$.

\noindent $\bullet$ {\bf Oracle inequality}.
For any $a \in (0, 1)$, we decompose $\cL(\hat{\sbb}_{V, \btheta})$ as
\begin{align*}
\cL(\hat{\sbb}_{V, \btheta}) & = \cL(\hat{\sbb}_{V, \btheta}) - (1+a)\hat{\cL}(\hat{\sbb}_{V, \btheta}) + (1+a)\hat{\cL}(\hat{\sbb}_{V, \btheta}) \\
& \overset{(i)}{\leq} \cL^{\rm trunc}(\hat{\sbb}_{V, \btheta}) - (1+a)\hat{\cL}^{\rm trunc}(\hat{\sbb}_{V, \btheta}) + \cL(\hat{\sbb}_{V, \btheta}) - \cL^{\rm trunc}(\hat{\sbb}_{V, \btheta}) + (1+a)\hat{\cL}(\hat{\sbb}_{V, \btheta}) \\
& = \underbrace{\cL^{\rm trunc}(\hat{\sbb}_{V, \btheta}) - (1+a)\hat{\cL}^{\rm trunc}(\hat{\sbb}_{V, \btheta})}_{(A)} + \underbrace{\cL(\hat{\sbb}_{V, \btheta}) - \cL^{\rm trunc}(\hat{\sbb}_{V, \btheta})}_{(B)} + (1+a) \underbrace{\inf_{\sbb_{V, \btheta} \in \cS_{\rm NN}} \hat{\cL}(\sbb_{V, \btheta})}_{(C)}.
\end{align*}
where in $(i)$, $\cL^{\rm trunc}$ is defined as
\begin{align*}
\cL^{\rm trunc}(\hat{\sbb}_{V, \btheta}) = \EE_{\xb \sim P_{\rm data}} \left[\ell^{\rm trunc}(\xb; \hat{\sbb}_{V, \btheta})\right] = \EE_{\xb \sim P_{\rm data}}  \left[\ell(\xb; \hat{\sbb}_{V, \btheta})\mathds{1}\{\norm{\xb}_2 \leq R\} dt \right],
\end{align*}
for some radius $R > B$ to be determined. In the sequel, we bound $(A)$ -- $(C)$ separately.

\noindent $\star$ {\bf Bounding term $(A)$}. This term measures the concentration of the empirical loss to its population counterpart. We denote $\cG = \{\ell^{\rm trunc}(\cdot; \sbb_{V, \btheta}) : \sbb_{V, \btheta} \in \cS_{\rm NN}\}$ as an induced function class of score network $\cS_{\rm NN}$. We first determine an upper bound on $\cG$. For any $\sbb_{V, \btheta} \in \cS_{\rm NN}$, we have
\begin{align*}
\ell^{\rm trunc} (\xb; \sbb_{V, \btheta}) & = \frac{1}{T-t_0} \int_{t_0}^T \EE_{\xb' \sim \phi_t(\xb' | \xb)} \left[\norm{\sbb_{V, \btheta}(\xb', t) - \nabla \log \phi_t(\xb' | \xb)}_2^2 \mathds{1}\{\norm{\xb}_2 \leq R\}\right] \diff t \\
& = \frac{1}{T-t_0} \int_{t_0}^T \EE_{\xb' \sim \phi_t(\xb' | \xb)} \left[\norm{\sbb_{V, \btheta}(\xb', t) + \frac{1}{h(t)} (\xb' - \alpha(t)\xb)}_2^2 \mathds{1}\{\norm{\xb}_2 \leq R\}\right] \diff t \\
& \leq \frac{2}{T-t_0} \int_{t_0}^T \left(\sup_{\xb'} \norm{\sbb_{\btheta}(\xb', t) + \frac{1}{h(t)} \xb'}_2^2 + \norm{\frac{\alpha(t)}{h(t)} \xb}_2^2\right) \mathds{1}\{\norm{\xb}_2 \leq R\} \diff t \\
& = \frac{2}{T-t_0} \int_{t_0}^T \left(\sup_{\xb'} \norm{\frac{1}{h(t)}V\fb_{\btheta}(V^\top \xb', t)}_2^2 + \norm{\frac{\alpha(t)}{h(t)} \xb}_2^2\right) \mathds{1}\{\norm{\xb}_2 \leq R\} \diff t \\
& \overset{(i)}{\leq} \frac{K^2 + R^2}{T-t_0} \int_{t_0}^T  \frac{2}{h^2(t)} \diff t \\
& = \cO\left(\frac{K^2 + R^2}{t_0(T-t_0)}\right),
\end{align*}
where inequality $(i)$ invokes the uniform upper bound of $\cS_{\rm NN}$. Moreover, suppose given $\sbb_{V_1, \btheta_1}$ and $\sbb_{V_2, \btheta_2}$ with $\sup_{\norm{\xb}_2 \leq 3R + \sqrt{D\log D}, t \in [t_0, T]} \norm{\sbb_{V_1, \btheta_1}(\xb, t) - \sbb_{V_2, \btheta_2}(\xb, t)}_2 \leq \iota$. We evaluate
\begin{align*}
& \quad \norm{\ell^{\rm trunc}(\cdot; \sbb_{V_1, \btheta_1}) - \ell^{\rm trunc}(\cdot; \sbb_{V_2, \btheta_2})}_{\infty} \\
& = \sup_{\norm{\xb}_2 \leq R}  \frac{1}{T-t_0} \int_{t_0}^T \EE_{\xb' \sim \phi_t(\xb' | \xb)} \left[\norm{\sbb_{V_1, \btheta_1}(\xb', t) - \sbb_{V_2, \btheta_2}(\xb', t)}_2 \norm{\sbb_{V_1, \btheta_1}(\xb', t) - \sbb_{V_2, \btheta_2}(\xb', t) - 2\nabla \log \phi_t(\xb' | \xb)}_2\right] \diff t \\
& \leq \sup_{\norm{\xb}_2 \leq R} \frac{2(K + R)}{T-t_0} \int_{t_0}^T \frac{1}{h(t)} \EE_{\xb' \sim \phi_t(\xb' | \xb)} \left[\norm{\sbb_{V_1, \btheta_1}(\xb', t) - \sbb_{V_2, \btheta_2}(\xb', t)}_2  \mathds{1}\{\norm{\xb'}_2 \leq 3R + \sqrt{D\log D}\}\right] \diff t \\
& \quad + \sup_{\norm{\xb}_2 \leq R} \frac{2(K + R)}{T-t_0} \int_{t_0}^T \frac{1}{h(t)} \EE_{\xb' \sim \phi_t(\xb' | \xb)} \left[\norm{\sbb_{V_1, \btheta_1}(\xb', t) - \sbb_{V_2, \btheta_2}(\xb', t)}_2 \mathds{1}\{\norm{\xb'}_2 > 3R + \sqrt{D\log D}\}\right] \diff t \\
& \leq \frac{2\iota}{T-t_0} (K + R) \int_{t_0}^T \frac{1}{h(t)} \diff t \\
& \quad + \sup_{\norm{\xb}_2 \leq R} \frac{2(K + R)}{T-t_0} \int_{t_0}^T \frac{1}{h(t)} \EE_{\xb' \sim \phi_t(\xb' | \xb)} \left[\norm{\sbb_{V_1, \btheta_1}(\xb', t) - \sbb_{V_2, \btheta_2}(\xb', t)}_2 \mathds{1}\{\norm{\xb'}_2 > 3R + \sqrt{D\log D}\}\right] \diff t \\
& \leq \frac{\iota}{T-t_0} (K + R) \int_{t_0}^T \frac{1}{h(t)} \diff t + \frac{4(K+R)K}{T-t_0} \int_{t_0}^T \frac{1}{h^2(t)} \diff t \int_{\norm{\xb'}_2 > 3R + \sqrt{D\log D}} \phi_t(\xb' | \xb) \diff \xb' \\
& \overset{(i)}{=} \cO\left(\frac{\iota}{T-t_0} (K + R) \log \frac{T}{t_0} + \frac{4K(K+R)}{t_0(T - t_0)} D (3R + 2\sqrt{D\log D})^{D-2} \exp\left(-\frac{1}{2h(t)} \left(2R^2 + \frac{1}{2} D \log D\right)\right)\right) \\
& = \cO\left(\frac{\iota}{T-t_0} (K + R) \log \frac{T}{t_0} + \frac{4K(K+R)}{t_0(T - t_0)} (R/D)^{D-2} \exp\left(-\frac{1}{h(t)} R^2\right)\right),
\end{align*}
where in $(i)$, we upper bound $\phi_t(\xb' | \xb) \leq (2\pi h(t))^{-D/2} \exp\left(-\frac{1}{2h(t)}\left(\frac{1}{2}\norm{\xb'}_2^2 - \norm{\xb}_2^2\right)\right)$ and invoke Lemma \ref{lemma:gaussian_tail}. Denote $\eta = \frac{4K(K+R)}{t_0(T - t_0)} (R/D)^{D-2} \exp\left(-\frac{1}{h(t)} R^2\right)$. The last display above indicates that an $\iota$-covering of $\cS_{\rm NN}$ induces a $\frac{\iota}{T-t_0} (K + R) \log \frac{T}{t_0} + \eta$-covering of $\cG$. Now we apply Lemma \ref{lemma:concentration_bernstein} and obtain with probability $1 - \delta$,
\begin{align*}
(A) = \cO\left(\frac{(1+3/a)(K^2 + R^2)}{n t_0(T-t_0)} \log \frac{\cN\left(\frac{(T-t_0)(\iota-\eta)}{(K + R)\log (T/t_0)}, \cS_{\rm NN}, \norm{\cdot}_2 \right)}{\delta} + (2+a)\tau\right).
\end{align*}
We emphasize that norm in the covering of $\cS_{\rm NN}$ is $\sup_{\norm{\xb}_2 \leq 3R + \sqrt{D\log D}} \norm{\sbb_{V, \btheta}(\xb, t)}_2$.

\noindent $\star$ {\bf Bounding term $(B)$}. By the truncation, we have
\begin{align*}
(B) & = \EE_{\xb \sim P_{\rm data}}\left[\ell(\xb; \hat{\sbb}_{V, \btheta}) \mathds{1}\{\norm{\xb}_2 > R\}\right] \\
& = \frac{1}{T - t_0}\int_{t_0}^T \EE_{\xb \sim P_{\rm data}} \left[\EE_{\xb' \sim \phi_t(\xb' | \xb)} \left[\norm{\hat{\sbb}_{V, \btheta}(\xb', t) - \nabla \log \phi_t(\xb' | \xb)}_2^2 \right] \mathds{1}\{\norm{\xb}_2 > R\}\right] \diff t \\
& \leq \frac{2}{T-t_0} \int_{t_0}^T \EE_{\xb \sim P_{\rm data}} \left[\EE_{\xb' \sim \phi_t(\xb' | \xb)} \left(\norm{\hat{\sbb}_{V, \btheta}(\xb', t) + \frac{1}{h(t)}\xb'}_2^2 + \norm{\frac{\alpha(t)}{h(t)}\xb}_2^2\right) \mathds{1}\{\norm{\xb}_2 > R\} \right] \diff t \\
& \leq \frac{2}{T-t_0} \int_{t_0}^T \frac{1}{h^2(t)} \EE_{\xb \sim P_{\rm data}} \left[\left(K^2 + \norm{\xb}_2^2\right) \mathds{1}\{\norm{\xb}_2 > R\} \right] \diff t \\
& \overset{(i)}{\leq} \frac{2}{T-t_0}\left(C_1 K^2 R^{d-2} \frac{d2^{-d/2 + 1}}{C_2 \Gamma(d/2+1)}\exp(-C_2 R^2 / 2) + C_1 \frac{d 2^{-d/2+1}}{C_2\Gamma(d/2 + 1)} R^{d} \exp(-C_2 R^2 / 2)\right) \int_{t_0}^T \frac{1}{h^2(t)} \diff t \\
& = \cO\left(\frac{1}{t_0(T - t_0)} K^2 R^d \frac{2^{-2/d+2}d}{\Gamma(d/2+1)} \exp\left(-C_2 R^2 / 2\right)\right).
\end{align*}
where the last inequality follows from $\xb = A\zb$ and applying Lemma \ref{lemma:gaussian_tail}, since $p_z(\zb) \leq (2\pi)^{-d/2} C_1 \exp(-C_2 \norm{\zb}_2^2 / 2)$ for $\norm{\zb}_2 > B$.

\noindent $\star$ {\bf Bounding term $(C)$}.
For any $\epsilon > 0$, denote $\bar{\sbb}_{V, \btheta}$ as the constructed network approximator to the score function in Theorem \ref{thm:score_approximation}. Then we have
\begin{align*}
(C) \leq \underbrace{\hat{\cL}(\bar{\sbb}_{V, \btheta}) - (1+a)\cL^{\rm trunc}(\bar{\sbb}_{V, \btheta})}_{(C_1)} + (1+a) \underbrace{\cL^{\rm trunc}(\bar{\sbb}_{V, \btheta})}_{(C_2)},
\end{align*}
where $(C_1)$ is the statistical error and $(C_2)$ is the approximation error.

As data distribution $P_{\rm data}$ has sub-Gaussian tail, $\hat{\cL}(\bar{\sbb}_{V, \btheta}) = \hat{\cL}^{\rm trunc}(\bar{\sbb}_{V, \btheta})$ holds with high probability. In fact, Lemma \ref{lemma:gaussian_tail} yields
\begin{align*}
\PP_{\rm data}\left(\norm{\xb}_2 > R \right) \leq C_1 \frac{d2^{-d/2+1}}{C_2 \Gamma(d/2+1)} R^{d-2} 
 \exp(-C_2 R^2 / 2).
\end{align*}
Applying union bound leads to
\begin{align*}
\PP_{\rm data}\left(\norm{\xb_i}_2 \leq R ~\text{for~all}~i = 1, \dots, n \right) \geq 1 - n C_1 \frac{d2^{-d/2+1}}{C_2 \Gamma(d/2+1)} R^{d-2} 
 \exp(-C_2 R^2 / 2).
\end{align*}
Therefore, $(C_1)$ is equal to
\begin{align*}
(C_1) = \hat{\cL}^{\rm trunc}(\bar{\sbb}_{V, \btheta}) - (1+a)\cL^{\rm trunc}(\bar{\sbb}_{V, \btheta})
\end{align*}
with high probability. Since $\bar{\sbb}_{V, \btheta}$ is a fixed function, Lemma \ref{lemma:concentration_bernstein} implies
\begin{align*}
\hat{\cL}^{\rm trunc}(\bar{\sbb}_{V, \btheta}) - (1+a)\cL^{\rm trunc}(\bar{\sbb}_{V, \btheta}) = \cO \left(\frac{(1+6/a)(K^2 + R^2)}{n t_0(T-t_0)} \log \frac{1}{\delta}\right).
\end{align*}
with probability $1 - \delta$. For $(C_2)$, we have
\begin{align*}
\cL^{\rm trunc}(\bar{\sbb}_{V, \btheta}) & \leq \cL(\bar{\sbb}_{V, \btheta}) \\
& = \frac{1}{T-t_0} \int_{t_0}^T \norm{\bar{\sbb}_{V, \btheta}(\cdot, t) - \nabla \log p_t(\cdot)}_{L^2(P_t)}^2 \diff t \\
& \quad + \underbrace{\cL(\bar{\sbb}_{V, \btheta}) - \frac{1}{T-t_0} \int_{t_0}^T \norm{\bar{\sbb}_{V, \btheta}(\cdot, t) - \nabla \log p_t(\cdot)}_{L^2(P_t)}^2 \diff t}_{(\cE)}.
\end{align*}
Recall that the two terms in $(\cE)$ are equivalent score matching objective functions. Their difference is an absolute constant, denoted as $(\cE) = E$. By Theorem \ref{thm:score_approximation}, we have
\begin{align*}
(C_2) = \cO\left(\frac{d}{t_0(T-t_0)} \epsilon^2\right) + E.
\end{align*}

\noindent $\bullet$ {\bf Putting $(A), (B), (C)$ together}. We first take $R = \cO\left(\sqrt{d\log d + \log K + \log \frac{n}{\delta}}\right)$ such that $\eta \leq \frac{1}{n t_0(T - t_0)}$, $(B) \leq \frac{1}{n t_0(T - t_0)}$ and $\PP_{\rm data}\left(\norm{\xb_i}_2 \leq R ~\text{for~all}~i = 1, \dots, n \right) \geq 1 - \delta$. Next, we set $\iota = \frac{2}{nt_0(T-t_0)}$, which gives rise to
\begin{align*}
(A) = \cO\left(\frac{(1+3/a) \left((1+\beta)^2d^2\log \frac{d}{t_0 \epsilon} + \log \frac{n}{\delta}\right)}{n t_0 (T-t_0)} \log \frac{\cN\left(\frac{1}{n (K + R)t_0\log (T/t_0)}, \cS_{\rm NN}, \norm{\cdot}_2 \right)}{\delta} + \frac{1}{n}\right)
\end{align*}
with probability $1 - \delta$. For term $(C)$, we have
\begin{align*}
(C) = \cO\left(\frac{(1+6/a) \left((1+\beta)^2d^2\log \frac{d}{t_0 \epsilon} + \log \frac{n}{\delta}\right)}{nt_0(T-t_0)} \log \frac{1}{\delta} + \frac{1}{n} + \frac{d}{t_0(T-t_0)} \epsilon^2\right) + (1+a) E
\end{align*}
with probability $1 - 2\delta$. Summing up error terms $(A), (B)$ and $(C)$, we derive
\begin{align*}
\cL(\hat{\sbb}_{V, \btheta}) & \leq (A) + (B) + (1+a) \cdot (C) \\
& = \cO\left(\frac{(1+6/a) \left((1+\beta)^2d^2\log \frac{d}{t_0 \epsilon} + \log \frac{n}{\delta}\right)}{n t_0(T - t_0)}\log \frac{\cN\left(\frac{1}{n(K + R)t_0\log (T/t_0)}, \cS_{\rm NN}, \norm{\cdot}_2 \right)}{\delta} + \frac{1}{n} + \frac{d}{t_0(T-t_0)} \epsilon^2 \right) \\
& \quad + (1+a)^2 E
\end{align*}
with probability $1 - 3\delta$. Using the relation $\frac{1}{T-t_0} \int_{t_0}^T \norm{\bar{\sbb}_{V, \btheta}(\cdot, t) - \nabla \log p_t(\cdot)}_{L^2(P_t)}^2 \diff t = \cL(\bar{\sbb}_{V, \btheta}) - E$, with probability $1 - 3\delta$, we can bound
\begin{align*}
& \quad \frac{1}{T-t_0} \int_{t_0}^T \norm{\bar{\sbb}_{V, \btheta}(\cdot, t) - \nabla \log p_t(\cdot)}_{L^2(P_t)}^2 \diff t \\
& = \cO\left(\frac{(1+6/a) \left((1+\beta)^2d^2\log \frac{d}{t_0 \epsilon} + \log \frac{n}{\delta}\right)}{n t_0(T - t_0)}\log \frac{\cN\left(\frac{1}{n(K + R)t_0\log (T/t_0)}, \cS_{\rm NN}, \norm{\cdot}_2 \right)}{\delta} + \frac{1}{n} + \frac{d}{t_0(T-t_0)} \epsilon^2 \right) \\
& \quad + (2a+a^2) E.
\end{align*}
Setting $a = \epsilon^2$ leads to
\begin{align}\label{eq:bound_with_covering}
& \quad \frac{1}{T-t_0} \int_{t_0}^T \norm{\bar{\sbb}_{V, \btheta}(\cdot, t) - \nabla \log p_t(\cdot)}_{L^2(P_t)}^2 \diff t \nonumber \\
& = \cO\left(\frac{\left((1+\beta)^2d^2\log \frac{d}{t_0 \epsilon} + \log \frac{n}{\delta}\right)}{\epsilon^2 n t_0(T - t_0)}\log \frac{\cN\left(\frac{1}{n(K + R)t_0\log (T/t_0)}, \cS_{\rm NN}, \norm{\cdot}_2 \right)}{\delta} + \frac{1}{n} + \frac{d}{t_0(T-t_0)} \epsilon^2 \right)
\end{align}
with probability $1 - 3\delta$.

\noindent $\star$ {\bf Covering number of $\cS_{\rm NN}$}. The only remaining task is to find the covering number of $\cS_{\rm NN}$. $\cS_{\rm NN}$ consists of two components: 1) matrix $V$ with orthonormal columns; 2) network function $\fb_{\btheta}$. Suppose we have $V_1, V_2$ and $\btheta_1, \btheta_2$ such that $\norm{V_1 - V_2}_{\rm F} \leq \delta_1$ and $\sup_{\norm{\xb}_2 \leq 3R+ \sqrt{D\log D}, t \in [t_0, T]} \norm{\fb_{\btheta_1}(\xb, t) - \fb_{\btheta_2}(\xb, t)}_2 \leq \delta_2$. Then we evaluate
\begin{align*}
& \quad \sup_{\norm{\xb}_2 \leq 3R + \sqrt{D\log D}, t \in [t_0, T]} \norm{\sbb_{V_1, \btheta_1}(\xb, t) - \sbb_{V_2, \btheta_2}(\xb, t)}_2 \\
& = \frac{1}{h(t)} \sup_{\norm{\xb}_2 \leq 3R + \sqrt{D\log D}, t \in [t_0, T]} \norm{V_1 \fb_{\btheta_1}(V_1^\top \xb, t) - V_2 \fb_{\btheta_2}(V_2^\top \xb, t)}_2 \\
& = \frac{1}{h(t)} \sup_{\norm{\xb}_2 \leq 3R + \sqrt{D\log D}, t \in [t_0, T]} \Big[\norm{V_1 \fb_{\btheta_1}(V_1^\top \xb, t) - V_1 \fb_{\btheta_1}(V_2^\top \xb, t)}_2 + \norm{V_1 \fb_{\btheta_1}(V_2^\top \xb, t) - V_1 \fb_{\btheta_2}(V_2^\top \xb, t)}_2 \\
& \hspace{2in} + \norm{V_1 \fb_{\btheta_2}(V_2^\top \xb, t) - V_2 \fb_{\btheta_2}(V_2^\top \xb, t)}_2 \Big] \\
& \leq \frac{1}{h(t)} \left(\gamma \delta_1 
\sqrt{d}(3R + \sqrt{D\log D}) + \delta_2 + \delta_1 K\right),
\end{align*}
where we recall $\gamma$ upper bounds the Lipschitz constant of $\fb_{\btheta_1}$. For set $\{V \in \RR^{D \times d}: \norm{V}_{\rm 2} \leq 1\}$, its $\delta_1$-covering number is $\left(1 + 2 \frac{\sqrt{d}}{\delta_1}\right)^{Dd}$ \citep[Lemma 8]{chen2019generalization}. For the $\delta_2$-covering number of $\fb_{\btheta}$, we follow the upper bound in \citet[Lemma 5.3]{chen2022nonparametric}:
\begin{align*}
\left(\frac{2L^2 M (3R + \sqrt{D \log D})) \kappa^L M^{L+1}}{\delta_2}\right)^{J}.
\end{align*}
To this end, with $R = \cO\left(\sqrt{d\log d + \log K + \log \frac{n}{\delta}}\right)$, we compute the log covering number of $\cS_{\rm NN}$ as
\begin{align*}
\log \cN(\iota, \cS_{\rm NN}, \norm{\cdot}_2) & = \cO\Bigg( 2Dd \cdot \log \left(1+\frac{6 K \gamma\sqrt{d}(3R+\sqrt{D\log D})}{t_0 \iota}\right) \\
& \quad + J \log \frac{6L^2M (3R + \sqrt{D \log D})) \kappa^L M^{L+1}}{t_0\iota} \Bigg) \\
& = \cO\left(\left((1+\beta)^d T\tau d^{d/2} \epsilon^{-(d+1)} \log^{d/2} \frac{d}{t_0 \epsilon} + D d\right) \left(d\log \frac{1}{\epsilon} + d^2\right) \log \frac{T\tau D d \log D}{t_0\iota \epsilon} \right).
\end{align*}
Substituting the log covering number into \eqref{eq:bound_with_covering}, we have
\begin{align*}
& \quad \frac{1}{T-t_0} \int_{t_0}^T \norm{\bar{\sbb}_{V, \btheta}(\cdot, t) - \nabla \log p_t(\cdot)}_{L^2(P_t)}^2 \diff t \\
& = \cO\Bigg(\frac{\left((1+\beta)^2d^2\log \frac{d}{t_0 \epsilon} + \log \frac{n}{\delta}\right)}{\epsilon^2 n t_0(T - t_0)} \left((1+\beta)^d T\tau d^{d/2} \epsilon^{-(d+1)} \log^{d/2} \frac{d}{t_0 \epsilon} + D d\right) \left(d\log \frac{1}{\epsilon} + d^2\right) \log \frac{nT\tau D d \log D}{(T-t_0)\epsilon} \\
& \quad + \frac{1}{n} + \frac{d}{t_0(T-t_0)} \epsilon^2 \Bigg).
\end{align*}
\noindent $\bullet$ {\bf Balancing error terms}. Note that $\log^{d/2} \frac{1}{\epsilon} \leq \left(\frac{1}{\epsilon}\right)^{\frac{d \log \log (1/\epsilon)}{2\log (1/\epsilon)}}$. We set $\epsilon = n^{-\frac{1 -\delta(n)}{d+5}}$, which implies $\frac{1}{n} \epsilon^{-d-3} \log^{d/2} \frac{1}{\epsilon} \leq n^{-\frac{2-2\delta(n)}{d+5}}$. Then with probability $1 - 3\delta$, it holds
\begin{align*}
& \frac{1}{T-t_0} \int_{t_0}^T \norm{\bar{\sbb}_{V, \btheta}(\cdot, t) - \nabla \log p_t(\cdot)}_{L^2(P_t)}^2 \diff t \\
& \hspace{0.7in} = \cO\left(\frac{\tau (1+\beta)^{d+2} 
 d^{d/2+4}}{t_0} \left(n^{-\frac{2-2\delta(n)}{d+5}} + D d n^{-\frac{d+3+2\delta(n)}{d+5}}\right)\log^{d/2+3} \left(\frac{d}{\delta t_0}\right) \log D \log^3 n\right).
\end{align*}
Setting $\delta = \frac{1}{3n}$ gives rise to
\begin{align*}
& \frac{1}{T-t_0} \int_{t_0}^T \norm{\bar{\sbb}_{V, \btheta}(\cdot, t) - \nabla \log p_t(\cdot)}_{L^2(P_t)}^2 \diff t \\
& \hspace{0.7in} = \cO\left(\frac{\tau (1+\beta)^{d+2} 
 d^{d/2+4} }{t_0} \left(n^{-\frac{2-2\delta(n)}{d+5}} + D d n^{-\frac{d+3+2\delta(n)}{d+5}}\right)\log^{d/2+3} \left(\frac{d}{t_0}\right) \log D \log^3 n\right)
\end{align*}
with probability $1 - \frac{1}{n}$. Omitting factors in $d, \beta, \tau, \log D, \log t_0$ yields the bound in Theorem \ref{thm:score_estimation}.
\end{proof}

\subsection{Conditional covariance bound}\label{pf:conditional_cov}
We repeat the on-support score expression for reference:
\begin{align}\label{eq:score1}
\sbb_{\parallel}(\zb', t) = \frac{\alpha(t)}{h(t)} A \int \frac{\zb \cdot \phi_t(\zb' | \zb) p_z(\zb)}{\int \phi_t(\zb' | \zb) p_z(\zb) \diff \zb} \diff \zb - \frac{1}{h(t)} A\zb'.
\end{align}

Using \eqref{eq:score1} and taking derivative with respect to $\zb'$, we have
\begin{align*}
\frac{\partial}{\partial \zb'} \sbb_{\parallel}(\zb', t) & = \left(\frac{\alpha(t)}{h(t)}\right)^2 A \left[\int  \frac{\zb\zb^\top \phi_t(\zb' | \zb) p_z(\zb)}{\int \phi_t(\zb' | \zb) p_z(\zb) \diff \zb} \diff \zb - \int \frac{\zb \phi_t(\zb' | \zb) p_z(\zb)}{\int \phi_t(\zb' | \zb) p_z(\zb) \diff \zb} \diff \zb \int \frac{\zb^\top \phi_t(\zb' | \zb) p_z(\zb)}{\int \phi_t(\zb' | \zb) p_z(\zb) \diff \zb} \diff \zb \right] - \frac{1}{h(t)} A \\
& = \left(\frac{\alpha(t)}{h(t)}\right)^2 A \left[\Cov(\zb | \zb') - \frac{1}{h(t)}I_d\right],
\end{align*}
which implies
\begin{align*}
\norm{\Cov(\zb | \zb')}_{\rm op} \leq \frac{h^2(t)}{\alpha^2(t)} \left(\beta + \frac{1}{h(t)}\right).
\end{align*}

\subsection{Truncation error}
\begin{lemma}\label{lemma:truncation_error}
Suppose Assumption \ref{assumption:pz} holds. Let $\gb$ be defined in \eqref{eq:score_sbb}. Given $\epsilon > 0$, with $R = c \left(\sqrt{d\log \frac{d}{t_0} + \log \frac{1}{\epsilon}} \right)$ for an absolute constant $c$, it holds
\begin{align*}
\norm{\gb(A^\top \xb, t)\mathds{1}\{\norm{A^\top \xb}_2 \geq R\}}_{L^2(P_t)} \leq \epsilon \quad \text{for}\quad t \in [t_0, T].
\end{align*}
\end{lemma}
\begin{proof}
Let $\eta \in (0, 1/2)$ to be chosen later. Plugging in the expression of $\gb$, we have
\begin{align*}
& \quad \int \norm{\int \frac{\zb \phi_t(A^\top \xb | \zb) p_z(\zb) \diff \zb}{\int \phi_t(A^\top \xb | \zb) p_z(\zb) \diff \zb}}_2^2 \mathds{1}\{\norm{A^\top \xb}_2 > R\} p_t(\xb) \diff \xb \\
& \leq \int_{\norm{A^\top \xb}_2 > R} \int_{\norm{\zb}_2 \leq \eta \norm{A^\top\xb}_2} \norm{\zb}_2^2 \frac{\phi_t(A^\top\xb | \zb) p_z(\zb)}{\int \phi_t(A^\top \xb | \zb) p_z(\zb) \diff \zb} p_t(\xb) \diff \xb \\
& \quad + \int_{\norm{A^\top \xb}_2 > R} \int_{\norm{\zb}_2 
> \eta \norm{A^\top \xb}_2} \norm{\zb}_2^2 \frac{\phi_t(A^\top \xb | \zb) p_z(\zb)}{\int \phi_t(A^\top \xb | \zb) p_z(\zb) \diff \zb} p_t(\xb) \diff \xb \\
& \leq \int_{\norm{A^\top \xb}_2 > R} \int_{\norm{\zb}_2 \leq \eta \norm{A^\top \xb}_2} \norm{\zb}_2^2 \phi_t(A^\top \xb | \zb) \phi_t((I_D - AA^\top) \xb)p_z(\zb) \diff \zb \diff \xb \\
& \quad + \int_{\norm{A^\top \xb}_2 > R} \int_{\norm{\zb}_2 > \eta \norm{A^\top \xb}_2} \norm{\zb}_2^2 \phi_t(A^\top\xb | \zb) \phi_t((I_D - AA^\top) \xb)p_z(\zb) \diff \zb \diff \xb \\
& \overset{(i)}{=} \underbrace{\int_{\norm{\zb'}_2 > R} \int_{\norm{\zb}_2 \leq \eta \norm{\zb'}_2} \norm{\zb}_2^2 \phi_t(\zb' | \zb) p_z(\zb) \diff \zb \diff \zb'}_{(A)} \\
& \quad + \underbrace{\int_{\norm{\zb'}_2 > R} \int_{\norm{\zb}_2 > \eta \norm{\zb'}_2} \norm{\zb}_2^2 \phi_t(\zb' | \zb) p_z(\zb) \diff \zb \diff \zb'}_{(B)},
\end{align*}
where we recall Gaussian density $\phi_t((I_D - AA^\top)\xb) = (2\pi)^{-(D-d)/2} h^{-(D-d)/2}(t) \exp\left(-\frac{1}{2h(t)}\norm{\left(I_D - AA^\top \right)\xb}_2^2\right)$, and in $(i)$, we observe $\phi_t(A^\top \xb | \zb)$ and $\phi_t((I_D - AA^\top)\xb)$ are independent Gaussians for any fixed $\zb$. 

In term $(A)$, when $\norm{\zb}_2 \leq \eta \norm{\zb'}_2$, we have $\norm{\zb' - \alpha(t) \zb}_2^2 \geq \frac{1}{2} \norm{\zb'}_2^2 - \alpha^2(t) \norm{\zb}_2^2 \geq \left(\frac{1}{2} - \eta\right)\norm{\zb'}_2^2$. As a result, we have
\begin{align*}
(A) & \leq \int_{\norm{\zb'}_2 > R} \int_{\norm{\zb}_2 \leq \eta \norm{\zb'}_2} \norm{\zb}_2^2 (2\pi h(t))^{-d/2} \exp\left(-\frac{\frac{1}{2} - \eta}{2h(t)} \norm{\zb'}_2^2\right) p_z(\zb) \diff \zb \diff \zb' \\
& \leq \EE[\norm{\zb}_2^2] \int_{\norm{\zb'}_2 > R} (2\pi h(t))^{-d/2} \exp\left(-\frac{\frac{1}{2} - \eta}{2h(t)} \norm{\zb'}_2^2\right) \diff \zb' \\
& \leq \EE[\norm{\zb}_2^2] \frac{2^{-d/2 + 2} d h^{-d/2+1}(t)}{(1/2 - \eta) \Gamma(d/2 + 1)} R^{d-2} \exp\left(-\frac{\frac{1}{2} - \eta}{2h(t)} R^2\right).
\end{align*}
For term $(B)$, under the condition $R > \eta^{-1}B \vee 1$, we have
\begin{align*}
(B) & = \int_{\norm{\zb'}_2 > R} \int_{\norm{\zb}_2 > \eta \norm{\zb'}_2} \norm{\zb}_2^2 \phi_t(\zb' | \zb) (2\pi)^{-d/2} C_1 \exp(-C_2 \norm{\zb}_2^2 / 2) \diff \zb \diff \zb' \\
& \leq C_1 \int_{\norm{\zb'}_2 > R} \int_{\norm{\zb}_2 > \eta \norm{\zb'}_2} \norm{\zb}_2^2 (2\pi h(t))^{-d} \exp\left(-\frac{C_2}{2(\alpha^2(t) + C_2h(t))} \norm{\zb'}_2^2 \right) \\
& \quad \cdot \exp\left(-\frac{\alpha^2(t) + C_2 h(t)}{2h(t)} \norm{\zb - \frac{\alpha(t)}{\alpha^2(t) + C_2 h(t)}\zb'}_2^2\right) \diff \zb \diff \zb' \\
& \leq C_1 (\alpha^2(t) + C_2h(t))^{-d/2} (2\pi h(t))^{-d/2} \\
& \quad \cdot \int_{\norm{\zb'}_2 > R} \left[\frac{\alpha^2(t)}{(\alpha^2(t) + C_2 h(t))^2} \norm{\zb'}_2^2 + \frac{h(t) d}{\alpha^2(t) + C_2 h(t)} \right] \exp\left(-\frac{C_2}{2(\alpha^2(t) + C_2h(t))} \norm{\zb'}_2^2 \right) \diff \zb' \\
& \leq C_1 (\alpha^2(t) + C_2h(t))^{-d/2} \frac{2^{-d/2 + 2}d h^{-d/2}(t)}{C_2 \Gamma(d/2 + 1)} R^d \exp\left(-\frac{C_2}{2(\alpha^2(t) + C_2h(t))} R^2\right).
\end{align*}
It suffices to choose $\eta = \frac{1}{4}$. Combining $(A)$ and $(B)$, we conclude
\begin{align*}
\norm{\gb(A^\top \xb, t)\mathds{1}\{\norm{A^\top \xb}_2 \geq R\}}_{L^2(P_t)}^2 \leq c' \frac{2^{-d/2+3} d h^{-d/2}(t)}{\Gamma(d/2+1)} R^d \exp\left(-\frac{C_2}{8(\alpha^2(t) + C_2h(t))} R^2\right)
\end{align*}
for an absolute constant $c'$. In order for $\norm{\gb(A^\top \xb, t)\mathds{1}\{\norm{A^\top \xb}_2 \geq R\}}_{L^2(P_t)}^2 \leq \epsilon$, we deduce
\begin{align*}
R = c \left(\sqrt{d\log \frac{d}{t_0} + \log \frac{1}{\epsilon}} \right),
\end{align*}
where $c$ is an absolute constant.
\end{proof}

\section{Omitted proofs in Section~\ref{sec:distro_guarantee}} \label{append:proof:de}
    
\subsection{Subspace Error and Latent Score Matching Error}
For simplicity, we define the (unnormalized) expectation $\myEE$ as
\[
    \myEE [\phi(\xb, t)] = \int_{t_0}^T \frac{1}{h^2(t)} \EE_{\xb \sim P_t} [\phi(\xb,t)] \mathrm{d} t.
\]
During the analysis, we also denote $\zb = A^\top \xb$ and 
\[
    \myEE [\phi(\zb, t)] = \int_{t_0}^T \frac{1}{h^2(t)} \EE_{\xb \sim P_t} [\phi(A^\top \xb,t)] \mathrm{d} t.
\]

Define 
\[
    \gb(\zb, t) = h(t) \nabla  \log \ptz(\zb) + \zb,
\]
Then the objective of diffusion models is 
\[
    \int_{t_0}^T  \EE_{\Xb_t \sim P_t} \| \sbb_{V, \theta}(\Xb_t, t) -     \nabla \log p_t(\Xb_t) \|^2_2 \mathrm{d} t = \myEE \| V \fb_{\btheta}(V^\top \xb, t) - A \gb (A^\top \xb, t) \|^2_2.  
\]

\begin{lemma}\label{lem:main}
Assume that the following holds 
    \begin{align*}
        \EE_{\zb \sim P_z} \| \nabla  \log p_z(\zb)\|^2_2 &\leq \CE,\\
       \lambda_{\min} \EE_{\zb \sim P_z}   [\zb \zb^\top]  &\geq \Czero,\\
        \EE_{\zb \sim P_z} \| \zb \|^2_2 &\leq \Ctwo . 
    \end{align*}
We set  $t_0  \leq \min\Big\{ \log(d/\CE   + 1 ) , 1, \log(1+\Czero), \frac{\Czero}{4e\log(4e)}\Big\}$ and $T \geq \max\{\log( \Ctwo / d+1), 1\}$.
Suppose we have 
\[ 
    \myEE \| V \fb_{\btheta}(V^\top \xb, t) - A \gb (A^\top \xb, t) \|^2_2 \leq \epsilon.
\]
Then we have
\[
    \|VV^\top - AA^\top\|_{\rm F}^2 =  \mathcal{O}\big(\frac{t_0}{c_0} \epsilon \big), 
\]
and there exists an orthonormal matrix $U \in \RR^{d\times d}$, such that:
    \begin{align*} 
        &\quad \myEE \| U ^\top \fb_{\btheta}(U \zb, t) -  \gb(\zb, t) \|_2^2 \\ 
        & \lesssim \epsilon \cdot \Big[ 1+ \frac{t_0}{\Czero}\Big((T-\log t_0) d\cdot \max_t \|\fb_{\btheta}(\cdot,t)\|^2_{Lip} + \CE \Big) + \frac{\max_t \|\fb_{\btheta}(\cdot,t)\|^2_{Lip}\cdot\Ctwo}{\Czero} \Big].
    \end{align*} 
\end{lemma}

\subsection{Backward processes}\label{appen:backward} 
In this section, we provide the distribution estimation error of the learned backward SDEs. The objects of our arguments are all in the latent space. Specifically, we consider the following decomposition of the ground-truth backward process: $\bXb_t = A \bZb_t + \bXbo$, where
\begin{align*}
\bZb_t = A^\top \bXb_t \quad \text{and} \quad \bXbo = (I - AA^\top) \bXb_t.
\end{align*}
We know that the forward SDE for $(\Zb_t)_{t \geq 0}$ is 
\[
    \diff \Zb_t = -\frac{1}{2} \Zb_t \diff t + \diff (A^\top \Wb_t),
\]
where $\Zb_0 \sim P_z$. Denote $P_t^{\sf LD}$ as the marginal distribution of $\Zb_t$  .  The backward SDE for $\bZb_t$ is
\begin{align*}
\diff  \bZb_t = \left[\frac{1}{2} \bZb_t  +  \nabla  \log \pld _{T-t} (\bZb_t )  \right] \diff t +   \diff (A^\top\overline{\Wb}_t).
\end{align*}

For the learned process ${\tildebXb}_t$, we consider a similar decomposition ${\tildebXb}_t = V {\tildebZb}_{t } + {\tildebXb}_{t, \perp}$, where
\begin{align*}
{\tildebZb}_{t } = V^\top {\tildebXb}_t  \quad \text{and} \quad      {\tildebXb}_{t, \perp} = (I-VV^\top) {\tildebXb}_t.
\end{align*}

For any orthogonal matrix $U \in \RR^{d \times d}$, define the $U$ transformed version of ${\tildebZb}_{t }$ as ${\tildebZbRot}_{t } = U^\top {\tildebZb}_{t }$. The backward SDEs for   ${\tildebZbRot}_{t }$ is 
\begin{align}
\diff {\tildebZbRot}_{t }  = \left[\frac{1}{2} {\tildebZbRot}_{t }  + \sld_{U, \btheta}({\tildebZbRot}_{t },T-t)  \right] \diff t +   \diff  (U^\top V^\top\overline{\Wb}_t), \label{eq:latentLearned}
\end{align}
where 
\[
    \sld_{U, \btheta}(\zb,t) =  \frac{1}{h(t)}\Big[ -\zb  + U^\top \fb_{\btheta} (U \zb  , t) \Big] .    
\]
When ${\tildebXb}_0 \sim {\sf N}(0,I)$, we have ${\tildebZbRot}_0 \sim {\sf N}(0,I_d)$. We define $\hat{P}^{\sf LD}_{t_0}$ to be the marginal distribution of ${\tildebZbRot}_{T-t_0}$.

The discretized backward SDE is 
\begin{align*}
    \diff {\tildebZbRotDis}_{t }  = \left[\frac{1}{2} {\tildebZbRotDis}_{k \eta }  + \sld_{U, \btheta}({\tildebZbRotDis}_{k \eta},T-k \eta)  \right] \diff t +   \diff  (U^\top V^\top\overline{\Wb}_t) \text{ for } t \in [ k\eta,  (k+1)\eta).
\end{align*}
We define $\hat{P}^{\sf LD, dis}_{t_0}$ to be the marginal distribution of ${\tildebZbRotDis}_{T-t_0}$.

\begin{lemma}\label{lem:de}
Assume that $P_z$ is subGaussian. $\fb_{\btheta}(\zb, t)$ and $\nabla \log \ptz(\zb)$ is Lipschitz in both $\zb$ and $t$. Assume we have the latent score matching error bound 
\[
    \int_{t_0}^T  \EE_{\Zb_t \sim P_t^{\sf LD}} \|\sld_{U, {\btheta}} (\Zb_t,t) - \nabla \log \ptz (\Zb_t) \|_2^2 \diff t \leq \epsilon_{latent}(T-t_0).
\]
Then we have the following latent distribution estimation error for the undiscretized backward SDE 
\[
    \TV({P}^{\sf LD}_{t_0}, \hat{P}^{\sf LD}_{t_0}) \lesssim \sqrt{\epsilon_{latent}(T-t_0)}  + \sqrt{\KL(P_z||N(0,I_d))}\exp(-T) . 
\] 
Furthermore, we have the following latent distribution estimation error for the discretized backward SDE 
\[
    \TV({P}^{\sf LD}_{t_0}, \hat{P}^{\sf LD, dis}_{t_0}) \lesssim \sqrt{\epsilon_{latent}(T-t_0)}  + \sqrt{\KL(P_z||N(0,I_d))}\exp(-T) + \sqrt{\epsilon_{dis}(T-t_0)}, 
\]
where 
    \begin{align*}
        \epsilon_{dis} &= \Big(\frac{\max_{\zb} \| \fb_{\btheta} (\zb , \cdot) \|_{Lip} }{h(t_0)} +\frac{\max_{\zb,t} \| \fb_{\btheta} (\zb,t) \|_2}{t_0^2} \Big)^2 \eta^2  + \Big( \frac{\max_t \| \fb_{\btheta} (\cdot , t) \|_{Lip}}{h(t_0)}   \Big)^2  \eta^2 \max\{ \EE \|\Zb_0 \|^2 , d\}   + \eta  d .
\end{align*} 
\end{lemma}

\subsection{Orthogonal process}
\begin{lemma}\label{lem:ortho:process}
Consider the following SDE 
\[
    \diff \Yb_t = \Big[\frac{1}{2} - \frac{1}{h(T-t)} \Big] \Yb_t \diff t + \diff \Bb_t,
\]
where $\Yb_0 \sim {\sf N}(0, I)$.
Then when $T>1$ and $t_0 \leq 1$, we have $\Yb_{T-t_0} \sim {\sf N}(0, \sigma^2 I)$ with $\sigma^2 \leq e t_0$. 
\end{lemma}

\begin{lemma}[Discretized version]\label{lem:ortho:process:dis}
    Consider the following discretized SDE with step size $\eta$ satisfying $T-t_0 = K_T \eta$.  
    \[
        \diff \Yb_t = \Big[\frac{1}{2} - \frac{1}{h(T-k\eta)} \Big] \Yb_{k\eta} \diff t + \diff \Bb_t, ~\text{for } t \in [k\eta, (k+1)\eta),
    \]
    where $\Yb_0 \sim {\sf N}(0, I)$.
    
    Then when $T>1$ and $t_0 +\eta \leq 1$, we have $\Yb_{T-t_0} \sim {\sf N}(0, \sigma^2 I)$ with $\sigma^2 \leq e (t_0+\eta)$. 
\end{lemma}

\subsection{Proof of Theorem~\ref{thm:distro_estimation}}
\begin{proof}
In Lemma~\ref{lem:main}, we replace $\epsilon$ to be $\epsilon(T-t_0)$ and we set $\CE = \beta d$ by Lemma~\ref{lem:score2bound}, we have 
    \[
        \|V V^\top-  AA^\top\|_{\rm F}^2 = \epsilon \cdot \mathcal{O}\Big( \frac{t_0 T }{\Czero}\Big).  
    \]
Substituting the score estimation error in Theorem \ref{thm:score_estimation} and $T = \cO(\log n)$ into the bound above, we deduce
\begin{align*}
\|V V^\top-  AA^\top\|_{\rm F}^2 = \tilde{\cO}\left(\frac{1}{c_0} n^{-\frac{2-2\delta(n)}{d+5}} \log^{7/2} n\right).
\end{align*}
The first item in Theorem \ref{thm:distro_estimation} is proved.

Lemma~\ref{lem:score2bound} also implies
    \[
        \myEE \| U ^\top \fb_{\btheta}(U \zb, t) -  \gb(\zb, t) \|_2^2 \lesssim \epsilon_{latent}(T-t_0),
    \]
    where 
    \[
        \epsilon_{latent}  = \epsilon  \cdot \mathcal{O}\Big( \Big[\frac{t_0}{\Czero}\Big((T-\log t_0) d\cdot \gamma^2 + d \beta \Big)  + \frac{\gamma^2\cdot\Ctwo}{\Czero} \Big] \Big)  .
    \]
Some algebra yields
    \[
        \myEE \| U ^\top \fb_{\btheta}(U \zb, t) -  \gb(\zb, t) \|_2^2      = \int_{t_0}^T \EE_{\zb \sim P_t^{\sf LD}} \Big \| \frac{U ^\top \fb_{\btheta}(U \zb, t) -\zb }{h(t)} - \nabla \log \ptz (\zb) \Big \|_2^2  \diff t \leq \epsilon_{latent} (T - t_0).
    \]
    Therefore, by Lemma~\ref{lem:de}, we obtain
    \begin{align*}
    \TV({P}^{\sf LD}_{t_0}, \hat{P}^{\sf LD, dis}_{t_0}) & \lesssim \sqrt{\epsilon_{latent}(T-t_0)}  + \sqrt{\KL(P_z||{\sf N}(\boldsymbol{0},I_d))}\exp(-T) + \sqrt{\epsilon_{dis}(T-t_0)} \\
    & = \tilde{\cO}\left(\frac{1}{\sqrt{t_0 c_0}} n^{-\frac{1-\delta(n)}{d+5}} \log^2 n + \frac{1}{n} + \eta \frac{\sqrt{d\log d}}{t_0^2} + \sqrt{\eta} \sqrt{d} \right).
    \end{align*}
    With $\eta \lesssim \frac{t_0^2}{\sqrt{d\log d}} n^{-\frac{2-2\delta(n)}{d+5}}$, we deduce
    \begin{align*}
    \TV({P}^{\sf LD}_{t_0}, \hat{P}^{\sf LD, dis}_{t_0}) = \tilde{\cO}\left(\frac{1}{\sqrt{c_0t_0}} n^{-\frac{1-\delta(n)}{d+5}} \log^2 n \right).
    \end{align*}
    By definition, $\hat{P}^{\sf LD, dis}_{t_0} = (UV)^\top_{\sharp} \hat{P}^{\sf dis}_{t_0}$. The total variation distance bound in item 2 is proved. The Wasserstein-2 distance ${\sf W}_2(P_{t_0}^{\sf LD}, P_z)$ is bounded using the same technique as \citet[Lemma 16]{chen2022sampling}. Although they require bounded support, the proof only relies on finite second moment of $P_z$, which is verified under our Assumption \ref{assumption:pz}. As a result, we have
    \begin{align*}
    {\sf W}_2(P_{t_0}^{\sf LD}, P_z) = \cO\left(\sqrt{d t_0}\right).
    \end{align*}

    Lastly, in item 3, due to our score decomposition, the orthogonal process follows that in Lemma~\ref{lem:ortho:process:dis}. Invoking the marginal distribution at time $T - t_0$ and observing $\eta \ll t_0$, we obtain the desired result.
\end{proof}

\section{Omitted proofs in Section~\ref{append:proof:de}}

\subsection{Proof of Lemma~\ref{lem:main}}
We introduce several lemmas in preparation for the proof of Lemma~\ref{lem:main}.
\begin{lemma}\label{lem:subspace}
    Let $X, Y$ be random variables, $A, V \in \RR^{D \times d}$ have orthonormal columns. Then $ \myEE \norm{V X -  A Y }^2_2 \leq \epsilon$ implies
    \[
        \| (I_D - VV^\top) A \|_{\rm F}^2 \leq \epsilon_V = \frac{1}{\lambda_{\min}} \epsilon, 
    \]
    where $\lambda_{\min}$ is the smallest eigenvalue of $\myEE [YY^\top]$.
\end{lemma}

\begin{proof}[proof of Lemma~\ref{lem:subspace}]
    Notice that the best $L^2$ approximation in the subspace $Im(V)$ to $AY$ is $V^\top AY$, which can be verified through the following calculation:
    \begin{align*}
        \|VX- AY\|^2_2 &= \| VX -VV^\top AY \|^2_2   + \| VV^\top AY - AY \|^2_2 + 2 \langle VX -VV^\top AY,   VV^\top AY - AY\rangle \\
        &=  \| VX -VV^\top AY \|^2_2   + \| VV^\top AY - AY \|^2_2 + 2 \langle X -V^\top AY,   V^\top (VV^\top AY - AY)\rangle  \\
        &=  \| VX -VV^\top AY \|^2_2   + \| VV^\top AY - AY \|^2_2.
    \end{align*}
    Therefore, we have 
    \[
        \|VX- AY\|^2_2 \geq \| VV^\top AY - AY \|^2_2 = \| (I_D - VV^\top) A Y\|^2_2.
    \]
    Then 
    \begin{align*}
        \epsilon &\geq \myEE \|VX- AY\|^2_2 \\
        &\geq \myEE \| (I_D - VV^\top) A Y\|^2_2\\
        &=  \tr  \left[ A^\top (I_D - VV^\top) (I_D - VV^\top) A \cdot  \myEE YY^\top \right]\\
        &\geq \lambda_{\min}  \tr  \left[  A^\top (I_D - VV^\top) (I_D - VV^\top) A \right]\\
        &\geq \lambda_{\min} \|(I_D - VV^\top) A\|_{\rm F}^2.
    \end{align*}
\end{proof}

\begin{lemma}\label{lem:main1}
    Assume that we have 
    \[ 
        \myEE \| V \fb_{\btheta}(V^\top \xb, t) - A \gb (A^\top \xb, t) \|^2_2 \leq \epsilon.
\]
    There exists an orthonormal matrix $U \in \RR^{d\times d}$, such that:
    \begin{align*} 
        &\quad \myEE \| U ^\top \fb_{\btheta}(U \zb, t) -  \gb(\zb, t) \|_2^2 \\
        & \lesssim \epsilon +  \frac{\epsilon}{\lambda_{\min}} \cdot \myEE \|\gb (\zb,t)\|_2^2 + \frac{\epsilon}{\lambda_{\min}} \myEE \|\zb\|_2^2\cdot \max_t \|\fb_{\btheta}(\cdot,t)\|^2_{Lip}.
    \end{align*}
    where  $\lambda_{\min} = \lambda_{\min} (\myEE [\gb(\zb, t)\gb(\zb, t)^\top])$.
\end{lemma}

\begin{proof}[Proof of Lemma~\ref{lem:main1}]
Since 
\[
    \myEE \| V \fb_{\btheta}(V^\top \xb, t) - A \gb(A^\top \xb, t) \|^2_2 \leq \epsilon,
\]
by Lemma~\ref{lem:subspace}, we have 
\[
    \| (I_D - VV^\top) A \|_{\rm F}^2 \leq \epsilon_V \overset{def}{=} \frac{1}{\lambda_{\min}} \epsilon, 
\]
where $\lambda_{\min}$ is the smallest eigenvalue of $\myEE [\gb( \zb, t)\gb(\zb, t)^\top]$.

Then by Lemma~\ref{lem:ortho}, we know that there exists an orthonormal matrix $U \in \RR^{d \times d}$, such that 
\[
    \| U - V^\top A \|_{\rm F}^2 \leq 2 \epsilon_V.
\]

We have the following error decomposition
\begin{align*}
\myEE \| U^\top \fb_{\btheta}(U \zb, t) -  \gb(\zb, t) \|_2^2 &= \myEE \|   \fb_{\btheta}(U \zb, t) - U \gb(\zb, t) \|_2^2  \\
& \lesssim \myEE \| \fb_{\btheta}(U \zb, t) - \fb_{\btheta}(UU^\top V^\top A \zb, t)  \|_2^2 \\
&\quad + \myEE \|  \fb_{\btheta}(UU^\top V^\top A \zb, t) - V^\top A \gb(A^\top \xb , t ) \|_2^2 \\
&\quad + \myEE \|   V^\top A \gb(A^\top \xb , t ) -   U\gb(\zb, t) \|_2^2.
\end{align*}
Next we provide upper bounds on the three terms.
\begin{align*}
    \myEE \| \fb_{\btheta}(U \zb, t) - \fb_{\btheta}(UU^\top V^\top A \zb, t)  \|_2^2 & \leq \myEE  \|\fb_{\btheta}(\cdot, t)\|_{Lip}^2\cdot \|U (I_d - U^\top V^\top A) \zb \|_2^2 \\
    &\leq \max_{t}  \|\fb_{\btheta}(\cdot, t)\|_{Lip}^2 \cdot\myEE \|U (I_d - U^\top V^\top A) \zb \|_2^2  \\
    &\leq \max_{t}  \|\fb_{\btheta}(\cdot, t)\|_{Lip}^2\cdot \|I_d - U^\top V^\top A\|_2^2\cdot \myEE \|   \zb \|_2^2  \\
    &= \max_{t}  \|\fb_{\btheta}(\cdot, t)\|_{Lip}^2 \cdot\|U -  V^\top A\|_2^2 \cdot\myEE \|   \zb \|_2^2  \\
    &\leq 2 \max_{t}  \|\fb_{\btheta}(\cdot, t)\|_{Lip}^2 \cdot\myEE \| \zb \|_2^2 \cdot\epsilon_V.
\end{align*}
\begin{align*}
    \myEE \|  \fb_{\btheta}(UU^\top V^\top A \zb, t) - V^\top A \gb(A^\top \xb , t ) \|_2^2 &= \myEE \|  \fb_{\btheta}( V^\top A \zb, t) - V^\top A \gb(A^\top \xb , t ) \|_2^2 \\
    &\leq  \myEE \| V  \fb_{\btheta}( V^\top A \zb, t) - A \gb(A^\top \xb , t ) \|_2^2 \\
    &\leq \epsilon.
\end{align*}
\begin{align*}
    \myEE \|   V^\top A \gb(A^\top \xb , t ) -   U\gb(\zb, t) \|_2^2 &\leq \|V^\top A - U\|_2^2 \cdot \myEE \| \gb(\zb, t) \|_2^2\\
    &\leq 2 \epsilon_V \cdot \myEE \| \gb(\zb, t) \|_2^2.
\end{align*}
\end{proof}

\begin{proof}[Proof of Lemma~\ref{lem:main}]
The proof is dedicated to compute the problem constants in Lemma~\ref{lem:main1}.

Denote $\EE_t \phi(\xb) = \EE_{\xb\sim P_t} \phi(\xb)$ and $\EE_t \phi(\zb) = \EE_{\xb\sim P_t, \zb=A^\top \xb} \phi( \zb)$. Specifically, $\EE_0 \phi(\zb)  = \EE_{\zb \sim P_z} \phi(\zb)$.

\paragraph{Properties of $h(t)$.}

We set $g(t) = 1$. Then $h(t) = 1 - \exp(-t)$, $h^{-1}(w) = -\log(1-w)$. And we have 
\[
    \int \frac{1-h(t)}{h^2(t)} \mathrm{d} t = \frac{1}{1-\exp(t)} + { \rm Constant}. 
\]
\[
    \int \frac{1}{h(t)} \mathrm{d} t = \log (  \exp(t) -1) +  { \rm Constant}. 
\]
\[
    \int \frac{1}{1-h(t)} \mathrm{d} t =  \exp(t)  +  { \rm Constant}. 
\]
We have the following bounds 
\[
    \int_{t_1}^{t_2} \frac{1-h(t)}{h^2(t)} \mathrm{d} t \leq \frac{1}{t_1}  .
\]
\[
    \int_{t_1}^{t_2} \frac{1}{h(t)} \mathrm{d} t  \leq t_2 -\log t_1 .
\]
\[
    \int_{t_1}^{t_2}  \frac{1}{1-h(t)} \mathrm{d} t   \leq \exp(t_2) - t_1 - 1.
\] 

\paragraph{Upper bounds for $\myEE \|\zb\|^2_2$.}
\begin{align*}
    \myEE \|\zb\|^2_2 &= \int_{t_0}^T \frac{1}{h^2(t)} \EE_t \| \zb\|^2_2 \mathrm{d}t \\
    &=\int_{t_0}^T \frac{1}{h^2(t)}  [(1-h(t))\EE_0\|\zb\|^2_2 + h(t)d ] \mathrm{d}t\\
    &=\int_{t_0}^T \frac{1-h(t)}{h^2(t)}\mathrm{d}t \cdot  \EE_0\|\zb\|^2_2  + \int_{t_0}^T  \frac{1 }{h(t)}\mathrm{d}t \cdot    d \\
    &\leq \frac{1}{t_0} \EE_0\|\zb\|^2 + (T-\log t_0)\cdot d\\
    &\leq \frac{1}{t_0} \Ctwo + (T-\log t_0)\cdot d.
\end{align*}

\paragraph{Upper bounds for $\myEE \|\gb(\zb,t)\|^2_2 $.}
\[
    \myEE \|g(\zb,t)\|^2_2 \leq 2 \myEE h(t)^2 \| \nabla \log \ptz(\zb)\|^2_2 + 2 \myEE \|\zb\|^2_2. 
\]
By Lemma~\ref{lem:score}, we have 
\begin{align*}
\myEE h(t)^2 \| \nabla \log \ptz(\zb)\|^2_2 &= \int_{t_0}^T \EE_t \| \nabla \log \ptz(\zb)\|^2_2\mathrm{d}t  \\
&\leq \int_{t_0}^T  \min \left\{\frac{1}{1-h(t)}\EE_0 \| \nabla \log p_z(\zb)\|^2_2 , \frac{1}{h(t)}d \right\} \mathrm{d}t.
\end{align*}
We see that when $t$ increases, $1/(1-h(t))$ increases and $1/h(t)$ decreases. By setting 
\[
     \frac{1}{1-h(t^*)}\EE_0 \| \nabla \log p_z(\zb)\|^2_2  = \frac{1}{h(t^*)}d
\]
we have 
\[
    t^* =    h^{-1}\left( \frac{d}{d+\EE_0 \| \nabla_{\zb} \log p_z(\zb)\|^2_2 }  \right).
\]
Notice that we have chosen $t_0 \leq \log(d/\CE   + 1 )$, where $\EE_0 \| \nabla \log p_z(\zb)\|^2_2\leq \CE$. Then we have 
\[
    t_0 \leq   \log(d/\CE   + 1 ) \leq  \log(d/\EE_0 \| \nabla \log p_z(\zb)\|^2_2    + 1 )  = t^*.
\]
Therefore
\begin{align*}
    \myEE h(t)^2 \| \nabla \log \ptz(\zb)\|^2_2 & \leq \int_{t_0}^{t^* \wedge T} \frac{1}{1-h(t)}\mathrm{d}t \cdot   \EE_0 \| \nabla \log p_z(\zb)\|^2_2  +  \int_{t^* \wedge T}^{T} \frac{1}{h(t)}\mathrm{d}t \cdot  d \\
    &\leq \exp(t^*)\cdot \EE_0 \| \nabla \log p_z(\zb)\|^2_2  + (T - \log (t^* \wedge T))\cdot d \\
    &\leq (d+\EE_0 \| \nabla \log p_z(\zb)\|^2_2 ) +  d(T-\log t_0) \\
    &\lesssim \EE_0 \| \nabla \log p_z(\zb)\|^2_2 + d(T-\log t_0) .
\end{align*}

\paragraph{Lower bounds for $\lambda_{\min} (\myEE \gb(\zb,t)\gb(\zb,t)^\top) $.}
By Lemma~\ref{lem:score}, we have 
\begin{align*}
\EE_t \gb(\zb,t)\gb(\zb,t)^\top &= \EE_t \zb\zb ^\top +  h(t)^2 \EE_t \nabla \log \ptz(\zb)\nabla \log \ptz(\zb)^\top  \\
& \quad + h(t)  \EE_t \nabla \log \ptz(\zb) \zb  ^\top  +  h(t)  \EE_t  \zb \nabla \log \ptz(\zb)  ^\top \\
&= (1-h(t)) \EE_0 \zb \zb^\top - h(t) I  + h^2(t) \EE_t \nabla \log \ptz(\zb)\nabla \log \ptz(\zb)^\top \\
&\succeq  (1-h(t)) \EE_0 \zb \zb^\top - h(t) I.
\end{align*}
Denote $\lambda_0 = \lambda_{\min} (\EE_0 \zb \zb^\top) $, then we have for any $t_0 \leq T^* \leq T$,
\begin{align*}
\lambda_{\min} (\myEE \gb(\zb,t)\gb(\zb,t)^\top) & \geq \int_{t_0}^{T^*} \Big( \frac{1-h(t)}{h^2(t)} \lambda_0 - \frac{1}{h(t)} \Big) \mathrm{d}t .
\end{align*}
Taking maximum w.r.t. to $T^*$ and we get:
\[
    T^* = h^{-1}(\lambda_0 / (\lambda_0+1)).
\]
We need to verify that the above $T^*$ lies in $[t_0,T]$. Notice that we have $d\lambda_0 \leq \EE_{0} \|\zb\|^2  \leq \Ctwo$. By the assumptions that $t_0 \leq \log(1+\Czero)$ and $T \geq \log( \Ctwo / d+1) $, we have 
\[
    T \geq \log( \Ctwo / d+1) \geq  \log(1+\lambda_0)   = T^* ,
\]
and
\[
    t_0 \leq \log(1+\Czero) \leq  \log(1+\lambda_0)  = T^*.
\]  

Therefore
\begin{align*}
    \lambda_{\min} (\myEE \gb(\zb,t)\gb(\zb,t)^\top) & \geq \int_{t_0}^{T^*} \Big( \frac{1-h(t)}{h^2(t)} \lambda_0 - \frac{1}{h(t)} \Big) \mathrm{d} t \\
    & \geq \Big [\frac{1}{1-\exp(T^*)} - \frac{1}{1 - \exp(t_0) } \Big ]\lambda_0 - (T^* - \log t_0) \\
    & =  \frac{1}{\exp(t_0)-1} \lambda_0 -1  - \log(1+\lambda_0) + \log t_0 \\
    & \overset{(i)}{\geq} \frac{\lambda_0}{e}\frac{1}{t_0} -1  - \log(1+\lambda_0) + \log t_0 \\
    &  \overset{(ii)}{\geq} \frac{1}{2e}\frac{\lambda_0}{t_0} \\
    & \geq \frac{1}{2e}\frac{\Czero}{t_0},
\end{align*}
where we use $\exp(t_0) - 1 \leq et_0$ for $t_0\leq 1$ in $(i)$.

Then by Lemma~\ref{lem:subspace} and Lemma~\ref{lem:ortho} we know that 
\[
    \|VV^\top - AA^\top \|_{\rm F}^2 \leq \epsilon \cdot \mathcal{O} \big( \frac{t_0}{\Czero} \big)    
\]

Next we show that $(ii)$ holds. Since we have chosen $t_0 \leq \frac{\Czero}{4e\log(4e)}$, one can show that 
\[
    \frac{1}{t_0} \geq  \frac{4e}{\Czero} \log \Big( \frac{4e(1+\Czero)}{\Czero}\Big).
\]  
Then 
\begin{equation}
    \frac{1}{t_0} \geq \frac{4e}{\Czero} \log \Big( \frac{4e(1+\Czero)}{\Czero}\Big) \geq \frac{4e}{\lambda_0} \log \Big( \frac{4e(1+\lambda_0)}{\lambda_0}\Big). \label{eq:c1}
\end{equation}
By $\log(x_2/x_1) \leq x_2/x_1 - 1$, we have  
\[
    \log (\frac{e(1+\lambda_0)}{t_0})   - \log \frac{4e^2(1+\lambda_0)}{\lambda_0} \leq \frac{\lambda_0}{4e t_0} -1 .
\]
Then 
\begin{align*}
    1  + \log(1+\lambda_0) - \log t_0  = \log (\frac{e(1+\lambda_0)}{t_0})  &\leq \log \frac{4e^2(1+\lambda_0)}{\lambda_0} + \frac{\lambda_0}{4e t_0} -1  \\
    &= \log \frac{4e (1+\lambda_0)}{\lambda_0} + \frac{\lambda_0}{4e t_0}  \\
    &\leq   \frac{\lambda_0}{4et_0} + \frac{\lambda_0}{4e t_0} \tag{By \eqref{eq:c1}}  \\
    &=   \frac{\lambda_0}{2et_0}.
\end{align*}
By substituting the above bounds into Lemma~\ref{lem:main1}, we have 
\begin{align*} 
    &\quad \myEE \| U ^\top \fb_{\btheta}(U \zb, t) -  \gb(\zb, t) \|_2^2 \\
    & \lesssim \epsilon +  \frac{\epsilon}{\lambda_{\min}} \cdot \myEE \|\gb(\zb,t)\|_2^2 + \frac{\epsilon}{\lambda_{\min}} \myEE \|\zb\|_2^2\cdot \max_t \|\fb_{\btheta}(\cdot,t)\|^2_{Lip} \\
    & \lesssim \epsilon \cdot \Big[ 1+ \frac{t_0}{\Czero}\Big((T-\log t_0) d\cdot \max_t \|\fb_{\btheta}(\cdot,t)\|^2_{Lip} + \CE \Big) + \frac{\max_t \|\fb_{\btheta}(\cdot,t)\|^2_{Lip}\cdot\Ctwo}{\Czero} \Big],
\end{align*} 
where we assume $\max_t \|\fb_{\btheta}(\cdot,t)\|^2_{Lip} = \Omega(1)$.
\end{proof}

\subsubsection{Evolution of score function}
In the subsection we analyze the property of $ \nabla \log \ptz (\zb)$ in terms of the assumptions made on $ \nabla \log p_z(\zb)$. Specifically, at time $t$, the distribution  $\ptz (\zb)$ is given by 
\[
    \zb_0 \sim P_z, \quad \zb|\zb_0 \sim {\sf N}(\sqrt{1-h(t)}\zb_0, h(t) I_d).
\]
\begin{lemma}\label{lem:score}
We have the following holds 
\[
\int \ptz(\zb) \| \nabla \log \ptz(\zb) \|^2_2 \diff \zb  \leq \min \{\frac{1}{{1-h(t)}} 
     \int p_z(\zb_0)   \|\nabla \log p_z(\zb_0) \|^2_2  \mathrm{d} \zb_0, \  \frac{d}{h(t)}  \},
\]
and 
\[
    \int \ptz(\zb)  \nabla \log \ptz(\zb)  \zb ^\top  \diff \zb = -  I_d  .
\]
\end{lemma}

\begin{proof}
In the proof, we drop the superscript in $\ptz$ for simplicity and denote $p_t$ as the probability density function of $\zb$ at time $t$. We use $\phi_t(\zb|\zb_0)$ to represent the density function of $\zb|\zb_0 \sim {\sf N}(\sqrt{1-h(t)}\zb_0 , h(t)I_d)$. By Integration by parts, one can verify that 
\[
    \nabla \log p_t(\zb) = \frac{1}{\sqrt{1-h(t)}} \frac{\int p_0(\zb_0) \phi_t(\zb|\zb_0) \nabla \log p_0(\zb_0) \mathrm{d} \zb_0}{\int p_0(\zb_0) \phi_t(\zb|\zb_0)  \mathrm{d} \zb_0}.
\]
\begin{align*}
    \int p_t(\zb)  \| \nabla \log p_t(\zb) \|^2_2  \mathrm{d} \zb  &= \frac{1}{{1-h(t)}} 
    \int \frac{ \| \int p_0(\zb_0) \phi_t(\zb|\zb_0) \nabla \log p_0(\zb_0) \mathrm{d} \zb_0 \|^2_2}{ p_t( \zb)} \mathrm{d} \zb \\
    &= \frac{1}{{1-h(t)}} 
    \int \frac{ \| \EE_{p_t(\zb_0|\zb)}  [p_t(\zb) \nabla \log p_0(\zb_0)] \|^2_2}{ p_t( \zb)} \mathrm{d} \zb \\
    &\leq  \frac{1}{{1-h(t)}} 
    \int \frac{  \EE_{p_t(\zb_0|\zb)}  [p_t^2(\zb) \|\nabla \log p_0(\zb_0) \|^2_2 ]}{ p_t( \zb)} \mathrm{d} \zb \\
    &= \frac{1}{{1-h(t)}} 
    \iint {   p_t(\zb_0|\zb)  [p_t( \zb) \|\nabla \log p_0(\zb_0) \|^2_2 ]} \mathrm{d} \zb_0 \mathrm{d} \zb \\
    &= \frac{1}{{1-h(t)}} 
     \int p_0(\zb_0)   \|\nabla \log p_0(\zb_0) \|^2_2  \mathrm{d} \zb_0 .
\end{align*}
Further, we have
\begin{align*}
    \nabla \log p_t(\zb) &= \frac{\nabla p_t(\zb) }{p_t(\zb)} \\
    &=  \frac{\int  p_0(\zb_0) \nabla \phi_t(\zb|\zb_0) \diff \zb_0 }{p_t(\zb)} \\
    &=  \frac{\int  p_0(\zb_0)  \phi_t(\zb|\zb_0) \frac{-(\zb-\sqrt{1-h(t)}\zb_0)}{h(t)} \diff \zb_0 }{p_t(\zb)} .
\end{align*}
Therefore, 
\begin{align*}
    \int p_t(\zb) \| \nabla \log p_t(\zb) \|^2 \diff \zb  &= \int p_t(\zb)\frac{ \| \int  p_0(\zb_0)  \phi_t(\zb|\zb_0) \frac{-(\zb-\sqrt{1-h(t)}\zb_0)}{h(t)} \diff  \zb_0  \|^2_2}{p_t^2(\zb)}   \diff \zb \\
    &= \int p_t(\zb)\frac{ \| \int  p_t(\zb)  p_t(\zb_0|\zb) \frac{-(\zb-\sqrt{1-h(t)}\zb_0)}{h(t)} \diff  \zb_0  \|^2_2}{p_t^2(\zb)}   \diff \zb \\
    &= \int p_t(\zb){ \left\| \int  p_t(\zb_0|\zb) \frac{-(\zb-\sqrt{1-h(t)}\zb_0)}{h(t)} \diff  \zb_0  \right\|^2_2}  \diff \zb \\
    &\leq \int p_t(\zb){  \int  p_t(\zb_0|\zb)\left\| \frac{-(\zb-\sqrt{1-h(t)}\zb_0)}{h(t)} \right\|^2_2 \diff  \zb_0  }  \diff \zb \\
    &=  \int p_0(\zb_0)  \int    \phi_t(\zb|\zb_0)\left\| \frac{-(\zb-\sqrt{1-h(t)}\zb_0)}{h(t)} \right\|^2_2  \diff \zb \diff  \zb_0   \\
    &= \frac{d}{h(t)},
\end{align*}
where we use the fact that $\zb|\zb_0 \sim {\sf N}(\sqrt{1-h(t)}\zb_0, h(t) I_d)$ in the last equality.

To summarize, we have 
\[
\int p_t(\zb) \| \nabla \log p_t(\zb) \|^2_2 \diff \zb  \leq \min \{\frac{1}{{1-h(t)}} 
     \int p_0(\zb_0)   \|\nabla \log p_0(\zb_0) \|^2_2  \mathrm{d} \zb_0, \  \frac{d}{h(t)}  \}.
\]
This is tight for Gaussian.

Next we prove that 
\[
    \int p_t(\zb)  \nabla \log p_t(\zb)  \zb ^\top  \diff \zb = -  I_d.
\]
We have 
\begin{align*}
    \int p_t(\zb)  \nabla \log p_t(\zb)  \zb ^\top  \diff \zb & = \int p_t(\zb)      \frac{\int  p_0(\zb_0)  \phi_t(\zb|\zb_0) \frac{-(\zb-\sqrt{1-h(t)}\zb_0)}{h(t)} \diff \zb_0 }{p_t(\zb)} \zb ^\top  \diff \zb  \\
    & = \iint p_0(\zb_0)  \phi_t(\zb|\zb_0) \frac{-(\zb-\sqrt{1-h(t)}\zb_0)}{h(t)} \zb ^\top  \diff \zb_0   \diff \zb  \\
    & = - I_d.
\end{align*}
where we use the fact that $\zb|\zb_0 \sim {\sf N}(\sqrt{1-h(t)}\zb_0, h(t) I_d)$ in the last equality. 

\end{proof}

\subsubsection{Other lemmas.}

\begin{lemma}\label{lem:score2bound}
    Assume that $\nabla \log p_z(\zb)$ is $\beta$-Lipschitz. Then we have $\EE_{\zb \sim P_z} \| \nabla  \log p_z(\zb)\|^2_2 \leq d \beta$.
\end{lemma}
\begin{proof}
We have 
    \begin{align*}
        \EE_{\zb \sim P_z}   \nabla  \log p_z(\zb)  \nabla  \log p_z(\zb)^\top  &= \int p_z(\zb)\nabla  \log p_z(\zb)  \nabla  \log p_z(\zb)^\top \diff \zb \\
        &= \int \nabla p_z(\zb)  \nabla  \log p_z(\zb)^\top \diff \zb \\
        &= - \int p_z(\zb) \nabla  \nabla  \log p_z(\zb)^\top \diff \zb. \tag{Integration by parts.} \\
    \end{align*}
    Therefore 
    \begin{align*}
        \EE_{\zb \sim P_z} \| \nabla  \log p_z(\zb)\|^2   &=   \tr \Big[ - \int p_z(\zb) \nabla  \nabla  \log p_z(\zb)^\top \diff \zb  \Big]  \leq \beta d.
    \end{align*} 
\end{proof}
    
\subsection{Proof of Lemma~\ref{lem:de}, Undiscretized Setting}

First, we show that the Novikov's condition holds
\begin{lemma}[Novikov's condition]\label{lem:novikov}
We have 
\[
\EE  \exp\Big(\frac{1}{2} \int_{0}^{T-t_0}  \|\sld_{{\btheta}, U} (\bZb_t,T-t) - \nabla \log \pld_{T-t} (\bZb_t) \|_2^2 \mathrm{d}t \Big) < \infty,
\]
where the expectation is taken over the ground-truth latent backward diffusion process $(\bZb_t)_t$.
\end{lemma}
\begin{proof}[Proof of Lemma~\ref{lem:novikov}]
We consider the forward process $(\Zb_t)_{0 \leq t \leq T}$, which is an O-U process. We know that $(\bZb_{T-t})_{t_0\leq t \leq T}$ and $(\Zb_t)_{t_0 \leq t \leq T}$ has the same distribution. Therefore, we have 
\begin{align*}
&\quad \EE_{(\bZb_t)_t} \exp\Big(\frac{1}{2} \int_{0}^{T-t_0} \|\sld_{{\btheta}, U} (\bZb_t,T-t) - \nabla \log \pld_{T-t} (\bZb_t) \|_2^2 \mathrm{d}t \Big)  \\
&= \EE_{(\Zb_t)_t} \exp\Big(\frac{1}{2} \int_{t_0}^{T}  \|\sld_{{\btheta}, U} (\Zb_t,t) - \nabla \log \pld_{t} (\Zb_t) \|_2^2 \mathrm{d}t \Big) .
\end{align*}

The solution of $(\Zb_t)$ can be explicitly calculated as 
\[
    \Zb_t = e^{-t/2} \Zb_0 + \int_{0}^t e^{s/2} \mathrm{d} \Wb_s.
\]
And the two terms $\Zb_0$ and  $\int_{0}^t e^{s/2} \mathrm{d} \Wb_s$ are independent.

Denote $C = \max_{t \in [t_0, T]} \| \sld_{{\btheta}, U} (\cdot,t)  \|_{Lip} +  \max_{t \in [t_0, T]}  \| \nabla \log \pld_{t} (\cdot)\|_{Lip}$ and $C_0 =  \max_{t \in [t_0, T]} \| \sld_{{\btheta}, U} (\boldsymbol{0},t) - \nabla \log \pld_{t} (\boldsymbol{0}) \|_2$. By our assumptions on the Lipschitz constants of the score network and the ground truth latent score function, we have $C, C_0 < \infty$, we have
\begin{align*}
    &\quad \EE \exp\Big(\frac{1}{2} \int_{t_0}^T   \|\sld_{{\btheta}, U} (\Zb_t,t) - \nabla \log \pld_{t} (\Zb_t) \|_2^2 \mathrm{d}t \Big) \\
    & \leq 
    \EE \exp\Big(\frac{1}{2} \int_{t_0}^T  C^2 \|  \Zb_t \|^2_2 \mathrm{d}t \Big) \cdot \exp\Big(\frac{1}{2} \int_{t_0}^T  C_0^2  \mathrm{d}t \Big)  \\
    &\lesssim \EE \exp\Big(  \int_{t_0}^T  C^2 \| e^{-t/2} \Zb_0 \|^2_2   \mathrm{d}t + \int_{t_0}^T C^2\Big\|\int_{0}^t e^{s/2} \mathrm{d} \Wb_s \Big\|^2_2 \mathrm{d}t \Big)  \\
    &= \EE \exp\Big(  \int_{t_0}^T  C^2 \| e^{-t/2} \Zb_0 \|^2_2   \mathrm{d}t\Big) \cdot \EE \exp\Big( \int_{t_0}^T C^2\Big\|\int_{0}^t e^{s/2} \mathrm{d} \Wb_s \Big\|^2_2 \mathrm{d}t \Big).
\end{align*}
Since by our assumption that $\Zb_0$ is Sub-Gaussian, we have the first term is finite.  

For the second term, by Theorem 5.13 of \citep{le2016brownian}, there exists a $d$ dimensional Brownian motion $\Bb_t = (B_t^{(1)},\cdots,B_t^{(d)})$ such that 
\[
    \int_0^t e^{s/2} \mathrm{d} \Wb_s \overset{\rm a.s.}{=} \Bb_{e^t-1}.
\]
Therefore,
\begin{align*}
    \EE \exp\Big( \int_{t_0}^T C^2 \Big\|\int_{0}^t e^{s/2} \mathrm{d} \Wb_s \Big\|^2_2 \mathrm{d}t \Big) &=  \EE \exp\Big(  C^2\int_{t_0}^T \| \Bb_{e^t-1} \|^2_2 \mathrm{d}t \Big)  \\
    &= \EE \exp\Big(  C^2\int_{e^{t_0}-1}^{e^T-1} \| \Bb_{s} \|^2_2 \frac{1}{s+1} \mathrm{d}s \Big) \\
    &= \EE \exp\Big( d  C^2\int_{e^{t_0}-1}^{e^T-1} | B^{(1)}_{s} |^2 \frac{1}{s+1} \mathrm{d}s \Big) \\
    &\leq \EE \exp\Big( d  C^2\int_{e^{t_0}-1}^{e^T-1}  \frac{1}{s+1} \mathrm{d}s \cdot \sup_{0\leq s \leq t} | B^{(1)}_{s} |^2 \Big).
\end{align*}
Denote $C_2 = d  C^2 \int_{e^{t_0}-1}^{e^T-1}  \frac{1}{s+1} \mathrm{d}s < \infty$.

By the property of Brownian Motion (Theorem 2.21 of \citep{le2016brownian}), $\sup_{0\leq s\leq t} B^{(1)}_s$ has the same distribution as $|B^{(1)}_t|$, which is sub-gaussian. Since $\sup_{0\leq s\leq t} |B^{(1)}_s| \leq \sup_{0\leq s\leq t}  B^{(1)}_s - \sup_{0\leq s\leq t} ( -B^{(1)}_s)$, we know that 
\begin{align*}
    \EE \exp\Big( C_2 \sup_{0\leq s \leq t} | B^{(1)}_{s} |^2 \Big) &\leq \EE \exp\Big( C_2  \Big| \sup_{0\leq s\leq t}  B^{(1)}_s - \sup_{0\leq s\leq t} ( -B^{(1)}_s) \Big|^2 \Big) \\ 
    &\leq \EE \exp\Big( 2C_2  \Big| \sup_{0\leq s\leq t}  B^{(1)}_s \Big|^2 + \Big| \sup_{0\leq s\leq t} ( -B^{(1)}_s) \Big|^2 \Big) \\ 
    &\leq \EE^{1/2} \exp\Big( 4C_2  \Big| \sup_{0\leq s\leq t}  B^{(1)}_s \Big|^2\Big) \cdot \EE^{1/2} \exp\Big( 4C_2  \Big| \sup_{0\leq s\leq t} ( -B^{(1)}_s) \Big|^2 \Big) < \infty.
\end{align*}
\end{proof}

Then we have the following result:
\begin{lemma}\label{lem:g1}
When both started with $\bZb_0 =_d  \tildebZbRot_0 \sim P^{\sf LD}_T $, the KL divergence between the laws of the paths of the processes $(\bZb_t)_{ 0\leq t\leq T-t_0}$ and $({\tildebZbRot}_t )_{ 0\leq t\leq T-t_0}$ can be bounded by
\[
    \KL =  \EE   \Big(\frac{1}{2} \int_{0}^{T-t_0}  \|\sld_{{\btheta}, U} (\bZb_t,T-t) - \nabla \log \pld_{T-t} (\bZb_t) \|_2^2 \mathrm{d}t \Big) 
    \leq \frac{1}{2} \epsilon_{latent}(T-t_0).
\]
\end{lemma}
\begin{proof}[Proof of Lemma~\ref{lem:g1}]
Since by Lemma~\ref{lem:novikov} the Novikov's condition holds, we invoke Girsanov's Theorem~\citep{chen2022sampling} (Theorem 6).
\end{proof}

\begin{proof}[Proof of Lemma~\ref{lem:de}, part 1]
We use the same argument in \citep{chen2022sampling}. The subtlety here lies in that the initial distribution of the learned backward process~\eqref{eq:latentLearned} is ${\sf N}(0, I_d)$ rather than $P^{\sf LD}_T$. Recall that $\tilde{P}^{\sf LD}_{t_0}$ is the marginal distribution of $\tildebZbRot_{T-t_0}$ when started from ${\sf N}(0, I_d)$. We define $\tilde{Q}^{\sf LD}_{t_0}$ to be the marginal distribution of $\tildebZbRot_{T-t_0}$ when started from $\tildebZbRot_{ 0} \sim P^{\sf LD}_T$.

Then we have 
\[
    \TV({P}^{\sf LD}_{t_0}, \tilde{P}^{\sf LD}_{t_0}) \leq \TV({P}^{\sf LD}_{t_0}, \tilde{Q}^{\sf LD}_{t_0}) + \TV(\tilde{Q}^{\sf LD}_{t_0}, \tilde{P}^{\sf LD}_{t_0}) 
\]
For the first term, since marginalization only reduces the KL-divergence, we have by Lemma~\ref{lem:g1} and Pinsker's Inequality
\[
    \TV({P}^{\sf LD}_{t_0}, \tilde{Q}^{\sf LD}_{t_0})  \lesssim   \sqrt{\epsilon_{latent}(T-t_0)}.
\]
For the second term, $\tilde{P}^{\sf LD}_{t_0}$ and $\tilde{Q}^{\sf LD}_{t_0}$ are obtained through the same backward SDE but with different initial distributions. Therefore by Data Processing Inequality and Pinsker's Inequality, we know that 
\[
    \TV(\tilde{Q}^{\sf LD}_{t_0}, \tilde{P}^{\sf LD}_{t_0})  \lesssim \sqrt{\KL(\tilde{Q}^{\sf LD}_{t_0}||\tilde{P}^{\sf LD}_{t_0})} \leq \sqrt{\KL(P_T^{\sf LD}||{\sf N}(0,I_d))} \lesssim \sqrt{\KL(P_z||{\sf N}(0,I_d))} \exp(-T),
\]
where in the last inequality we use the exponential convergence of the O-U process.
\end{proof}

\subsection{Proof of Lemma~\ref{lem:de}, Discretized Setting}

Assume we choose $\eta$ as the time interval such that $T-t_0 = K_T \eta$. We first show the Novikov's condition holds.

\begin{lemma}[Novikov's condition]\label{lem:novikov2}
We have the Novikov's condition holds for the discretized setting. 
\begin{align*}
    \EE \left[\exp \Big(\sum_{k=0}^{K_T-1} \frac{1}{2} \int_{ k\eta}^{ (k+1)\eta} \left\| \frac{1}{2} \bZb_{k \eta }  + \sld_{U, \btheta}(\bZb_{k \eta},T-k \eta) -  \frac{1}{2} \bZb_t  - \nabla \log \pld_{T-t} (\bZb_t )   \right\|^2_2 \mathrm{d}t \Big)\right] < \infty,
\end{align*}
where the expectation is taken over $(\bZb_t)_{t\geq 0}$.
\end{lemma}
\begin{proof}[Proof of Lemma~\ref{lem:novikov2}]
The proof is similar to the proof of Lemma~\ref{lem:novikov}. 
\begin{align*}
    &\quad \EE \exp \Big(\sum_{k=0}^{K_T-1} \frac{1}{2} \int_{ k\eta}^{ (k+1)\eta} \| \frac{1}{2} \bZb_{k \eta }  + \sld_{U, \btheta}(\bZb_{k \eta},T-k \eta) -  \frac{1}{2} \bZb_t  - \nabla \log \pld_{T-t} (\bZb_t )   \|^2_2 \mathrm{d}t \Big) \\
    & = \EE \exp \Big(\sum_{k=0}^{K_T-1}\frac{1}{2} \int_{ T- (k+1)\eta}^{T-k\eta} \| \frac{1}{2} \Zb_{T-k \eta }  + \sld_{U, \btheta}(\Zb_{T-k \eta}, T-k \eta) -  \frac{1}{2} \Zb_t  - \nabla \log \pld_{T-t} (\Zb_t )   \|^2_2 \mathrm{d}t \Big) \\
    & \leq  \EE \exp \Big(\sum_{k=0}^{K_T-1}\frac{3}{2} \int_{ T- (k+1)\eta}^{T-k\eta} \| \frac{1}{2} \Zb_{T-k \eta }  -  \frac{1}{2} \Zb_t \|^2_2  + \|\sld_{U, \btheta}(\Zb_{T-k \eta}, T-k \eta) \|^2_2  + \| \nabla \log \pld_{T-t} (\Zb_t )   \|^2_2 \mathrm{d}t \Big) \\
    & \leq  \EE \exp \Big(\sum_{k=0}^{K_T-1}\frac{3}{2} \int_{ T- (k+1)\eta}^{T-k\eta}  C_0^2 + C^2 \|\Zb_{T-k \eta}\|^2_2  + C^2 \| \Zb_t  \|^2_2 \mathrm{d}t \Big) \\
    & = \EE \exp \Big( \frac{3C^2}{2} \int_{t_0} ^T \| \Zb_t  \|^2_2 \mathrm{d}t + (T-t_0) \frac{3C_0^2}{2}+ \frac{3C^2}{2} \sum_{k=0}^{K_T-1} \|\Zb_{T-k \eta}\|^2_2   \Big) \\
    & \overset{(i)}{\lesssim} \EE \exp \Big( \frac{3C^2(K_T+2)}{2} \int_{t_0} ^T \| \Zb_t  \|^2_2 \mathrm{d}t \Big) + \EE \exp \Big( (T-t_0) \frac{3C_0^2(K_T+2)}{2} \Big) \\
    &\quad +\sum_{k=0}^{K_T-1} \EE \exp \Big(\frac{3C^2(K_T+2)}{2}  \|\Zb_{T-k \eta}\|^2_2   \Big)\\
    & \overset{(ii)}{<} \infty.
\end{align*}
where 
\[
C_0 \lesssim \max_t\| \nabla \log \ptz(\boldsymbol{0})\|_2 + \max_t \|\sld_{U, \btheta}(\boldsymbol{0}, t ) \|_2 < \infty ,
\] 
\[
     C \lesssim 1 +\max_t \| \nabla \log \ptz(\cdot)\|_{Lip} + \max_t \|\sld_{U, \btheta}(\cdot, t ) \|_{Lip}  < \infty.     
\]
and in $(i)$ we use 
\[
    \EE A_1\cdot \cdots \cdot A_n \leq \frac{\EE A_1^n +\cdots + \EE A_n^n}{n},
\]
and in $(ii)$ we use the fact that $\Zb_0$ is subGaussian, and a similar argument in the proof of Lemma~\ref{lem:novikov}.
\end{proof}

\begin{lemma} \label{lem:kl:2}
When both started with $\bZb_0 =_d  \tildebZbRotDis_0 \sim P^{\sf LD}_T$, the KL divergence between the laws of the paths of the processes $(\bZb_t)_{ 0\leq t\leq T-t_0}$ and $({\tildebZbRotDis}_t )_{ 0\leq t\leq T-t_0}$ can be bounded by
\begin{align*}
        \KL &= \sum_{k=0}^{K_T-1}\EE \Big( \int_{ k\eta}^{ (k+1)\eta} \| \frac{1}{2} \bZb_{k \eta }  + \sld_{U, \btheta}(\bZb_{k \eta},T-k \eta) -  \frac{1}{2} \bZb_t  - \nabla \log \pld_{T-t} (\bZb_t )   \|^2_2 \mathrm{d}t \Big) \\
        &\lesssim \Big(\frac{\max_{\zb} \| \fb_{\btheta} (\zb , \cdot) \|_{Lip} }{h(t_0)} +\frac{\max_{\zb,t} \| \fb_{\btheta} (\zb,t) \|_2}{t_0^2} \Big)^2 \eta^2 (T-t_0) + \Big( \frac{\max_t \| \fb_{\btheta} (\cdot , t) \|_{Lip}}{h(t_0)}   \Big)^2  \eta^2 (T-t_0)\max\{ \EE \|\Zb_0 \|^2_2, d\} \\
        & \quad + \eta (T-t_0) d + \epsilon_{latent}(T-t_0).
\end{align*}

\end{lemma}
\begin{proof}[Proof of Lemma~\ref{lem:kl:2}]
 Since by Lemma~\ref{lem:novikov2} the Novikov's condition holds, we can invoke Girsanov's Theorem as in \citep{chen2022sampling} (Theorem 6). Next we provide an upper bound on the discretized score matching error.
\begin{align*}
    &\quad \EE \Big(\frac{1}{2} \int_{ k\eta}^{ (k+1)\eta} \| \frac{1}{2} \bZb_{k \eta }  + \sld_{U, \btheta}(\bZb_{k \eta},T-k \eta) -  \frac{1}{2} \bZb_t  - \nabla \log \pld_{T-t} (\bZb_t )   \|^2_2 \mathrm{d}t \Big) \\
    & \leq \EE \Big( \int_{ k\eta}^{ (k+1)\eta} \|  \sld_{U, \btheta}(\bZb_{k \eta},T-k \eta) - \nabla \log \pld_{T-t} (\bZb_t )   \|^2_2 \mathrm{d}t \Big) +  \EE  \int_{ k\eta}^{ (k+1)\eta} \|   \frac{1}{2} \bZb_{k \eta }  -  \frac{1}{2} \bZb_t  \|^2_2 \diff t \\ 
\end{align*}
We decompose the first term as 
\begin{align*}
    &\quad   \EE \Big( \int_{ k\eta}^{ (k+1)\eta} \|  \sld_{U, \btheta}(\bZb_{k \eta},T-k \eta) - \nabla \log \pld_{T-t} (\bZb_t )   \|^2_2 \mathrm{d}t \Big)  \\
    &\lesssim \EE \Big( \int_{ k\eta}^{ (k+1)\eta} \|  \sld_{U, \btheta}(\bZb_{k \eta},T-k \eta) -  \sld_{U, \btheta}(\bZb_{k \eta},T-t)   \|^2_2 \mathrm{d}t \Big)  \\
    &\quad + \EE \Big( \int_{ k\eta}^{ (k+1)\eta} \|  \sld_{U, \btheta}(\bZb_{k \eta},T-t) - \sld_{U, \btheta}(\bZb_{t},T-t)   \|^2_2 \mathrm{d}t \Big)  \\
    &\quad  + \EE \Big( \int_{ k\eta}^{ (k+1)\eta} \|  \sld_{U, \btheta}(\bZb_{t},T-t) - \nabla \log \pld_{T-t} (\bZb_t )   \|^2_2 \mathrm{d}t \Big)  \\
    &\lesssim \EE \Big( \int_{ k\eta}^{ (k+1)\eta} \| \bar{L}_t  (t-k\eta)  \|^2_2 \mathrm{d}t \Big)  \\
    &\quad + \EE \Big( \int_{ k\eta}^{ (k+1)\eta} \bar{L}_z^2 \|  \bZb_{k \eta} - \bZb_{t}   \|^2_2 \mathrm{d}t \Big)  \\
    &\quad  + \EE \Big( \int_{ k\eta}^{ (k+1)\eta} \|  \sld_{U, \btheta}(\bZb_{t},T-t) - \nabla \log \pld_{T-t} (\bZb_t )   \|^2_2 \mathrm{d}t \Big).
\end{align*}

For any $s\leq t$,
\[
    \EE  \|  \Zb_{s} - \Zb_{t}   \|^2 \mathrm{d}t \lesssim (t-s)^2 \EE \| \Zb_s\|^2_2 + (t-s) d  \leq (t-s)^2 \max\{ \EE \|\Zb_0 \|^2_2 , d\} + (t-s)d.
\]
Therefore
\begin{align*}
    &\quad \EE \Big( \int_{ k\eta}^{ (k+1)\eta}   \|  \bZb_{k \eta} - \bZb_{t}   \|^2_2 \mathrm{d}t \Big) \\
    & \leq \EE \Big( \int_{ k\eta}^{ (k+1)\eta} \big[ (t-k\eta)^2 \max\{ \EE \|\Zb_0 \|^2_2 , d\} + (t-k\eta)d \big] \mathrm{d}t \Big) \\
    &\lesssim \eta^3 \max\{ \EE \|\Zb_0 \|^2_2 , d\} + \eta^2 d.
\end{align*}

Finally we have 
\begin{align*}
    &\quad \sum_{k=0}^{K_T-1}\EE \Big( \int_{ k\eta}^{ (k+1)\eta} \| \frac{1}{2} \bZb_{k \eta }  + \sld_{U, \btheta}(\bZb_{k \eta},T-k \eta) -  \frac{1}{2} \bZb_t  - \nabla \log \pld_{T-t} (\bZb_t )   \|^2_2 \mathrm{d}t \Big) \\
    &\lesssim \bar{L}_t^2 \eta^2 (T-t_0) + (1+\bar{L}_z^2)  \eta^2 (T-t_0)\max\{ \EE \|\Zb_0 \|^2_2 , d\} + \eta (T-t_0) d + \epsilon_{latent}(T-t_0).
\end{align*}
where  
\[
    \bar{L}_z  \overset{def}{=} \max_t  \| \sld_{U, \btheta}(\cdot ,t) \|_{Lip} \leq  \frac{1}{h(t_0)} (1+\max_t \| \fb_{\btheta} (\cdot , t) \|_{Lip}),
\]
and 
\[
    \bar{L}_t \overset{def}{=}  \max_{\zb} \| \sld_{U, \btheta}(\zb ,\cdot)\|_{Lip} \leq \frac{\max_{\zb} \| \fb_{\btheta} (\zb , \cdot) \|_{Lip} }{h(t_0)} +\frac{\max_{\zb,t} \| \fb_{\btheta} (\zb,t) \|_2}{t_0^2}.
\]
To see why the above two bounds on the Lipschitz constants hold, notice that 
\[
    \sld_{U, \btheta}(\zb,t) =  \frac{1}{h(t)}\Big[ -\zb  + U^\top \fb_{\btheta} (U \zb  , t) \Big] .    
\]
To calculate the Lipschitz constant of $\frac{a(t)}{b(t)}$, notice that 
\[
    \left\vert \frac{a(t)}{b(t)} - \frac{a(s)}{b(s)} \right\vert \leq \left\vert \frac{a(t)}{b(t)} - \frac{a(s)}{b(t)} \right\vert + \left\vert \frac{a(s)}{b(t)} - \frac{a(s)}{b(s)} \right\vert \leq \frac{\|a\|_{Lip}|t-s|}{\min_t |b(t)|} + \max_t |a(t)| \cdot |t-s|\cdot  \|1/b\|_{Lip}.
\] 
We use the fact that 
\[
    \left\| \frac{1}{h(t)}\right\|_{Lip} = \max_{t\in[t_0,T]} \left | \frac{h'(t)}{h^2(t)}\right | = \frac{1}{e^{t_0} + e^{-t_0} - 2} \leq \frac{1}{t_0^2}.
\]
\end{proof}

\begin{proof}[proof of Lemma~\ref{lem:de}, part 2]
 For the discretized setting, only notice that by Lemma~\ref{lem:kl:2} there is an additional error term $\epsilon_{dis}(T-t_0)$.
\end{proof} 
    \subsection{proof of Lemma~\ref{lem:ortho:process}}
\begin{proof}[proof of Lemma~\ref{lem:ortho:process}]

Define $\psi(t) = \exp \int_0^t \Big[\frac{1}{h(T-s)} - \frac{1}{2} \Big] \diff s$. Plug in $h(t)=1-\exp(-t)$, we have 
\[
    \psi(t)= \frac{e^T-1}{e^T -e^t}e^{t/2}.
\]

We know that the solution of $\Yb_t$ is 
\[
    \Yb_t  =   \frac{1}{\psi(t)}\Big[ \Yb_0 + \int_0^t \psi(s) \diff \Bb_s \Big].
\]

\[
    \int_{0}^t    \psi(s)^2 \diff s = (e^T-1)^2 [1/(e^T-e^t) - 1/(e^T-1)].
\]

When $\Yb_0 \sim {\sf N}(0, I)$, we have 
\[
    \Yb_t \sim {\sf N}\Big(0, \frac{1 + \int_{0}^t    \psi(s)^2 \diff s}{\psi(t)^2} I\Big).
\]

We provide an upper bound of 
\begin{align*}
    V_t \overset{def}{=} \frac{1 + \int_{0}^t    \psi(s)^2 \diff s}{\psi(t)^2} &\leq \frac{(e^T-1)^2 [1/(e^T-e^t)  ]}{\psi(t)^2} \tag{when $T>1$} \\
    & = (e^T-e^t)/e^t = e^{T-t} - 1.
\end{align*}
Therefore, we have when $t_0 \leq 1$
\[
    V_{T-t_0} \leq e^{t_0}-1 \leq et_0.
\]
To conclude, we know that $\Yb_{T-t_0}$ is a zero-mean Gaussian random variable with covariance bounded by $et_0 I$.

\end{proof}

\begin{proof}[proof of Lemma~\ref{lem:ortho:process:dis}]
Denote $\alpha(t) = \frac{1}{h(T-t)} - \frac{1}{2}$. We know that 
\[
    \Yb_{(k+1)\eta} - \Yb_{k\eta}     = - \eta \alpha(k\eta) \Yb_{k\eta}  + \Bb_{(k+1)\eta} - \Bb_{k\eta}.
\]
Denote by $V_{k}$ the variance of $\Yb_{k\eta}$. We know that $\Yb_{k\eta} \sim {\sf N}(0, V_k)$. And we have the following recursion
\[
    V_0 = 1, \text{ and } V_{k+1} = (1-\alpha(k\eta) \eta )^2 V_k + \eta.    
\]
By solving the recursion we know that 
\[
    V_{K_T} = \prod_{k=0}^{K_T-1} \Big[1-\alpha(k\eta) \eta\Big]^2 + \eta \sum_{i=1}^{K_T-1} \Big[ \prod_{k=i}^{K_T-1} \big[1-\alpha(k\eta) \eta\big]^2 \Big]
\]

Define $\psi(t) = \exp \int_0^t \alpha(s) \diff s$. We have  
\[
    \psi(t)= \frac{e^T-1}{e^T -e^t}e^{t/2}.
\]
Since $\alpha(t)$ is monotonically increasing, we have 
\begin{align*}
    \prod_{k=k_1}^{k_2} \Big[1-\alpha(k\eta) \eta\Big] &\leq \prod_{k=k_1}^{k_2} \exp \Big[-\alpha (k\eta) \eta\Big] \\
    &\leq  \exp \Big[ - \sum_{k=k_1}^{k_2} \alpha (k\eta) \eta\Big] \\
    &\leq \exp \Big[-\int_{(k_1 -1)\eta}^{k_2\eta} \alpha(t) \diff t \Big]\\
    &= \frac{\psi((k_1-1)\eta)}{\psi(k_2\eta)}.
\end{align*}
Therefore we have 
\begin{align*} 
    V_{K_T} \leq \frac{\psi^2(-\eta)}{\psi^2((K_T-1)\eta)} +  \eta \sum_{k=1}^{K_T-1}\frac{\psi^2((k-1)\eta)}{\psi^2((K_T-1)\eta)}.
\end{align*}
Since $\psi(t)\geq 0 $ and $\psi(t)$ monotonically increases, we have 
\begin{align*} 
    V_{K_T} &\leq \frac{\psi^2(-\eta)+  \eta \sum_{k=1}^{K_T-1} \psi^2((k-1)\eta)}{\psi^2((K_T-1)\eta)} \\
    &\leq \frac{\psi^2(-\eta)+   \int_{0}^{(K_T-1)\eta}   \psi^2(t) \diff t}{\psi^2((K_T-1)\eta)}.
\end{align*}
By 
\[
    \int_{0}^t    \psi(s)^2 \diff s = (e^T-1)^2 [1/(e^T-e^t) - 1/(e^T-1)]
\]
We have 
\begin{align*} 
    V_{K_T}  &\leq \frac{\psi^2(-\eta)+   \int_{0}^{(K_T-1)\eta}   \psi^2(t) \diff t}{\psi^2((K_T-1)\eta)} \\
    &\leq \frac{\psi^2(-\eta)+ (e^T-1)^2 [1/(e^T-e^{T-t_0-\eta}) - 1/(e^T-1)] }{\psi^2(T-t_0-\eta)} \\ 
    &\leq \frac{1+  (e^T-1)^2 [1/(e^T-e^{T-t_0-\eta}) - 1/(e^T-1)] }{\psi^2(T-t_0-\eta)} \tag{$\psi^2(-\eta) \leq 1$} \\ 
    &\leq \frac{ (e^T-1)^2 (e^T-e^{T-t_0-\eta})  }{\psi^2(T-t_0-\eta)} \tag{when $T\geq 1$} \\
    &\leq e^{t_0+\eta} - 1 \\
    &\leq e(t_0 + \eta). \tag{when $t_0 + \eta \leq 1$}
\end{align*}
\end{proof}

\section{Helper lemmas}
We collect technical results frequently used in previous proofs. We group them according to topics: concentration inequality, Gaussian integral tail bounds, matrix norm inequalities.

\paragraph{Bernstein-type concentration inequality} The following concentration bound is useful in the proof of Theorem \ref{thm:score_estimation}.
\begin{lemma}\label{lemma:concentration_bernstein}
Let $\cG$ be a bounded function class, i.e., there exists a constant $B$ such that any $g \in \cG : \RR^d \mapsto [0, B]$. Let $\zb_1, \dots, \zb_n \in \RR^d$ be i.i.d. random variables. For any $\delta \in (0, 1)$, $a \leq 1$, and $\tau > 0$, we have
\begin{align*}
& \PP\left(\sup_{g \in \cG} \frac{1}{n} \sum_{i=1}^n g(\zb_i) - (1+a)\EE[g(\zb)] > \frac{(1+3/a)B}{3n}\log \frac{\cN(\tau, \cG, \norm{\cdot}_\infty)}{\delta} + (2+a)\tau \right) \leq \delta \quad \text{and} \\
& \PP\left(\sup_{g \in \cG} \EE[g(\zb)] - \frac{1+a}{n} \sum_{i=1}^n g(\zb_i) > \frac{(1+6/a)B}{3n}\log \frac{\cN(\tau, \cG, \norm{\cdot}_\infty)}{\delta} + (2+a)\tau \right) \leq \delta.
\end{align*}
\end{lemma}
\begin{proof}
The proof utilizes Bernstein-type inequalities. Consider the deviation $\sup_{g \in \cG} \frac{1}{n} \sum_{i=1}^n g(\zb_i) - (1+a)\EE[g(\zb)]$ first. Let $\{g_k\}_{k=1}^{\cN(\tau, \cG, \norm{\cdot}_{\infty})}$ be a discretization of $\cG$, where $\cN(\tau, \cG, \norm{\cdot}_{\infty})$ is the covering number with respect to the function $L_{\infty}$ norm. Then we have
\begin{align*}
\sup_{g \in \cG} \frac{1}{n} \sum_{i=1}^n g(\zb_i) - (1+a)\EE[g(\zb)] \leq \max_{k} \frac{1}{n} \sum_{i=1}^n g_k(\zb_i) - 2\EE[g_k(\zb)] + (2+a)\tau,
\end{align*}
as for any $g \in \cG$, we can find some $g_{k^\star}$ such that $\norm{g - g_{k^\star}}_\infty \leq \tau$. Therefore, it is enough to show
\begin{align*}
\PP\left(\max_{k} \frac{1}{n} \sum_{i=1}^n g_k(\zb_i) - (1+a)\EE[g_k(\zb)] > \frac{(1+3/a)B}{3n}\log \frac{\cN(\tau, \cG, \norm{\cdot}_\infty)}{\delta} \right) \leq \delta.
\end{align*}
By union bound, we have
\begin{align*}
& \PP\left(\max_{k} \frac{1}{n} \sum_{i=1}^n g_k(\zb_i) - (1+a)\EE[g_k(\zb)] > \frac{(1+3/a)B}{3n}\log \frac{\cN(\tau, \cG, \norm{\cdot}_\infty)}{\delta} \right) \\
& \hspace{1in} \leq \cN(\tau, \cG, \norm{\cdot}_\infty) \PP\left(\frac{1}{n} \sum_{i=1}^n g_1(\zb_i) - (1+a)\EE[g_1(\zb)] > \frac{(1+3/a)B}{3n}\log \frac{\cN(\tau, \cG, \norm{\cdot}_\infty)}{\delta} \right).
\end{align*}
Therefore, it further suffices to provide an upper bound on
\begin{align*}
\PP\left(\frac{1}{n} \sum_{i=1}^n g(\zb_i) - (1+a)\EE[g(\zb)] > \frac{(1+3/a)B}{3n}\log \frac{\cN(\tau, \cG, \norm{\cdot}_\infty)}{\delta} \right),
\end{align*}
where $g \in \cG$ is any fixed function. Let $\lambda > 0$ be some parameter to be chosen later. Chernoff bound yields
\begin{align}\label{eq:chernoff}
& \PP\left(\frac{1}{n} \sum_{i=1}^n g(\zb_i) - (1+a)\EE[g(\zb)] > \frac{(1+3/a)B}{3n}\log \frac{\cN(\tau, \cG, \norm{\cdot}_\infty)}{\delta} \right) \nonumber \\
& \hspace{1.7in} \leq \frac{\EE\left[\exp\left(\lambda\left(\frac{1}{n} \sum_{i=1}^n g(\zb_i) - (1+a)\EE[g(\zb)]\right)\right)\right]}{\exp\left(\frac{(1+3/a)\lambda B}{3n}\log \frac{\cN(\tau, \cG, \norm{\cdot}_\infty)}{\delta}\right)}.
\end{align}
It remains to find $\EE\left[\exp\left(\lambda\left(\frac{1}{n} \sum_{i=1}^n g(\zb_i) - (1+a)\EE[g(\zb)]\right)\right)\right]$. We rewrite
\begin{align*}
\frac{1}{n} \sum_{i=1}^n g(\zb_i) - (1+a)\EE[g(\zb)] = \frac{1}{n} \sum_{i=1}^n g(\zb_i) - a\EE[g(\zb)] - \EE[g(\zb)] \leq \frac{1}{n} \sum_{i=1}^n g(\zb_i) - \EE[g(\zb)] - \frac{a}{B} \EE[g^2(\zb)].
\end{align*}
Introducing independent ghost samples $\bar{\zb}_1, \dots, \bar{\zb}_n$, we have
\begin{align*}
\frac{1}{n} \sum_{i=1}^n g(\zb_i) - \EE[g(\zb)] - \frac{a}{B} \EE[g^2(\zb)] & = \frac{1}{n} \sum_{i=1}^n g(\zb_i) - \EE_{\bar{\zb}} \left[\frac{1}{n} \sum_{i=1}^n g(\bar{\zb}_i)\right] - \frac{a}{B} \EE[g^2(\zb)] \\
& = \EE_{\bar{\zb}} \left[\frac{1}{n} \sum_{i=1}^n g(\zb_i) - g(\bar{\zb}_i) \right] - \frac{a}{2B} \EE[g^2(\zb) + g^2(\bar{\zb})] \\
& \overset{(i)}{\leq} \EE_{\bar{\zb}} \left[\frac{1}{n} \sum_{i=1}^n g(\zb_i) - g(\bar{\zb}_i) \right] - \frac{a}{2B} \Var\left[g(\zb) - g(\bar{\zb})\right],
\end{align*}
where inequality $(i)$ invokes identity $\Var\left[g(\zb) - g(\bar{\zb})\right] = \EE[(g(\zb) - g(\bar{\zb}))^2] \leq \EE[g^2(\zb) + g^2(\bar{\zb})]$. For convenience, we denote $h_i = g(\zb_i) - g(\bar{\zb}_i)$. For $0 < \lambda < 3n/B$, we compute
\begin{align*}
\EE\left[\exp\left(\frac{\lambda}{n} h_i\right)\right] & = \EE\left[1 + \frac{\lambda}{n} h_i + \sum_{j=2}^\infty \frac{(\lambda/n)^j h_i^j}{j!}\right] \\
& \overset{(i)}{\leq} 1 + \EE\left[\sum_{j=2}^\infty \frac{(\lambda/n)^j B^{j-2}}{2 \cdot 3^{j-2}} h_i^2\right] \\
& = 1 + \frac{\lambda^2}{2n^2} \frac{1}{1 - \frac{\lambda B}{3n}} \EE[h_i^2] \\
& \overset{(ii)}{\leq} \exp\left(\frac{3\lambda^2}{6n^2 - 2\lambda Bn} \Var(h_i)\right),
\end{align*}
where inequality $(i)$ follows from $\EE[h_i] = 0$ and $|h_i| \leq B$, and inequality $(ii)$ invokes $1 + x \leq \exp(x)$ for $x \geq 0$. To this end, we derive
\begin{align*}
\EE\left[\exp\left(\lambda\left(\frac{1}{n} \sum_{i=1}^n g(\zb_i) - 2\EE[g(\zb)]\right)\right)\right] & \overset{(i)}{\leq} \EE\left[\frac{\lambda}{n} \sum_{i=1}^n h_i - \frac{\lambda a}{2Bn} \sum_{i=1}^n \Var[h_i]\right] \\
& \leq \exp\left(\frac{3\lambda^2}{6n^2 - 2\lambda B n} \sum_{i=1}^n \Var[h_i] - \frac{\lambda a}{2Bn} \sum_{i=1}^n \Var[h_i]\right),
\end{align*}
where $(i)$ follows from Jensen's inequality.
We choose $\lambda = \frac{3n}{(1+3/a)B}$, which satisfies $\frac{3\lambda^2}{6n^2 - 2\lambda B n} = \frac{\lambda a}{2Bn}$ and $\lambda < 3n/B$. Substituting into \eqref{eq:chernoff}, we obtain
\begin{align*}
\PP\left(\frac{1}{n} \sum_{i=1}^n g(\zb_i) - (1+a)\EE[g(\zb)] > \frac{(1+3/a)B}{3n}\log \frac{\cN(\tau, \cG, \norm{\cdot}_\infty)}{\delta} \right) & \leq \exp\left(- \log \frac{\cN(\tau, \cG, \norm{\cdot}_\infty)}{\delta}\right) \\
& = \frac{\delta}{\cN(\tau, \cG, \norm{\cdot}_{\infty})}.
\end{align*}
Therefore, the first inequality is proved. The second inequality can be proved in the exact same argument, by observing
\begin{align*}
\EE[g(\zb)] - \frac{1+a}{n} \sum_{i=1}^n g(\zb_i) & = 2 \left(\EE[g(\zb)] - \frac{1}{n} \sum_{i=1}^n g(\zb_i) - \frac{a}{2} \EE[g(\zb)]\right) \\
& \leq 2 \left(\EE[g(\zb)] - \frac{1}{n} \sum_{i=1}^n g(\zb_i) - \frac{a}{2B} \EE[g^2(\zb)]\right).
\end{align*}
The proof is complete.
\end{proof}

\paragraph{Tail bound of Gaussian integral}
Tail bounds of Gaussian integrals appear frequently in score approximation and estimation theories. We show the following results.
\begin{lemma}\label{lemma:gaussian_tail}
Consider a probability density function $p(\xb) = \exp\left(-C \norm{\xb}_2^2 / 2\right)$ for $\xb \in \RR^d$ and constant $C > 0$. Let $R > 0$ be a fixed radius. Then it holds
\begin{align*}
\int_{\norm{\xb}_2 > R} p(\xb) \diff \xb & \leq \frac{2d \pi^{d/2}}{C\Gamma(d/2 + 1)} R^{d-2} \exp(-CR^2/2), \\
\int_{\norm{\xb}_2 > R} \norm{\xb}_2^2 p(\xb) \diff \xb & \leq \frac{2d \pi^{d/2}}{C\Gamma(d/2 + 1)} R^{d} \exp(-CR^2 / 2).
\end{align*}
\end{lemma}
\begin{proof}
We apply change of variable using polar coordinate systems. For the first integral, we have
\begin{align*}
\int_{\norm{\xb}_2 > R} p(\xb) \diff\xb & = \int_{\norm{\xb}_2 > R} \exp(-C \norm{\xb}_2^2 / 2) \diff \xb \\
& = \int_{R}^\infty \int_{\theta_1, \dots, \theta_{d-1}} r^{d-1} \exp\left(-C r^2 / 2\right) \prod_{j=1}^{d-2} \sin^{d-j-1} (\theta_j) ~\diff r \diff\theta_1 \dots \diff\theta_{d-1} \\
& \overset{(i)}{=} \frac{d \pi^{d/2}}{\Gamma(d/2+1)} \int_R^{\infty} r^{d-1} \exp\left(-C r^2 / 2\right) \diff r \\
& \overset{(ii)}{=} \frac{d (2\pi)^{d/2}}{2C^{d/2}\Gamma(d/2 + 1)} \int_{CR^2/2}^\infty u^{d/2 - 1} \exp(-u) \diff u \\
& = \frac{ (2\pi)^{d/2}}{C^{d/2}\Gamma(d/2 + 1)} \int_{(CR^2/2)^{d/2}}^\infty \exp\left(-v^{2/d}\right) \diff v \\
& \overset{(iii)}{\leq} \frac{2d \pi^{d/2}}{C\Gamma(d/2 + 1)} R^{d-2} \exp(-CR^2/2).
\end{align*}
In $(i)$, we invoke the identity $\int_0^1 \int_{\theta_1, \dots, \theta_{d-1}} r^{d-1} \prod_{j=1}^{d-2} \sin^{d-j-1} (\theta_j) ~\diff r \diff\theta_1 \dots \diff\theta_{d-1} = \int_{\norm{\xb}_2 \leq 1} \diff \xb = \frac{\pi^{d/2}}{\Gamma(d/2 + 1)}$ being the volume of a unit $d$-ball. To obtain $(ii)$, we change variable by letting $u = Cr^2/2$. Inequality $(iii)$ bounds the upper tail of incomplete gamma function \citep[Inequality (10) with $\alpha = 2/d, A = -d$]{qi1999some}.

A similar argument can be applied to the second integral:
\begin{align*}
\int_{\norm{\xb}_2 > R} \norm{\xb}_2^2 p(\xb) \diff\xb & = \int_{\norm{\xb}_2 > R} \norm{\xb}_2^2 \exp(-C \norm{\xb}_2^2 / 2) \diff \xb \\
& = \int_{R}^\infty \int_{\theta_1, \dots, \theta_{d-1}} r^{d+1} \exp\left(-C r^2 / 2\right) \prod_{j=1}^{d-2} \sin^{d-j-1} (\theta_j) ~\diff r \diff\theta_1 \dots \diff\theta_{d-1} \\
& = \frac{d \pi^{d/2}}{\Gamma(d/2+1)} \int_R^{\infty} r^{d+1} \exp\left(-C r^2 / 2\right) \diff r \\
& = \frac{d \pi^{d/2}}{(d+2)\Gamma(d/2 + 1)} \left(\frac{2}{C}\right)^{d/2+1} \int_{(CR^2/2)^{d/2 + 1}}^\infty \exp\left(-v^{2/(d+2)}\right) \diff v \\
& \leq \frac{2d \pi^{d/2}}{C\Gamma(d/2 + 1)} R^{d} \exp(-CR^2 / 2).
\end{align*}
The proof is complete.
\end{proof}

\paragraph{Matrix norm inequalities} The following lemma deals with matrices with orthonormal columns, whose linear span is approximately equal. These are useful results in deriving score estimation error bounds in Theorem \ref{thm:distro_estimation}.
\begin{lemma}\label{lem:ortho}
Let $A, V \in \RR ^{D \times d}$ with $d < D$ be two matrices with orthonormal columns, i.e., $A^\top A = V^\top V = I_d$. Given any $\epsilon > 0$, if $\|  (I_D-VV^\top)A\|^2_{\rm F} \leq \epsilon$, then the following holds
\begin{enumerate}[label=(\alph*).]
\item
\begin{align*}
        \left\| (I_D - A A^\top) V\right\|_{\rm F}^2 &\leq \epsilon,  \\
        \left\| V V^\top - A A^\top \right\| ^2_{\rm F} &\leq   2\epsilon, \\
        \left\|V^\top A A^\top V - I_d \right\|_{\rm F}^2 &\leq   2 \epsilon.
\end{align*}
\item
    There exists an orthogonal matrix $U \in \RR^{d\times d}$ such that
    \[
        \left\| U - V^\top A \right\|_{\rm F}^2 \leq 2 \epsilon.
    \]
\end{enumerate}
\end{lemma}
\begin{proof}[Proof of Lemma~\ref{lem:ortho}]
The first set of results in item (a) follows from some algebraic manipulation. Consider $\norm{\left(I_D - AA^\top\right) V}_{\rm F}^2$ first. We have
\begin{align*}
\norm{\left(I_D - AA^\top\right) V}_{\rm F}^2 & = {\rm Tr}\left(\left(V - AA^\top V\right)\left(V - AA^\top V\right)^{\top} \right) \\
& = \tr\left(VV^\top - AA^\top VV^\top \right) \\
& \overset{(i)}{=} \frac{1}{2} \tr \left(VV^\top - AA^\top VV^\top - VV^\top AA^\top + AA^\top \right) \\
& = \frac{1}{2} \tr\left(\left(AA^\top - VV^\top\right) \left(AA^\top - VV^\top\right)\right) \\
& = \frac{1}{2} \norm{AA^\top - VV^\top}_{\rm F}^2,
\end{align*}
where $(i)$ follows from $\tr (VV^\top) = d = \tr (AA^\top)$. Similarly, we have 
\[
    \norm{\left(I_D - VV^\top\right) A}_{\rm F}^2  =     \frac{1}{2} \norm{AA^\top - VV^\top}_{\rm F}^2.
\]
Next we consider $ \|V^\top A A^\top V - I_d \|_{\rm F}^2$ . We have
\begin{align*}
    \left\|V^\top A A^\top V - I_d \right\|_{\rm F}^2 &= \tr \left( V^\top A A^\top VV^\top A A^\top V - 2V^\top A A^\top V + I_d      \right) \\
    &= \tr \left( VV^\top AA^\top (VV^\top - I_D) AA^\top + (I_D - VV^\top)AA^\top - AA^\top + I_d  \right) \\
    &= \tr \left( VV^\top AA^\top (VV^\top - I_D) AA^\top + (I_D - VV^\top)AA^\top\right) - \tr \left( AA^\top - I_d  \right) \\
    &= \tr \left( (VV^\top AA^\top - I_D) (VV^\top - I_D) AA^\top  \right)  \\
    &= \tr \left( (VV^\top AA^\top - VV^\top )(VV^\top - I_D) AA^\top  \right)  + \tr \left( (VV^\top - I_D) (VV^\top - I_D) AA^\top  \right)  \\
    &\leq \left\| VV^\top (AA^\top - I_D)\right\|_{\rm F} \cdot \left\| (VV^\top - I_D) AA^\top  \right\|_{\rm F} +  \left\| (VV^\top - I_D) A \right\|_{\rm F}^2  \\
    &\leq \epsilon + \epsilon = 2 \epsilon.
\end{align*}
For item (b), we consider the SVD decomposition of $V^\top A$. Let $V^\top A = W_1^\top \Sigma W_2$, where $W_1, W_2 \in \RR^{d\times d}$ are orthogonal matrices, and $\Sigma = \diag(s_1, s_2, \cdots, s_d)$ are diagonal matrix with $s_1,\dots,s_d$ being the singular values of $V^\top A$. Then we have 
\[
    \left\|V^\top A A^\top V - I_d \right\|_{\rm F}^2 = \sum_{i=1}^d (s_i^2 - 1)^2.
\]
Let $U = W_1^\top W_2 \in \RR^{d\times d}$. Then we know that $U$ is orthonormal. We have 
\begin{align*}
    \left\| U - V^\top A \right\| _{\rm F}^2 &= \sum_{i=1}^d (s_i-1)^2 \\
    &\leq \sum_{i=1}^d (s_i-1)^2 (s_i+1)^2 \\
    &= \sum_{i=1}^d (s_i^2-1)^2  \\
    &= \left\|V^\top A A^\top V - I_d \right\|_{\rm F}^2 .
\end{align*}
The proof is complete.
\end{proof}

\end{document}